\documentclass{article}
\usepackage[preprint]{style/neurips_2025}

\usepackage[utf8]{inputenc}
\usepackage[T1]{fontenc}
\usepackage{hyperref}
\usepackage{url}
\usepackage{shuffle}
\usepackage{booktabs}
\usepackage{amsfonts}
\usepackage{nicefrac}
\usepackage{microtype}
\usepackage{xcolor}
\usepackage{amssymb}
\usepackage{amsmath}
\usepackage{amsthm}
\usepackage{mathtools}
\usepackage{graphicx}
\usepackage{subcaption}
\usepackage{siunitx}
\usepackage{multirow}
\usepackage{verbatim}
\usepackage{algorithm}
\usepackage{algorithmicx}
\usepackage{algpseudocode}
\usepackage{placeins}
\usepackage{float}

\newtheorem{theorem}{Theorem}
\newtheorem{remark}[theorem]{Remark}
\newtheorem{lemma}[theorem]{Lemma}
\newtheorem{corollary}[theorem]{Corollary}

\newtheorem*{theorem*}{Theorem}
\newtheorem{definition}[theorem]{Definition}
\newtheorem{example}[theorem]{Example}

\newcommand{\R}{\mathbb R}
\renewcommand{\O}{\mathcal{O}}
\newcommand{\E}{\mathbb E}
\newcommand{\Var}{\mathrm{Var}}
\newcommand{\Cov}{\mathrm{Cov}}
\newcommand{\mes}{ \mathrm{d} }

\newcommand{\norm}[1]{\left\lVert #1\right\rVert}
\newcommand{\abs}[1]{\left\lvert #1\right\rvert}

\newcommand{\gen}{\mathcal{L}_0}
\newcommand{\err}{\operatorname{err}}

\newcommand{\pmc}[2]{$#1{\scriptstyle{}\mkern-1mu\pm\mkern-1mu#2}$}
\newcommand{\pmcb}[2]{{\boldmath$#1{\scriptstyle{}\mkern-1mu\pm\mkern-1mu#2}$}}
\newcommand{\numb}[1]{{\boldmath$#1$}}

\newcommand{\takeaway}[1]{{\interlinepenalty=10000 #1\par}}

\makeatletter
\renewcommand{\@noticestring}{Preprint.}
\makeatother
\graphicspath{{figures/}}

\title{Sig-DEG for Distillation: Making Diffusion Models Faster and Lighter}

\author{%
  Wen Ge\thanks{Equal contribution.} \\
  University College London \\
  \texttt{wen.ge.25@ucl.ac.uk}
  \And
  Lei Jiang\footnotemark[1] \\
  Alan Turing Institute \\
  University College London \\
  \texttt{ljiang@turing.ac.uk}
  \And
  Lingyi Yang \\
  University of Nottingham \\
  Continuous Intelligence \\
  \texttt{lingyi.yang@maths.ox.ac.uk}
  \And
  Leonardo Roveri \\
  University College London \\
  \texttt{l.roveri@ucl.ac.uk}
  \And
  Mingxuan Yi \\
  University of Bristol \\
  \texttt{mingxuan.yi@bristol.ac.uk}
  \And
  Po-Yu Chen\thanks{\textbf{Disclaimer} - The authors' views are their own. The authors' employers past and present make no representation and warranty whatsoever and disclaim all liability, for the completeness, accuracy or reliability of the information contained herein. This document is not intended as investment research or investment advice, or a recommendation, offer or solicitation for the purchase or sale of any security, financial instrument, financial product or service, or to be used in any way for evaluating the merits of participating in any transaction, and shall not constitute a solicitation under any jurisdiction or to any person, if such solicitation under such jurisdiction or to such person would be unlawful.} \\
  Barclays \\
  \texttt{poyu.chen@barclays.com}
  \And
  Niels Cariou-Kotlarek\footnotemark[2] \\
  University College London \\
  \texttt{niels.kotlarek.23@ucl.ac.uk}
  \And
  Francois Buet-Golfouse\footnotemark[2] \\
  Barclays \\
  \texttt{francois.buetgolfouse@barclays.com}
  \And
  Gaurav Mittal\footnotemark[2] \\
  Barclays \\
  \texttt{gaurav.mittal1@barclays.com}
  \And
  Hao Ni \\
  University College London \\
  \texttt{h.ni@ucl.ac.uk}
}

\begin{document}

\maketitle

\vspace{-1.3em}

\begin{abstract}
  Diffusion models achieve state-of-the-art generative performance but remain computationally expensive at inference, often requiring many discretization steps. We propose Sig-DEG (Signature-based Differential Equation Generator), a generator for distilling pre-trained diffusion models. Inspired by high-order stochastic Taylor expansions, Sig-DEG combines partial signatures that summarize Brownian motion over subintervals with a recurrent structure to approximate stochastic differential equation solutions to achieve universality. We formulate distillation as supervised learning, matching fine-resolution teacher trajectories on a coarse time grid. At inference, Sig-DEG enables fast generation, as the partial signatures can be simulated exactly from their known distributions without requiring fine-grained Brownian paths. Experiments on image and time-series generation demonstrate competitive sample quality with inference speedups of up to several hundred times over the teacher. A Schrödinger bridge example further demonstrates Sig-DEG’s applicability beyond diffusion distillation.
\end{abstract}

\section{Introduction}

Diffusion models \citep{sohl2015deep, ho2020denoising, song2021scorebased} have become a cornerstone of generative modelling, achieving state-of-the-art image synthesis \citep{dhariwal2021diffusion} and strong performance in applications like voice conversion \citep{choi2023diff} and motion synthesis \citep{jiang2025motionpcm}. These models gradually corrupt data into noise via a forward process and generate new samples by integrating either the reverse-time stochastic differential equation (SDE) or an equivalent probability-flow ordinary differential equation (ODE), iteratively denoising from a noise prior (typically Gaussian) into samples from complex data distributions.

However, the generation process often requires many fine-grained discretization steps, incurring a significant computational cost that creates a bottleneck for real-time or resource-constrained applications. A growing body of work accelerates inference by improving the numerical solvers \citep{lu2025dpm} or distilling pre-trained models to retain generation fidelity (e.g., progressive distillation \citep{salimans2022progressive}, consistency models \citep{song2023consistency} and their trajectory and phased variants \citep{kim2023consistency, wang2024phased}). Yet most existing approaches focus on ODE-based diffusion models, while high-fidelity distillation of SDE-based models remains challenging and relatively underexplored.%

\begin{figure*}
\centering
    \includegraphics[width=0.9\textwidth]{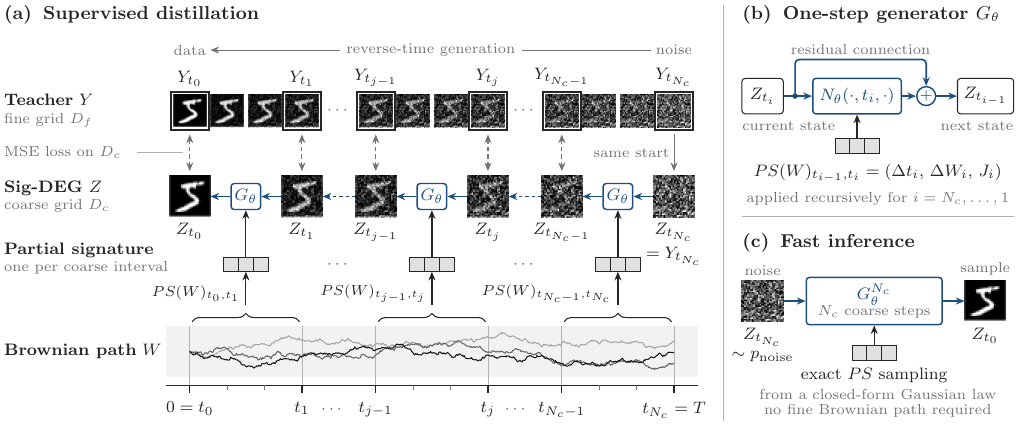}
    \caption{\textbf{Overview of Sig-DEG for distilling diffusion models.}}\label{fig:sigDeg}
\end{figure*}
To this end, we propose Signature-based Differential Equation Generator (\textit{Sig-DEG}), a fast and lightweight surrogate for SDE-based diffusion models that directly approximates the reverse-time SDE using tools from stochastic analysis and rough path theory \citep{friz2014course}. Our method builds on high-order numerical schemes, particularly stochastic Taylor expansions, which enable accurate approximation of SDE solutions using coarse discretizations. Specifically, we introduce \textit{partial signatures}, a low-dimensional and tractable subset of path signatures \citep{lyons2007differential}, to efficiently summarize Brownian motion trajectories over each sub-time interval while retaining the information needed for accurate SDE approximation. The partial signatures represent a set of principled features for Brownian motion, whose dimension is independent of the underlying discretization resolution, substantially reducing the input dimension compared with directly using discretized Brownian motion.

Sig-DEG models the generative process as a discrete-time recurrence over coarse time steps. At each step, a recurrent neural network consumes the current latent state and the analytically simulable partial signature to produce the next point in the reverse trajectory. We train Sig-DEG by supervised distillation from a fine-grained teacher using a mean-squared-error (MSE) loss. Crucially, because the distribution of partial signatures is known in closed form, Sig-DEG avoids the need to reconstruct fine-resolution Brownian paths at inference time -- leading to orders-of-magnitude acceleration while maintaining high generative quality. The overall workflow of training Sig-DEG is given in Figure~\ref{fig:sigDeg}.

Unlike consistency-based distillation methods, which enforce trajectory coherence through specialized losses or architectural constraints, Sig-DEG naturally respects the underlying SDE structure. Its design is grounded in approximation theory, offering both theoretical guarantees and empirical efficiency.

\paragraph{Contributions.}
\begin{enumerate}
\item We reformulate distillation as a supervised learning task, learning the mapping between the driving noise trajectory and the reverse diffusion process on a coarse time grid.
\item We propose Sig-DEG, a theoretically grounded model that uses high-order stochastic Taylor expansions and closed-form partial signature distributions for exact simulation and fast inference.
\item We demonstrate strong performance across diverse modalities -- including images and rough volatility time series -- achieving competitive generation quality with two orders of magnitude faster inference. We further reduce training costs through a resolution curriculum and progressive distillation across Sig-DEG models.
\item We extend Sig-DEG to distil Schrödinger bridge models, showcasing its applicability beyond diffusion model distillation.
\end{enumerate}

\section{Related Work}
\subsection{Diffusion Distillation}
A prominent line of work focuses on probability-flow ODEs. Consistency models (CMs) \citep{song2023consistency} learn maps to ODE trajectory endpoints, while consistency trajectory models (CTMs) \citep{kim2023consistency} and phased consistency models (PCMs) \citep{wang2024phased} extend this to arbitrary-time and phase-boundary mappings, respectively. Progressive distillation \citep{salimans2022progressive} halves the number of sampling steps at each stage by training one student step to match two deterministic teacher steps.
Recent work improves the stability and scalability of continuous-time consistency training \citep{lu2025simplifying}, while distribution-matching approaches \citep{yin2024one} instead align the student and teacher output distributions.%

SDE-based acceleration includes numerical solvers such as SEEDS and SA-Solver \citep{gonzalez2023seeds,xue2023sasolver}, learned stochastic transitions such as denoising diffusion GANs \citep{xiao2022tackling}, and SDE-based consistency distillation such as SCott \citep{liu2024scott}.
A central challenge is that coarse reverse transitions can be non-Gaussian \citep{guo2023gaussian}, while pathwise approximation must also capture dependence on the driving Brownian path.
Recent stochastic transition-map distillation learns conditional transition distributions without a pretrained teacher \citep{rapakoulias2026stochastic}, while It\^o-map models \citep{pan2026maps} learn Brownian-conditioned solution maps through self-distillation or teacher-guided training.

In contrast, Sig-DEG learns to approximate a pretrained teacher’s reverse-SDE trajectories through supervised, samplewise matching under shared Brownian noise. Using partial signatures to locally represent the Brownian path driving the teacher, it approximates the pathwise mapping from the current state and driving noise to the reverse-SDE solution over coarse time intervals. Our progressive distillation stage (Section~\ref{sec:training_accel}) adapts the strategy of \citet{salimans2022progressive} to the SDE setting by successively distilling between Sig-DEG models driven by the same Brownian path. Chen's identity provides the partial-signature features of merged intervals in closed form, enabling samplewise matching. %
\subsection{Neural Differential Equation Models}
Our work relates to the broader class of neural differential equation models, including neural ODEs \citep{chen2018neural}, neural controlled differential equations (CDEs) \citep{kidger2020neural, morrill2021neuralcontrolled}, and neural SDEs \citep{kidger2021neural}. These models represent trajectories through parametrized continuous-time dynamics, enabling flexible modelling of temporal processes. Neural SDEs model stochastic trajectories directly but typically rely on low-order solvers such as Euler--Maruyama, requiring fine time discretization to be accurate and thereby limiting their scalability for inference-intensive tasks like diffusion sampling.

Sig-DEG also connects to neural approximations of CDEs using path signatures \citep{lyons2007differential}, including Logsig-RNNs \citep{liao2019learning, liao2021logsig} and neural rough differential equations \citep{morrill2021neural, walker2024log}. These architectures leverage signature features to model complex temporal dependencies, drawing from rough path theory. However, existing signature-based models often require intractable simulations of high-order signature terms, particularly for stochastic drivers like Brownian motion.
In contrast, Sig-DEG uses a carefully chosen subset of signature terms, which we call \textit{partial signatures}, that are both analytically tractable and sufficient for high-order SDE approximation. This preserves theoretical fidelity while achieving exact simulation and practical speed-ups.

\section{Preliminaries}\label{sec:prelims}
\subsection{Diffusion Model}
Diffusion models are a class of generative models that learn to synthesize data by modelling the reversal of a stochastic noising process. The \textit{forward process}, also known as the diffusion process, incrementally adds Gaussian noise to data over a fixed time horizon $T$. More precisely,
    \begin{eqnarray}\mes X_t = f(X_t, t)\mes t+\sigma_t\cdot\mes W_t, \quad X_0 \sim p_{\text{data}},\label{eqn:forward}
    \end{eqnarray}
    where $t\in [0, T]$, $f(X_t, t)$ and $\sigma_t$ denote the drift and diffusion terms respectively, $W_t$ is a standard Wiener process, and $\cdot$ denotes Itô integration.

Associated with the forward noising process is a denoising process \citep{song2021scorebased}, known as the \textit{backward process}
     \begin{eqnarray}\mes Y_t &=& (f(Y_t, t) -\sigma_t^2 \nabla \log p(t, Y_t)) \mes t+\sigma_t \mes \tilde{W}_t, \nonumber\\
     Y_T &\sim& p_{\text{noise}},\label{eqn:backward}
 \end{eqnarray}
 where $\tilde{W}_t$ is the reversed standard Wiener process, with time $t$ evolving backwards, and
$\nabla \log p(t, Y_t)$ is the score function of the marginal probability $p(t, \cdot)$. The score function can be modelled by a time-dependent neural network, $s_\theta(X_t, t)$, optimized by minimizing the score matching objective
\begin{align*}
\mathbb{E}_{t,X_0,X_t}\!\Big[\lambda(t) \left\| \nabla_{X_t} \log p(X_t | X_0) - s_\theta(X_t, t) \right\|_2^2\Big],
\\
\text{with } t\!\sim\!U([0,T]),\; X_0\!\sim\!p_{\text{data}},\;
X_t\!\sim\!p(X_t \mid X_0),
\end{align*}
where $\lambda(t)$ is a weighting function.

\subsection{High-Order Approximation of SDEs}\label{subsec: SDE_approx}
It is well-known that there is a one-to-one, explicit correspondence between the It\^o and Stratonovich integral (see Appendix \ref{sec:ito_strat}). Given that numerical approximation theory has a simpler representation for Stratonovich SDEs compared to Itô SDEs, we consider the general SDE in the Stratonovich sense as follows:
\begin{eqnarray}\label{eqn:SDE}
    dY_t = \mu(t, Y_t)dt+\sigma_t \circ dW_t, \quad \forall t \in [0, T],
\end{eqnarray}
where $\circ$ denotes the Stratonovich integral. In fact, in the diffusion model, as $\sigma$ is a deterministic function w.r.t. $t$, the forward (resp. backward) process defined by Equation~\eqref{eqn:forward} (resp. Equation~\eqref{eqn:backward}) coincides with their Stratonovich SDEs. For simplicity, we assume that the vector fields $\mu$ and $\sigma$ are smooth.
\paragraph{Stochastic Taylor Expansion.} One has that for any $0\leq s\leq t$ and $t-s$ is sufficiently small \citep{baudoin2012stochastic, kloeden1992stochastic},
\begin{align*}
Y_t \approx&\ Y_s + \mu(s, Y_s) (t-s) + \sigma_s (W_t-W_s)\\
&{} +  \partial_y \mu(s, Y_s) \sigma_s \int_{s}^{t}\int_{s}^{u} dW_{s_1} \mes u + \cdots.
\end{align*}
It means that there exists $\varepsilon>0$ and function $F$ such that for any $|t-s| \leq \varepsilon$, $Y_t$ can be approximated by
\begin{eqnarray}\label{eqn:Y_approx}
Y_t \approx Y_s + F\left(s, Y_s, PS(W)_{s, t} \right),
\end{eqnarray}
where $PS(W)_{s, t}:=(t-s, W_t-W_s, \int_{s}^{t}\int_{s}^{u} dW_{s_1}du)$. The function $F: \mathbb{R}^{+} \times E \times \mathbb{R}^{+} \times E \times E \rightarrow E$ has the following form:
\begin{align*}
F(s,y,s_1,s_2,s_{21})
={} & \mu(s,y)s_1 + \sigma_s s_2 + \partial_y \mu(s,y)\,\sigma_s s_{21} \\
& + \frac12 \big( \partial_y \mu(s,Y_s) + \partial_t \mu(s,Y_s) \big) s_1^2 \\
& + \partial_t \sigma_s\,(s_1 s_2 - s_{21})
 .
\end{align*}
The derivation of $F$ can be found in Appendix~\ref{app:taylor}. The partial signature $PS(W)_{s, t}$ comprises a selected subset of the signature (up to degree 2) of the time-augmented path $W$ over $[s, t]$. Moreover, for any $0\leq s\leq u\leq t$, it satisfies the \textbf{multiplicative property}
\begin{align}
PS_{s,t} &= PS_{s,u}+PS_{u,t}+(0, 0, W_{s,u}(t-u)), \nonumber \\
PS_{s,s} &= 0. \label{eqn:ps_recursion}
\end{align}
This simple recurrence enables the efficient computation of the partial signature of a piecewise linear approximation of Brownian motion, with linear-time complexity.
See Appendix \ref{sec:path_prelims} for the formal definition of the path signature and a rigorous statement of the stochastic Taylor expansion.

\paragraph{Numerical Approximation Scheme.} We now construct a global approximation of the SDE solution by iteratively composing the local approximation above. We discretize $[0, T]$ using a finite partition ${D}_{c} = \{0=t_0<t_1 < \cdots <t_{N_c}=T\}$ and let $\Delta t_{c} = \max_{i\in \{1, \cdots, N_c \}}(t_i-t_{i-1}) $ denote the mesh size.  Starting from $\hat{Y}_{t_0} = Y_0$, we recursively apply the local approximation on each subinterval. Specifically for $i \in \{0, \cdots, N_c-1\}$, set
\begin{eqnarray}\label{eqn:SDE_rec}
    \hat{Y}_{t_{i+1}} = \hat{Y}_{t_{i}} + F\left(t_{i}, \hat{Y}_{t_i}, PS(W)_{t_{i}, t_{i+1}} \right),
    \end{eqnarray}
then $\hat{Y}_T = \hat{Y}_{t_{N_c}}$ gives the global approximation of $Y_T$.

Under standard regularity conditions on $\mu$ and $\sigma$, we have the following results.
\begin{theorem}
 Assume that $\sigma$ only depends on $t$, then $\hat{Y}_{t_{N_c}}$ approximates $Y_T$ in the strong sense with a local error of order $\mathcal{O}(\Delta t_c^2)$.
\end{theorem}

\begin{theorem}
    Assume Brownian motion as the driving noise. The scheme in Equation~\eqref{eqn:SDE_rec} yields a strong approximation of $Y_T$ with a global error $\mathcal{O}(\Delta t_c^{3/2})$.
\end{theorem}
See Appendix {\ref{app:taylor}} for more details.

\begin{remark}
This is improves the Euler–Maruyama scheme, which achieves only $\mathcal{O}(\Delta t_c)$ strong error. See \citet{kloeden1992stochastic} for a comprehensive treatment of these results. The inclusion of second-order terms from the partial signature allows Sig-DEG to achieve high accuracy even on coarse time grids.
\end{remark}

Moreover, the joint distribution of $(PS(W)_{t_i, t_{i+1}})_{i}$ is explicitly known (Lemma~\ref{lemma1}), and can be simulated exactly. The proof is deferred to Appendix \ref{sec:lemmaproof}.

\begin{lemma}\label{lemma1}
Given the time partition ${D}_c:=\{t_i\}_{i=0}^{N_c}$, each $PS(W)_{t_{i-1}, t_i}$ is mutually independent across $i$. For $0 \leq s \leq t$, $ (W_t - W_s, \int_s^t \int_s^{u} \mes W_{s_1} \mes u)$ is a Gaussian random vector with mean zero and covariance matrix
\begin{equation}
\Sigma =
\begin{bmatrix}
 t-s & \frac{(t-s)^2}{2}  \\
\frac{(t-s)^2}{2}  & \frac{(t-s)^3}{3}
\end{bmatrix}.
\end{equation}
\end{lemma}

\section{Sig-DEG for Diffusion Distillation}\label{sec:sig-deg}

For efficient model distillation, we approximate the backward process of a pre-trained diffusion model on a coarse time grid. To this end, we propose a signature-based differential equation generator ({Sig-DEG}), which reformulates this approximation as supervised learning using fine-resolution noise information. Once trained, Sig-DEG generates samples on the coarse grid without requiring fine-resolution noise, enabling substantially faster inference.

Throughout this section, let $D_c = \{0 = t_0 < t_1 < \dots < t_{N_c} = T\}$ and $D_f = \{0 = s_0 < s_1 < \dots < s_{N_f} = T\}$ denote the coarse and fine partitions respectively. We assume equally spaced nested partitions with $N_f = r N_c$ for some integer $r$, so that $D_c \subset D_f$, and \ $t_j = s_{rj}$. Let $W$ denote a $d$-dimensional Brownian motion on $[0, T]$.

\subsection{Sig-DEG Model}
Motivated by the high-order numerical approximation scheme of SDEs in Section \ref{subsec: SDE_approx}, we propose the Sig-DEG generator, which parameterizes the solution map in Equation~\eqref{eqn:SDE_rec} by neural networks.

\begin{algorithm*}[t]
\caption{Sig-DEG supervised training }%
\begin{algorithmic}[1]
    \State \textbf{Input:} teacher $g_\phi$ (one reverse step on $D_f$ driven by a Brownian increment); forward perturbation kernel $p(X_t \mid X_0)$; generator $N_\theta$; learning rate $\eta$; partitions $D_c = \{t_j\}_{j=0}^{N_c}$ and $D_f = \{s_l\}_{l=0}^{N_f}$ with $N_f = r N_c$ and $t_j = s_{jr}$.
    \Repeat
        \State $x_0 \sim p_{\text{data}}$, \quad $j \sim \mathcal{U}(\{1, \dots, N_c\})$
        \State $Y_{t_j} \sim p(X_{t_j} \mid X_0 = x_0)$ \Comment{forward perturbation kernel}
        \State $\mathrm{PS} \gets \mathbf{0}$
        \For{$l = jr, jr - 1, \dots, (j-1)r + 1$} \Comment{$r$ teacher steps from $t_j$ to $t_{j-1}$}
            \State $h_l \gets s_l - s_{l-1}$, \quad $\Delta W_l \gets \sqrt{h_l}\, \boldsymbol{\zeta}_l$, \quad $\boldsymbol{\zeta}_l \sim \mathcal{N}(0, \mathbf{I})$
            \State $Y_{s_{l-1}} \gets g_\phi(Y_{s_l}, \Delta W_l)$ \Comment{teacher step}
            \State $\mathrm{PS} \gets \mathrm{PS} + \big(h_l,\ \Delta W_l,\ \tfrac{1}{2} h_l \Delta W_l + \Delta W_l\,(t_j - s_l)\big)$ \Comment{now $\mathrm{PS}(W)_{s_{l-1}, t_j}$, Equation~\eqref{eqn:ps_recursion}}
        \EndFor
        \State $\hat{Y}_{t_{j-1}} \gets Y_{t_j} + N_\theta\big(Y_{t_j}, t_j, \mathrm{PS}\big)$ \Comment{$\mathrm{PS} = \mathrm{PS}(W)_{t_{j-1}, t_j}$}
        \State $\mathcal{L}(\theta) \gets \lVert \hat{Y}_{t_{j-1}} - Y_{t_{j-1}} \rVert^2$
        \State $\theta \gets \theta - \eta \nabla_\theta \mathcal{L}(\theta)$
    \Until{convergence}
\end{algorithmic}
\label{alg:Sig-DEG}
\end{algorithm*}
\begin{definition}[One-step Sig-DEG model]
\label{def:sigdeg}
Let $W$ be a $d$-dimensional Brownian motion on $[0,T]$, and
let $\Delta_T := \{(s,t): 0 \leq s < t \leq T\}$.
A one-step Sig-DEG model is a map
$G_\theta(\cdot;W): \mathbb{R}^d \times \Delta_T
\to \mathbb{R}^d$, defined by
\begin{equation}\label{eqn:rec_unit}
G_\theta(z,s,t;W)
:= z + N_\theta\!\left(z,t,PS(W)_{s,t}\right),
\end{equation}
where $N_\theta$ is a neural network with parameters $\theta$
and output in $\mathbb{R}^d$, and $PS(W)_{s,t}$ denotes the
partial-signature features of $W$ over $[s,t]$.
The model maps a state $z$ at time $t$ to a generated state
at time $s$.
\end{definition}

\paragraph{Sig-DEG rollout.}
Given the coarse partition $D_c$ and terminal sample $Z_{t_{N_c}}:=Y_{t_{N_c}} \sim p_{\text{noise}}$, we recursively generate
the reverse trajectory $(Z_{t_{N_c-1}},\ldots,Z_{t_0})$ via
\begin{equation}\label{eqn:sigdeg_rollout}
Z_{t_{i-1}}
= G_\theta\!\left(Z_{t_i},t_{i-1},t_i;W\right),
\qquad i=N_c,\ldots,1.
\end{equation}
For this partition, let $G_\theta^{N_c}$ denote the resulting
map from noise vector $Y_{T}$ to the output at time $t_0$
\begin{equation}\label{eqn:sigdeg_rollout_map}
\begin{aligned}
G_\theta^{N_c}(Y_T,T;W)
&:= Z_{t_0} = \bigl(
G_\theta(\cdot,t_0,t_1;W)
\circ \cdots \\
&\qquad \circ
G_\theta(\cdot,t_{N_c-1},t_{N_c};W)
\bigr)(Y_T).
\end{aligned}
\end{equation}
Although the rollout in Equation~\eqref{eqn:sigdeg_rollout} is discrete, the one-step map is defined for arbitrary time intervals. Thus, Sig-DEG can be viewed as a continuous-time model and rolled out on any, potentially non-uniform, partition $D_c$. We suppress the dependence on $D_c$ in the notation.

Having defined the Sig-DEG rollout, we next establish that this model class is sufficiently expressive to approximate the target backward process. In particular, under suitable regularity conditions, Sig-DEG can approximate the solution of a broad class of SDEs arbitrarily well as the coarse partition is refined.

\begin{theorem}[Universality of Sig-DEG, informal]\label{thm: uni_sig_DEG}
Suppose that $Y$ denotes the solution to the backward SDE of Equation \eqref{eqn:SDE} with the vector field $(\mu, \sigma)$ and the terminal condition $Y_T \sim p_{\text{noise}}$. Assume that $(\mu, \sigma, Y_T)$ satisfies some regularity conditions. For every $\varepsilon>0$ and $\delta\in(0,1)$, and every sufficiently fine finite partition $D_c$, there exists a Sig-DEG $G_{\theta}$ such that
$$
\mathbb P \left(
\max_{i \in \{1, \cdots N_c\}}||Y_{i}-Z_i^\theta||
\le\varepsilon
\right)\ge 1-\delta.
$$
\end{theorem}
See Appendix~\ref{appendix:univerality} for a precise statement and proof.

\subsection{Supervised Learning Module}
We formulate distillation as supervised learning of a pretrained diffusion model's solution map, evaluated on a coarse time partition.  Given a pre-trained teacher model $g_\phi$, we sample a terminal state $Y_T$
and a Brownian path $W$ on the fine partition $D_f$, generate the corresponding
teacher trajectory, and retain its values on the coarse partition $D_c$.
This yields training pairs $
\bigl((Y_T,W_{D_f}),\,Y_{D_c}\bigr).$
For a fixed teacher and numerical discretization, $Y_{D_c}$ is a deterministic
function of $(Y_T,W_{D_f})$. The underlying regression problem is therefore
\emph{noiseless} up to the approximation error introduced by SigDEG, which vanishes when the coarse time mesh is refined (Theorem \ref{thm: uni_sig_DEG}).

Let $Z^\theta_{D_c}$ denote the Sig-DEG rollout initialised at the same terminal
state $Y_T$ and driven by the same Brownian realisation (approximated by $ W_{D_f}$) as the teacher. We train
the model by minimising the samplewise loss
\begin{equation}\label{eqn:loss_global}
\mathcal{L}_{\theta}
:=
\mathbb{E}_{Y_T,W_{D_f}}
\left[
\left|\left|Y_{D_c}-Z^\theta_{D_c}\right|\right|^{2}
\right],
\end{equation}
where $||\cdot||$ is the $l_2$ norm, and the expectation is over the sampled terminal state and Brownian path on $D_f$.
Sharing the terminal state and driving path provides a samplewise coupling between the
teacher and student trajectories. Thus, a small loss ensures accurate trajectory approximation and controls the discrepancy between their output distributions at $t_0=0$.

\paragraph{Local loss for training acceleration.} Evaluating and differentiating the global loss $\mathcal{L}_{\theta}$ (Equation~\eqref{eqn:loss_global}) over the entire coarse time partitions can be computationally intensive. To reduce cost, we introduce a local learning objective by uniformly sampling a timepoint as the surrogate for $\mathcal{L}_{\theta}$ and matching the next coarse step per iteration. Specifically, we uniformly sample $j \in \{1, \cdots, N_c\}$ and let $t=t_j$; $\mathcal{L}_{\theta}^{\mathrm{local}}$ is defined by
\begin{equation} \nonumber%
\mathcal{L}_{\theta}^{\mathrm{local}}(t_j) := \mathbb{E}[||Y_{t_{j-1}} - G_{\theta}(Y_{t_j}, t_{j-1}, t_j, W)||^{2}],
\end{equation}
where the expectation is taken under $Y_{t_j}$ and independent $W_{D_f}$ restricted on time interval $[t_{j-1}, t_{j}]$.  This loss can be applied to any training interval $[s, t]$, provided that $|t-s|$ is sufficiently small.

In practice, using log-MSE or step-dependent weighted MSE as the local loss may improve optimisation.%
We present the training pseudo-code in Algorithm~\ref{alg:Sig-DEG}. %

\subsection{Training on Coarse Meshes}\label{sec:training_accel}%
We introduce two complementary strategies for efficient training and few-step generation. A resolution curriculum reduces Sig-DEG training cost, while progressive distillation further reduces the number of sampling steps of a high-quality pre-trained Sig-DEG model.
\paragraph{Resolution curriculum.}
Teacher rollouts dominate the cost of supervised training when $N_f$ is large, while moderate resolutions already provide useful training targets.
For a fixed student partition $D_c$, we therefore progressively refine the teacher mesh during training, following curriculum learning \citep{bengio2009curriculum}. Specifically, we use nested meshes
$D_f^{(1)} \subset D_f^{(2)} \subset \dots \subset D_f^{(R)} = D_f,
\qquad
N_f^{(1)} < \dots < N_f^{(R)} = N_f$,
each containing $D_c$. At stage $r$, training targets are generated by rolling $g_\phi$ out on $D_f^{(r)}$. Early stages use cheaper, less accurate targets to learn an initial transition map, which subsequent stages refine using increasingly accurate teacher rollouts. Training advances to the next stage when the validation metric fails to improve by more than a relative tolerance $\rho$ for $P$ consecutive evaluations, or when the stage reaches its maximum number of iterations. The final stage uses the full teacher mesh till the maximum number of iterations.%

\paragraph{Progressive distillation.}
Starting from a well-trained and high-fidelity Sig-DEG model, we further distil it into a student with fewer sampling steps. Let a pre-trained Sig-DEG parent use $K_p$ sampling steps and the child use $K_c=\lfloor K_p/g \rfloor$ steps, so that each child step approximates $g$ consecutive parent steps.

For each child interval $[t_{j-g},t_j]$, we combine the partial signatures of its $g$ parent intervals using the recurrence of partial signature in Equation~\eqref{eqn:ps_recursion}, obtaining
\[
P_j^{\mathrm{merge}}
:=PS(W)_{t_{j-g},t_j}
=
\bigl(t_j-t_{j-g},
\Delta W_j^{\mathrm{merge}},
J_j^{\mathrm{merge}}\bigr),
\]
where
\begin{align*}
\Delta W_j^{\mathrm{merge}}
&=\sum_{i=j-g+1}^{j}\Delta W_i,\\
J_j^{\mathrm{merge}}
&=\sum_{i=j-g+1}^{j}
\bigl[J_i+(t_j-t_i)\Delta W_i\bigr].
\end{align*}
This composition gives the exact merged feature directly from the parent features, without simulating a fine-resolution Brownian path.

Given a frozen parent rollout $(Z^{(p)}_{t_i})_i$, we train the child with
\begin{equation}\label{eqn}
\mathcal{L}^{\mathrm{PD}}_{\theta}
:=
\mathbb{E}\Big[
\big\lVert
Z^{(p)}_{t_{j-g}}-Z^{(p)}_{t_j}
-N_{\theta}\big(Z^{(p)}_{t_j},t_j,P_j^{\mathrm{merge}}\big)
\big\rVert^2
\Big],
\end{equation}
where $j$ is sampled uniformly from $\{g,2g,\dots,K_p\}$, and the expectation also covers the parent's initial state and partial signature sampling. The loss matches each child transition samplewise to the corresponding multi-step parent transitions. Following \citet{salimans2022progressive}, we repeat this procedure using each trained Sig-DEG child as the parent for the next stage, progressively reducing the number of sampling steps.

\providecommand{\tsgcol}{}\renewcommand{\tsgcol}[1]{\includegraphics[width=0.31\linewidth,height=0.62\linewidth]{tsg/#1}}
\providecommand{\tsglabels}{}\renewcommand{\tsglabels}{\makebox[0.31\linewidth]{\scriptsize Teacher}\hfill\makebox[0.31\linewidth]{\scriptsize Supervised}\hfill\makebox[0.31\linewidth]{\scriptsize Generated}}
\begin{figure*}[t]
    \centering
    \begin{subfigure}[t]{0.2\textwidth}
        \centering
        \caption{1D Gaussian Mixture}
        \includegraphics[width=\linewidth]{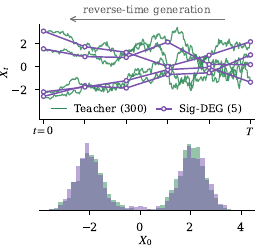} %
    \end{subfigure}\hfill{\color{gray}\vrule width 0.5pt}\hfill
    \begin{subfigure}[t]{0.24\textwidth}
        \centering
        \caption{rBerg}
        \tsgcol{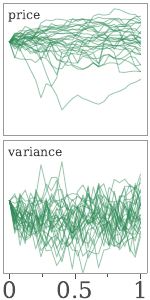}\hfill\tsgcol{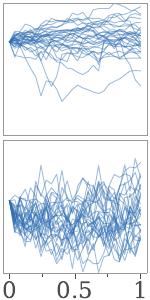}\hfill\tsgcol{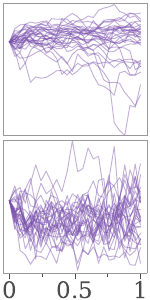}\\[1pt]
        \tsglabels
    \end{subfigure}\hfill
    \begin{subfigure}[t]{0.24\textwidth}
        \centering
        \caption{MNIST}
        \tsgcol{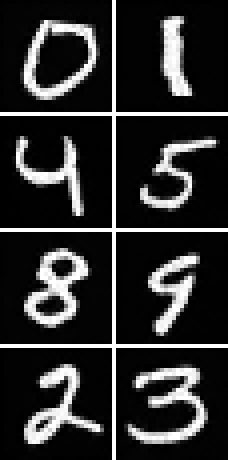}\hfill\tsgcol{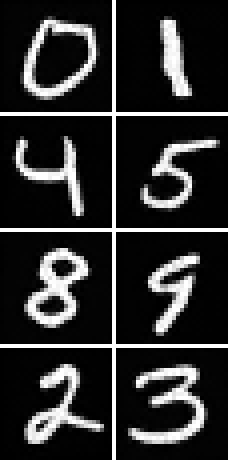}\hfill\tsgcol{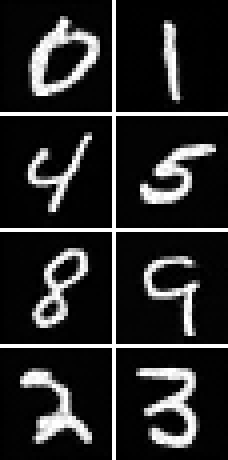}\\[1pt]
        \tsglabels
    \end{subfigure}\hfill
    \begin{subfigure}[t]{0.24\textwidth}
        \centering
        \caption{CIFAR-10}
        \tsgcol{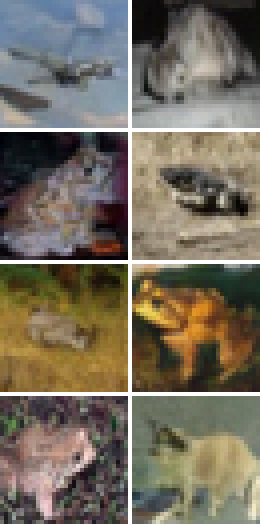}\hfill\tsgcol{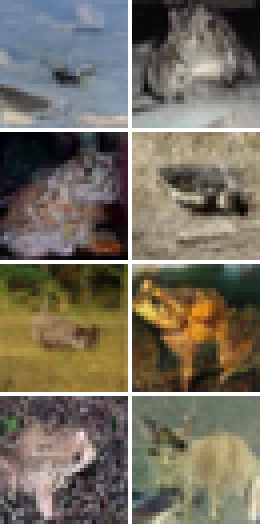}\hfill\tsgcol{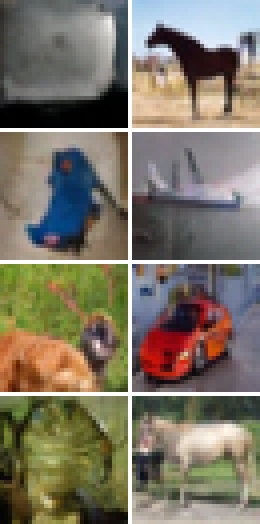}\\[1pt]
        \tsglabels
    \end{subfigure}
    \caption{\textbf{Qualitative generation results.} (a) Sig-DEG on a 1D Gaussian mixture: backward trajectories (top) and final-sample histograms (bottom). (b--d) Teacher samples, Sig-DEG driven by the teacher's Brownian path (Supervised), and Sig-DEG with independent noise (Generated): (b) rBerg price process (top) and variance process (bottom) paths; (c) MNIST; (d) CIFAR-10. Sig-DEG uses 5 coarse steps in (a--c) and 16 in (d).}%
    \label{fig:tsg_1d}
\end{figure*}

\subsection{Inference with Exact Simulation}
Once trained, Sig-DEG generates samples on the coarse partition $D_c$, starting from $Z_{t_{N_c}}\sim p_{\mathrm{noise}}$. For $i=N_c,\ldots,1$, we sample the stochastic components of $\mathrm{PS}(W_{D_f})_{t_{i-1},t_i}$ from a centred Gaussian distribution and apply the Sig-DEG update. We use the exact finite-partition covariance $\Sigma_{D_f}$ (Lemma~\ref{lem:discrete_ps_gaussian}), or its continuous-time approximation (Lemma~\ref{lemma1}) when $D_f$ is sufficiently fine. This avoids simulating fine-resolution Brownian paths. With precomputed covariances, inference cost scales linearly with $N_c$ and is independent of $N_f$, enabling substantial acceleration when $N_c\ll N_f$ while maintaining high generation quality.

\section{Numerical Experiments}\label{sec:experiments}

\paragraph{Datasets and baselines.}
We evaluate Sig-DEG by distilling pretrained diffusion models across four datasets spanning multiple modalities, with an extension to pretrained Schrödinger bridge models from \cite{noble2026twisted}. The diffusion benchmarks consist of four datasets: a one-dimensional Gaussian mixture (\textbf{1D-Gm}), rough-volatility
time series (\textbf{rBerg}), and images (\textbf{MNIST}
\citep{lecun1998gradient} and \textbf{CIFAR-10}
\citep{krizhevsky2009learning}). For a fair comparison, we benchmark Sig-DEG against a consistency model (CM) \citep{song2023consistency} distilled from the same teacher's ODE counterpart \citep{edm} on 1D-Gm, rBerg and MNIST.

\textbf{Test metrics}. For image generation, we report Fréchet Inception Distance (FID) \citep{heusel2017gans} and, on MNIST, Inception Score (IS) \citep{salimans2016improved}. For rBerg, we assess five performance metrics including marginal distribution loss, (auto-)correlation discrepancies, and discriminative (D-score) and predictive (P-score) scores following \cite{NEURIPS2023_7d0e8675}. Further details on the datasets, architectures, hyperparameters, and evaluation metrics are provided in  Appendix~\ref{sec::comput_details}.

We use Sig-DEG ($N_c$), CM ($N_c$) for $N_c$-step sampling and Teacher ($N$) for $N$-step sampling.

\subsection{Few-Step Generation}\label{sec:fewstep}

Figure~\ref{fig:tsg_1d}(a) illustrates Sig-DEG on a one-dimensional Gaussian mixture, approximating 300-step teacher trajectories with only 5 coarse steps. The sample plots in Figure~\ref{fig:tsg_1d}(b--d) provide qualitative evidence of well-trained Sig-DEG models across the remaining datasets.

\paragraph{Teacher fidelity at a fraction of the sampling cost.}
Across the four benchmarks, Sig-DEG uses 5--64 sampling steps
compared with 300--2048 steps for the corresponding teachers, while achieving comparable distributional metrics. On MNIST, five- or ten-step Sig-DEG achieve
FIDs of $7.96$ and $8.28$ compared with $8.94$ for a 1500-step teacher, and are $297\times$ and $149\times$ faster, respectively
(Figure~\ref{fig:quality_vs_time}(a)). Similarly, on rBerg,
ten-step Sig-DEG achieves a histogram loss of $0.0323$, compared with $0.0341$ for the teacher, with a $106\times$ speedup
(Figure~\ref{fig:quality_vs_time}(b);
Table~\ref{tab::rberg}).

On CIFAR-10, we fine-tune Sig-DEG with an additional distributional MMD loss in the later stages of training. With 64 and 32 steps, Sig-DEG reaches FID $3.42$ and $3.65$, against $3.63$ for the 2048-step teacher (Table~\ref{tab:cifar10}; Appendix~\ref{app:cifar}). At 16 and 8 steps, it achieves $58\times$ and $115\times$ speedups at FID $3.93$ and $4.47$. For $N_c=8$, an adaptive mesh based on error accumulation reduces FID from 4.47 to 3.88 (Appendix~\ref{app:cifar}). A comparison with reference training-free baselines is provided in Appendix \ref{app:cifar_baselines}.

\begin{figure}[ht]
    \centering
    \includegraphics{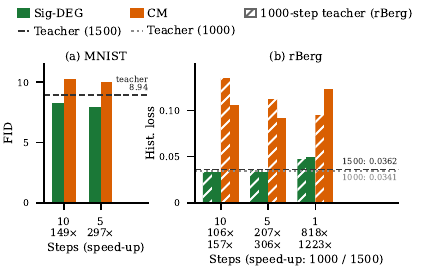}  %
    \caption{\textbf{Sample quality at matched sampling budgets.} Sig-DEG and CM at the same number of steps on (a) \textbf{MNIST} (1500-step teacher), measured by FID, and (b) \textbf{rBerg}, measured by histogram loss. Lines indicate teacher performance; values below each step report Sig-DEG's speed-up over its teacher.}%
    \label{fig:quality_vs_time}
\end{figure}
\begin{table}[!h]
    \centering
    \footnotesize
    \setlength{\tabcolsep}{4pt}
    \begin{tabular}{rccc}
        \toprule
        Method & FID $\downarrow$ & Time (s) $\downarrow$ & $\times$ Speed $\uparrow$ \\
        \midrule
        Teacher (2048)  & \pmc{3.63}{0.05} & \pmc{4517.0}{5.9} & -- \\
        DPM-Solver++ (8) & \pmc{10.45}{0.07} & 26.3 & 172 \\
        DPM-Solver++ (32) & \pmc{3.25}{0.03} & 105.2 & 43 \\
        \midrule
        Sig-DEG (64)    & \pmc{3.42}{0.04} & \pmc{314.22}{0.31} & $14$ \\
        Sig-DEG (32)    & \pmc{3.65}{0.02} & \pmc{157.24}{0.05} & $29$ \\
        Sig-DEG (16)    & \pmc{3.93}{0.03}  & \pmc{78.55}{0.19}  & $58$ \\
        Sig-DEG (8)   & \pmc{4.47}{0.05}  & \pmc{39.21}{0.20}  & $115$ \\
        Sig-DEG-A (8) & \pmc{3.88}{0.03} & \pmc{39.33}{0.02}  & $115$\\
        \bottomrule
    \end{tabular}
    \caption{\textbf{CIFAR-10 generation quality and inference speed.} FID on 50,000 samples, mean $\pm$ SD over five sampling seeds where given; time for 79 batches of 128 images on the same GPU. Teacher: Euler--Maruyama. DPM-Solver++ \citep{lu2025dpm}: training-free ODE solver (2M) on the same network as the teacher, better of two step schedules per budget. Sig-DEG (16, 8): resolution curriculum; Sig-DEG-A: adaptive mesh.}
    \label{tab:cifar10}
\end{table}

\paragraph{Consistent gains at matched sampling budgets.}
At matched sampling budgets and comparable inference costs, we compare Sig-DEG with CM on 1D-Gm, rBerg, and MNIST.
The clearest separation occurs on rBerg: across one, five and ten
steps, CM's histogram losses are approximately two to four times
higher, while Sig-DEG achieves D-scores of $0.22$--$0.26$
versus $0.42$--$0.44$ and lower P-scores at every resolution
(Table~\ref{tab::rberg}).
On MNIST, Sig-DEG reduces FID by approximately $20\%$ at both
five and ten steps ($7.96$ and $8.28$ versus $10.04$ and
$10.27$), with higher IS and nearly identical inference times
(Table~\ref{tab:mnist}).
The advantage is strongest in marginal fidelity and the
discriminative and predictive scores, and correlation errors remain
mixed, with CM attaining lower point estimates on rBerg at
five and ten steps.

\paragraph{Sensitivity Analysis} As shown in Figure~\ref{fig:quality_vs_time}(b), Sig-DEG is less sensitive than CM to teacher resolution. On rBerg, replacing the 1000-step teacher by a separately trained 1500-step teacher changes Sig-DEG's histogram loss by at most $0.003$, whereas CM changes by $0.02$--$0.03$ across all $N_c$, improving at 5 and 10 steps but deteriorates markedly at $N_c=1$ ($0.094$ to $0.123$). This robustness is consistent with Sig-DEG's
continuous-time formulation.%

\subsection{Effects of Progressive Distillation and Curriculum Learning}
\begin{figure}[t]
    \centering
    \begin{subfigure}[t]{0.49\linewidth}
        \centering
        \includegraphics[width=\linewidth]{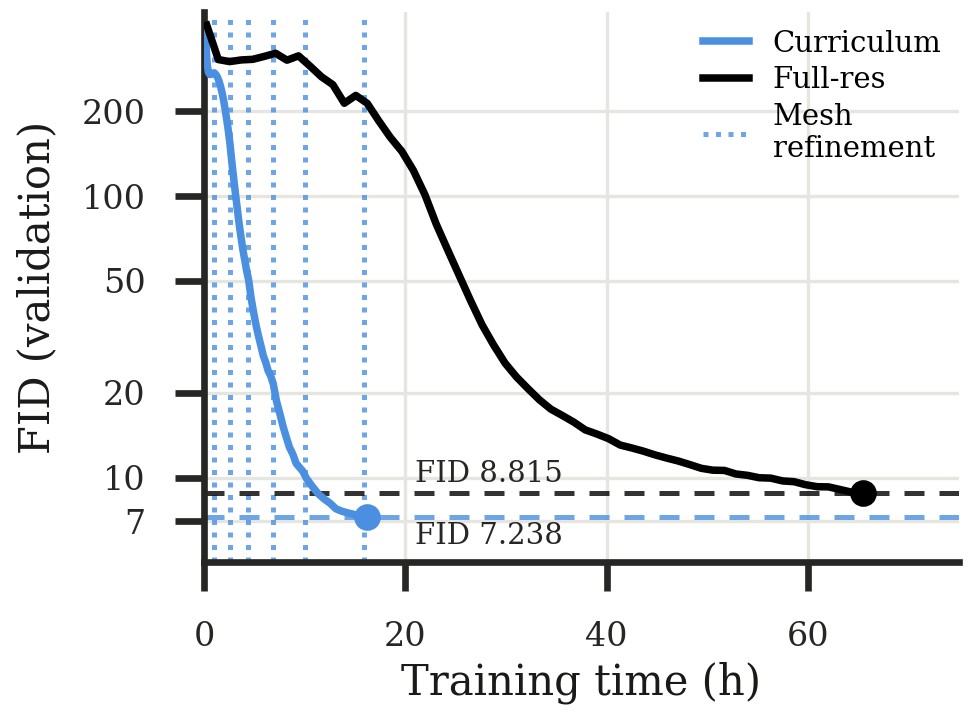}
        \caption{Curriculum}
        \label{fig:ladder_vs_flat}
    \end{subfigure}\hfill
    \begin{subfigure}[t]{0.49\linewidth}
        \centering
        \includegraphics[width=\linewidth]{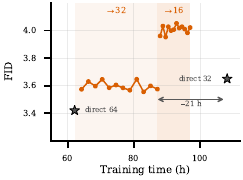}
        \caption{Progressive distillation}
        \label{fig:cifar_pd_curves}
    \end{subfigure}
    \caption{\textbf{CIFAR-10 training cost}. (a) Sig-DEG (8) trained on a curriculum vs.\ on full-resolution from the start, up to epoch $285$
    (total time: $16.2$ vs.\ $65.5$\,h). (b) FID of the PD chain $64 \to 32 \to 16$ per epoch (one evaluation each) vs.\ cumulative training time; stars: direct training (five-seed means).}
    \label{fig:cifar_cost}
\end{figure}

\paragraph{Curriculum Learning}
To evaluate curriculum learning, we compare it with training on the finest teacher resolution throughout on rBerg and CIFAR-10.
On rBerg, the curriculum cuts training time by $8$--$35\%$ across all sampling budget, with a similar or lower histogram loss at 2--10 steps (Table~\ref{tab:rberg_training}).
On CIFAR-10, we compare both methods at epoch \(285\), when the curriculum reaches the finest teacher mesh and both subsequently train on the same partition (Figure~\ref{fig:ladder_vs_flat}). The curriculum reaches this point \(4\times\) faster than direct training while also achieving a lower FID.

\paragraph{Progressive Distillation} We next assess whether progressive distillation can reduce training cost without sacrificing generation quality.
On CIFAR-10, the 32-step child distilled from Sig-DEG (64) closely matches the directly trained Sig-DEG (32) (FID $3.63$ vs.\ $3.65$), while reducing training time by $19\%$ (Figure~\ref{fig:cifar_pd_curves}). On rBerg, distilling Sig-DEG (10) to 5 and 2 steps both improves histogram loss and cuts training time by $12\%$ and $53\%$, respectively (Table~\ref{tab:rberg_pd}). These results show that a trained Sig-DEG can serve as an efficient teacher for further reducing the sampling budget.
\begin{figure}[h]
\includegraphics[width=0.49\linewidth]{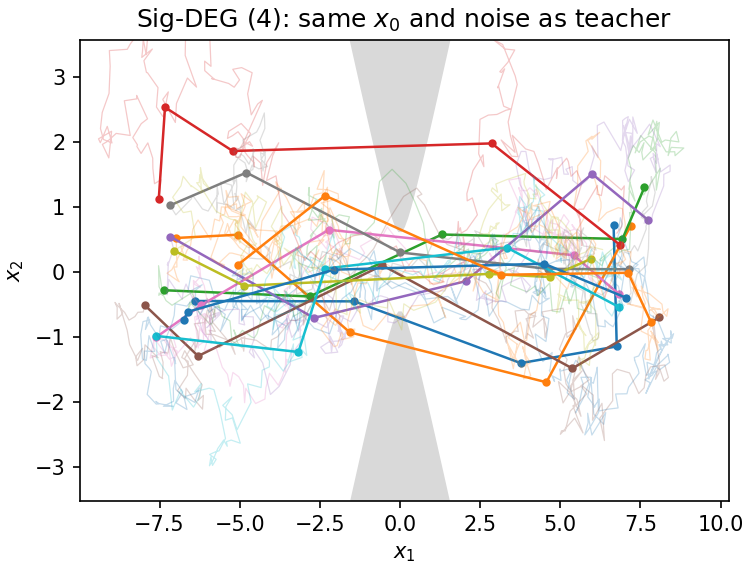}
\includegraphics[width=0.49\linewidth]{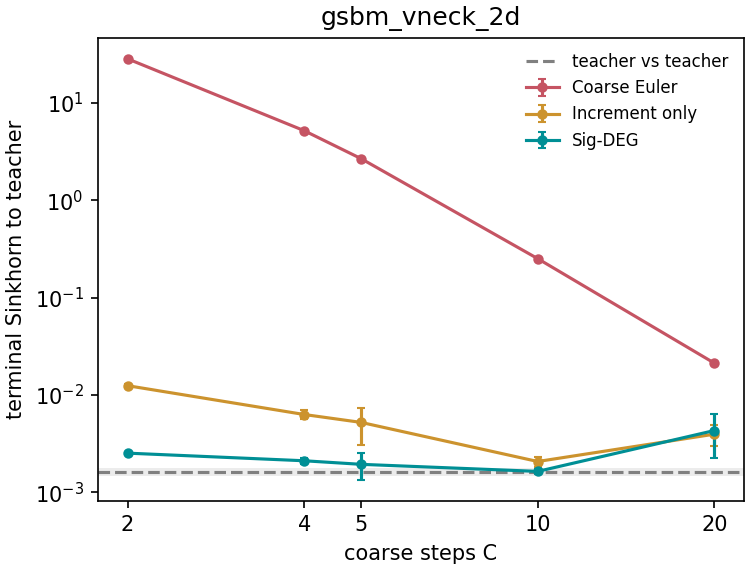}
    \caption{\textbf{V-neck task.} \emph{Left}: supervised trajectories (teacher's are shaded); connecting segments are shown only for visualisation and do not represent trajectories. \emph{Right}: performance of different coarse samplers.}
    \label{fig:sb1}
\end{figure}
\subsection{Schrödinger bridge distillation}

We consider three pre-trained two-dimensional Schr\"odinger bridge teachers  \citep{noble2026twisted}, with trajectories initialised from the source distribution. Their dynamics follow SDEs of the form~\eqref{eqn:forward}, allowing Sig-DEG to be applied directly.
Figure \ref{fig:sb1} shows that second-order partial signatures become increasingly important
as $N_c$ decreases.
Indeed, the Sig-DEG student's terminal divergence at $N_c=2$ is $1.6\times$ the teacher-vs-teacher floor, against $7.8\times$ for the increment-only one, and its joint
divergence is already at the floor ($0.0367$ vs $0.0366$, cf. Table \ref{tab:sb}).
The gap shrinks at higher values of $N_c$ and closes at $N_c=20$.
Directly sampling from the teacher at low resolution (\emph{Coarse Euler} in the figure) performs worse throughout, with the gap closing only on one teacher at $N_c=20$.
Appendix~\ref{app:sb} reports full results,
including source--terminal coupling errors and performance on
constrained multimodal transport.

\section{Conclusion}\label{sec:conclusion}

We introduced \textbf{Sig-DEG}, a modular, model-agnostic distillation framework combining partial path signatures with recurrent approximation of reverse-time SDEs. Grounded in stochastic Taylor expansions and rough path theory, it enables few-step generation with inference speedup of up to several hundred times, while matching or improving the teacher's fidelity on image and time series benchmarks. A resolution curriculum reduces training costs, while progressive distillation between Sig-DEG models further accelerates sampling. Beyond diffusion models, we demonstrate Sig-DEG on Schrödinger bridge distillation, highlighting its broader applicability to SDE-based generative models.

\textbf{Limitations \& Future Work.}
We would like to extend Sig-DEG to more complex tasks like high resolution image generation.
While we use fixed order partial signatures, learning the truncation, selection, or weighting of signature terms during training may further improve performance. Finally, incorporating higher-order stochastic features is another promising direction, with recent advances in cubature on Wiener space~\citep{foster2020numerical} and generative modelling of path functionals~\citep{jelinvcivc2023generative} providing tractable ways forward.

\section*{Acknowledgments and Disclosure of Funding}
WG is supported by the EPSRC through the UCL EPSRC Landscape Award (UELA) under Grant EP/Z534882/1. LJ is supported by the EPSRC [grant number EP/S026347/1], UKRI Prosperity Partnership Scheme (FAIR) under the EPSRC Grant EP/V056883/1 and
The Alan Turing Institute. HN is supported by the EPSRC [grant number EP/S026347/1],
The Alan Turing Institute under the EPSRC grant EP/N510129/1. HN and LR are supported by the EPSRC Programme Grant [Grant No. UKRI1010] entitled ``High order mathematical and computational infrastructure for streamed data that enhance contemporary generative and large language models''. NCK is supported by the Engineering and Physical Sciences Research Council [grant numbers EP/T517793/1, EP/W524335/1].

\bibliographystyle{plainnat}
\bibliography{references}

\newpage
\appendix

\raggedbottom
\clubpenalty=10000 \widowpenalty=10000
\makeatletter
\let\app@table\table \let\app@endtable\endtable
\let\app@figure\figure \let\app@endfigure\endfigure
\renewenvironment{table}[1][]{\nopagebreak\app@table[H]}{\app@endtable}
\renewenvironment{table*}[1][]{\begin{table}}{\end{table}}
\renewenvironment{figure}[1][]{\nopagebreak\app@figure[H]}{\app@endfigure}
\renewenvironment{figure*}[1][]{\begin{figure}}{\end{figure}}
\makeatother

\section{Preliminaries}\label{sec:path_prelims}
\subsection{Itô–Stratonovich Integral Conversion}\label{sec:ito_strat}

Fix $E=\mathbb{R}^d$. Let $X = (X_t)_{t \in [0, T]}$ be a semimartingale and $ f: \mathbb{R^{+}} \times E \rightarrow E$ a sufficiently smooth (e.g., $C^{1,2} $) function. Then the Itô and Stratonovich stochastic integrals are related by the following conversion formula:
\begin{eqnarray*}
    \int_0^T Y_{t-} \circ dX_t = \int_0^T Y_{t-} \cdot dX_t + \frac{1}{2}  [Y, X]^{c}_t,
\end{eqnarray*}
where $\circ$ denotes the Stratonovich integral and $\cdot$ denotes the It\^o integral and $[X, Y]^{c}_t$ is the continuous part of the covariation.
\begin{example}
    Let $\sigma: \mathbb{R}^{+} \rightarrow \mathbb{R}$ and $dX_t = \mu(t, X_t)dt+\sigma_t \circ dW_t$. Then $[\sigma_t, W_t]^{c}_t=0$ and hence
    $$dX_t = \mu(t, X_t)dt+\sigma_t \circ dW_t= \mu(t, X_t)dt+\sigma_t \cdot dW_t,$$
    where $\mu$ is sufficiently smooth.
\end{example}

\subsection{Stochastic Taylor Approximation}\label{app:taylor}
Consider the probability space $(C_{0}^{0}([0, T], E), \mathcal{F}, \mathbb{P})$, where $\Omega := C_{0}^{0}([0, T], E)$ is the space of $E$-valued continuous functions defined in $[0, T]$ starting at $0$ (i.e., the Wiener space), $\mathcal{F}$ is the Borel $\sigma$-field and $\mathbb{P}$ is the Wiener measure. By convention, we assume that a path $\omega \in C_{0}^{0}([0, T], E)$ is time-augmented, i.e., $\omega^{0}(t) = t$. We define the coordinate mapping process $W_t^{i}(\omega) = \omega^{i}(t), \forall t \in [0, T], i=1,\dots, d,\: \omega \in \Omega$. Then under the Wiener measure, $W = (W_t^1, \cdots, W_t^d)_{t \in [0, T]}$ is a $d$-dimensional Brownian motion starting at zeros. Moreover, $\bar{W}= (W^{0}_t, W_t)_{t \in [0, T]}$ denotes the time-augmented Brownian motion, where $W^{0}_t = t$.

We refer readers to \citet{kloeden1992stochastic} and \citep{milstein2013numerical} for a comprehensive treatment of stochastic Taylor expansions for general SDEs. In this section, we focus on the SDE of Eqn. \eqref{eqn: SDE_diffusion_model}, as the backward process of a general diffusion model belongs to this class of SDEs, where the volatility term $\sigma$ is a scalar function of time $t$ only. This implies that equivalently consider, as noticed in Section \ref{subsec: SDE_approx}, Eqn- \eqref{eqn:SDE} as

Let $Y:=(Y_{t})_{ t \in [0, T]}$ denote the solution to the following Stratonovich SDE
\begin{eqnarray}\label{eqn: SDE_diffusion_model}
    dY_{t} = \mu(t, Y_{t}) dt+ \sigma_t \cdot dW_t,
\end{eqnarray}
with initial condition $\quad Y_{0} = y \in E$, where $\mu \in \mathcal{C}([0, T] \times E, E)$ and $\sigma \in \mathcal{C}([0, T], \mathbb{R})$. The solution could be interpreted in the pathwise sense using rough path theory \citep{lyons2004cubature}.

We next explain the intuition behind the stochastic Taylor approximation.
By integrating both sides of Eqn. \eqref{eqn: SDE_diffusion_model} from $s$ to $t$, it holds that
\begin{equation}\label{eq:integral_expansion}
    \begin{aligned}
    Y_t -Y_s
    = {} &\int_{s}^{t} \mu(t_1, Y_{t_1})dt_1+ \int_{s}^{t}\sigma_{t_1} dW_{t_1} \\
    ={} & \int_{s}^{t}\mu(s,Y_s)dt_1+\int_{s}^{t}\sigma_s dW_{t_1}  + \int_{s}^{t}\int_{s}^{t_1} \partial_t \sigma_{t_2}dt_2 dW_{t_1} \\
    & +\int_{s}^t\int_{s}^{t_1}\nabla\mu(t_2, Y_{t_2})\sigma_{t_2} dW_{t_2} dt_1 \\
    & +\int_{s}^t\int_{s}^{t_1}\nabla\mu(t_2, Y_{t_2})\mu(t_2, Y_{t_2}) dt_2 dt_1
    + \int_{s}^t\int_{s}^{t_1}\partial_t\mu(t_2, Y_{t_2})dt_2dt_1 \\
    & + \tfrac12\int_s^t \int_s^{t_1}\sigma_{t_2}^2\Delta\mu(t_2,T_{t_2})dt_2dt_1 .
\end{aligned}
\end{equation}

Note that $\sigma$ is a function only depending on $t$ and its derivatives with respect to $y$ vanish.

Assume that $\mu$ has bounded derivatives of all orders and $\sigma$ has bounded derivatives of at least order 3. Then $Y$ admits the following stochastic Taylor expansion:
\begin{equation}
    \label{eqn:local_epansion_dev}
    \begin{aligned}
        Y_t ={} & Y_s + \mu(s, Y_s) (t-s) + \sigma_s(W_t-W_s) \\
        & + \nabla \mu(s,Y_s) \mu(s,Y_s) \tfrac{(t-s)^2}{2}
        + \partial_t \mu(s,Y_s) \tfrac{(t-s)^2}{2}             \\
        & + \tfrac12\sigma_s^2\Delta\mu(s,Y_s),\tfrac{(t-s)^2}{2} \\
        & + \nabla\mu(s, Y_{s})\sigma_{s} \int_{s}^t\int_{s}^{t_1}dW_{t_2} dt_1
        + \partial_t \sigma_s \int_{s}^{t} \int_{s}^{t_1} dt_2dW_{t_1}
        + R(t-s) ,
    \end{aligned}
\end{equation}
with the remainder $R$ has a (local) error of strong order $\mathcal{O}((t-s)^{5/2})$.
Note that, by integration by parts, $\int_{s}^t \int_{s}^{t_1} dt_2dW_{t_1}$ can be rewritten as
\begin{align*}
    \int_{s}^t (t_1-s) dW_{t_1} & =  (t_1-s)W_{t_1}|_{t_1=s}^{t}-  \int_{s}^t W_{t_1} dt_1=(t-s)W_t - \int_{s}^t W_{t_1} dt_1 \\
                                & = (t-s)(W_t-W_s) - \int_{s}^t (W_{t_1}-W_s) dt_1.                                           %
\end{align*}
Therefore, the dominant term in Eqn. \eqref{eqn:local_epansion_dev} is a function of $Y_s, \, s$, and $PS(W)_{s, t}$, where $PS(W)_{s, t} = (t-s, W_t-W_s, \int_s^{t}\left(W_{t_1}-W_s\right)dt_1)$.
Therefore, this yields the local expansion:
\begin{eqnarray}
    Y_t - Y_s = F(s, Y_s, PS(W_{s,t}))+R(t-s),
\end{eqnarray}
where the function $F: \mathbb{R}^{+} \times E \times \mathbb{R}^{+} \times E \times E \rightarrow E$ is defined as
\begin{equation}\label{eqn: F_def}
    \begin{aligned}
    F(s, y, s_1, s_2, s_{21}) ={} & \mu(s,y)s_1+\sigma_s s_2
    + \frac12 \big( \mu(s,Y_s) \nabla \mu(s,Y_s)
    + \partial_t \mu(s,Y_s) \big) s_1^2 \\
    & +\partial_t\sigma_s (s_1 s_2-s_{21})  + \nabla \mu(s,y) \sigma_s s_{21}
    + \tfrac12\sigma_s^2\Delta\mu(s,Y_s).
    \end{aligned}
\end{equation}

\textbf{Numerical approximation of SDEs}.
Consider a time partition ${D}_{c}:=\{t_0, \cdots, t_{N_c}\}$ of the time interval $[0, T]$, where $\delta = \max_i |t_{i+1}-t_i| < 1$.
Recall that we adopt a higher-order numerical scheme to approximate the solution $Y_T$ as follows. First set $\hat{Y}_{t_0} = Y_0$ and for $i \in \{0, \cdots, N_c-1\}$, let
\begin{eqnarray}\label{eqn:F_numercial_PS}
    \hat{Y}_{t_{i+1}} = \hat{Y}_{t_{i}} + F\left(t_{i}, \hat{Y}_{t_i}, PS(W)_{t_{i}, t_{i+1}} \right),
\end{eqnarray}
$\hat{Y}_{T}$ will be our numerical approximation of the solution $Y_T$ for the SDE \eqref{eqn: SDE_diffusion_model}.

\begin{theorem}[Numerical approximation theorem]\label{Thm: SDE_approx}
    Consider a drift of the form $\mu(t,y)=A(t)y+g(t,y)$, a diffusive term $\sigma(t)$ and assume that the following hold.
    \begin{enumerate}
        \item[(H1)] $\sigma, A\in C^2([0,T])$.
        \item[(H2)] $g\in C^{2,0}\cap C^{1,2}\cap C^{0,3}$.
        \item[(H3)] $\nabla g$ and $\nabla^2g$ are bounded; moreover, $g$, $\nabla^3g$, $\partial_tg$, $\partial_t\nabla g$ have at most linear and $\partial_t^2g$, $\partial_t\nabla^2g$ have at most polynomial growth in $y$, uniformly in $t$.
        \item[(H4)] $Y_0$ is independent of $W$, $\E|Y_0|^p<\infty$ for every $p$, and $\hat Y_0=Y_0$.
    \end{enumerate}
    Then there exists a constant $C$, independent from the partition $D_c$ (i.e., from $\delta$ and $N_c$) such that
    \[
        \max_{0\le i\le N}\big\|Y_{t_i}-\hat Y_i\big\|_{L^2}\le C\,\delta^{3/2}.
    \]
\end{theorem}

Before the above theorem, which is a key result to ensure the quality of our approximation, we need to state some preliminary results. For the sake of notation, we define here the generator
\[
    \gen = \partial_t + \mu \cdot \nabla + \tfrac12 \sigma_s^2 \Delta.
\]

The first Lemma \ref{lemma:moment_exact} is needed to bound the solution of our SDE \eqref{eqn: SDE_diffusion_model}, while Lemma \ref{lemma:moment_scheme} bounds its approximation.

\begin{lemma}\label{lemma:moment_exact}
    Let $Y^{s,\xi}$ be a solution of Eqn. \eqref{eqn: SDE_diffusion_model}, started at time $s$ from an $\mathcal F_t$-measurable initial condition $\xi$ with finite moments.
    $\forall p\ge2$ and $t\in[s,T]$,
    \begin{equation}\label{eq:moment_exact}
        \E\big[ \sup_{r\in[s,T]}|Y^{s,\xi}_t|^p \big] \leq C_p\big(1+\E|\xi|^p\big)
        \quad\text{and}\quad
        \E\big[|Y^{s,\xi}_t-\xi|^p\big] \leq C_p\big(1+\E[|\xi|^p]\big)(t-s)^{p/2}.
    \end{equation}
\end{lemma}
\begin{proof}
    let $Y=Y^{s,\xi}$ and $S_t=\sup_{r\in[s,t]}|Y_r|$.
    Given the integral formulation of \eqref{eqn: SDE_diffusion_model}, for $r\leq t$, $Y_r=\xi+\int_s^r\mu(u,Y_u)\,du+\int_s^r\sigma\,dW$ and (by (H3))
    \[
        |Y_r|\le|\xi|+C\int_s^t\big(1+|Y_u|\big)\,du+\sup_{t'\le t}\Big|\int_s^{t'}\sigma_u\,dW_u\Big| .
    \]
    After raising to the power $p$, by H\"older inequality
    \[
        \Big(C\int_s^t(1+|Y_u|)\,du\Big)^p
        \leq C\int_s^t\big(1+S_u^p\big)\,du ,
    \]
    and by Burkholder--Davis--Gundy inequality
    \[
        \E\sup_{t'\le t}\Big|\int_s^{t'}\sigma_u\,dW_u\Big|^p
        \leq C_p\,\E\Big(\int_s^t\sigma_u^2\,du\Big)^{p/2}
        \leq C_p\,\|\sigma\|_\infty^p\,(t-s)^{p/2}
        \leq C .
    \]
    By taking the supremum over $s\leq r\leq t$ and the expectation, eventually
    \[
        \E \big[ S_r^p \big]
        \leq C\big(1+\E\big[ |\xi|^p \big]\big)
        + C \int_s^t \E \big[ S_u^p \big] \,du ,\qquad t\in[s,T] .
    \]
    At this point, Gr\"onwall's inequality yields
    $\E \big[ S_T^p \big] \leq C(1+\E|\xi|^p)$
    .

    As for the second bound, it follows from the definition of $Y$ that
    \[
        \E\big[ |Y_t-\xi|^p \big]
        \lesssim \E \Big[ \Big| \int_s^t \mu(u,Y_u) \,du \Big|^p \Big]
        + \E \Big[ \Big| \int_s^t \sigma_u \,dW_u \Big|^p \Big] .
    \]
    Now, one can apply H\"older and the first bound to show that
    \[
        \begin{aligned}
        \E\Big[\Big|\int_s^t\mu\,du\Big|^p\Big]
        &\leq (t-s)^{p-1}\int_s^t \E \big[ |\mu(u,Y_u)|^p \big] \,du \\
        &\leq C(t-s)^p \big( 1 + \E \big[S_T^p\big] \big)
        \leq C ( 1 + \E[|\xi|^p] ) (t-s)^p ,
        \end{aligned}
    \]
    and, by BDG, that
    $\E [ |\int_s^t \sigma_u \,dW_u |^p] \leq C(t-s)^{p/2}$. Since
    $(t-s)^p \leq T^{p/2}(t-s)^{p/2}$, the thesis follows.
\end{proof}

We can now prove an analogous results for the moments of $\hat{Y}$.

\begin{lemma}\label{lemma:moment_scheme}
    The scheme illustrated in Eqn. \eqref{eqn:F_numercial_PS} satisfies $\max_i\E[|\hat Y_i|^{p}] \leq C_p$ for every $p$, uniformly over partitions with $\delta\leq1$.
\end{lemma}
\begin{proof}
    Let us rewrite $\hat{F}_i \coloneqq F\big(t_{i}, \hat{Y}_{t_i}, PS(W)_{t_{i}, t_{i+1}} \big) = a_i \delta_{i} + b_i$, where $a_i=\mu(t_i,\hat{Y}_i) + \frac{\delta_{i}}{2}\gen\mu_i$, $\delta_i = t_{i+1} - t_i$ and
    $b_i$ contains the terms depending on stochastic increments.
    By (H3), $|a_i|\leq C(1+|\hat Y_i|)$;
    also, conditionally on $\mathcal F_{t_i}$, $b_i$ is centred Gaussian with covariance $\leq Ch_i$ (this follows from $\nabla\mu$ being bounded and from lemma \ref{lemma1}). Hence, for every $k$,
    \begin{equation}\label{eq:F-moments}
        \big|\E[\hat{F}_i\mid\mathcal F_{t_i}]\big|\leq C \delta_i(1+|\hat Y_i|),
        \quad\text{and}\quad
        \E\big[|F_i|^{k}\mid\mathcal F_{t_i}\big]\leq C_k \delta_i^{k/2} (1+|\hat Y_i|)^k .
    \end{equation}
    Now, consider the expansion $|\hat Y_{i+1}|^{2m}=(|\hat Y_i|^2+\eta_i)^m$, with $\eta_i=2\langle\hat Y_i,\xi_i\rangle+|\xi_i|^2$. We can use \eqref{eq:F-moments} to bound the linear-in-$\eta$ term and the ones containing $\eta^k$, obtaining respectively
    \[
        \E\big[|\hat{Y}_i|^{2(m-1)}\eta\mid\mathcal{F}_{t_i}]
        \leq C\delta_i(1+|\hat{Y}_i)^2|\hat{Y}_i^{2(m-1)}
    \]
    and
    \begin{align*}
        \E\big[|\hat{Y}_i|^{2(m-k)}\eta^k\mid\mathcal{F}_{t_i}]
         & \leq C|\hat Y_i|^{2(m-k)}\big(|\hat Y_i|^{k}\E[|\xi_i|^k\mid\mathcal{F}_{t_i}]+\E[|\xi_i|^{2k}\mid\mathcal{F}_{t_i}]\big) \\
         & \leq C\delta_i^{k/2}(1+|\hat Y_i|)^{2k} |\hat Y_i|^{2(m-k)}
        \leq C\delta_i(1+|\hat Y_i|)^{2k} |\hat Y_i|^{2(m-k)} .
    \end{align*}
    Therefore,
    \[
        \E\big[|\hat Y_{i+1}|^{2m}\mid\mathcal F_{t_i}\big]
        \leq (1+C\delta_i)|\hat Y_i|^{2m}+C\delta_i
    \]
    and, taking expectations, $M_i:=\E|\hat Y_i|^{2m}$ satisfies
    \[
        M_{i+1}
        \leq (1+C\delta_i)\,M_i+C\delta_i,
        \qquad i=0,\dots,N-1 .
    \]
    We can then conclude by adding $1$ to both sides, so that $1+M_{i+1}\le(1+C\delta_i)(1+M_i)$ and by Gr\"onwall
    \[
        1+M_n
        \leq e^{CT}\big(1+\E\big[|Y_0|^{2m}\big]\big) .
    \]
    Hence, $\max_i\E[|\hat Y_n|^{2m}] \leq e^{CT}(1+\E|Y_0|^{2m})$. In case of odd $p$, the results is easily extended via Jensen's inequality, bringing inside the expectation any exponent $2m/p$, for $2m>p$.
\end{proof}

The last preliminary result necessary to prove Theorem \ref{Thm: SDE_approx}, and it is a bound on the remainder of the stochastic Taylor approximation.
Once again, we assume that the hypothesis $(H1)--(H4)$ hold.

For $0\leq t < t+\delta \leq T$, let us enhance the notation $R(t-s)$ from Eqn. \eqref{eqn: SDE_diffusion_model} by writing $R(s,y, \delta)$, so that it accounts for the time-state center of the Taylor expansion, with $t-s=\delta$.

\begin{lemma}\label{lemma:remainder}
    There exists $q_0\geq1$ such that
    \begin{equation}\label{eq:remainder_estimate}
        \|R(s,y,\delta)\|_{L^2}
        \leq C(1+|y|^{q_0})\,\delta^{2},
        \qquad
        \big| \E [ R(s,y,\delta) ] \big|
        \leq C(1+|y|^{q_0})\,\delta^{5/2}.
    \end{equation}
\end{lemma}

\begin{proof}
    Let $Y=Y^{s,y}$. First of all, we show that
    \begin{equation}\label{eq:rho-decomp}
        R(s,y,\delta) = R_\sigma + R_1 + R_2 ,
    \end{equation}
    where in particular
    \begin{align*}
        R_\sigma & =\int_s^{s+\delta}\big[\sigma_u-\sigma_s-\partial_t\sigma_s(u-s)\big]\,dW_u,                    \\
        R_1      & =\int_s^{s+\delta}(s+\delta-u)\big[\gen\mu(u,Y_u)-\gen\mu(s,y)\big]\,du,                        \\
        R_2      & =\int_s^{s+\delta}(s+\delta-u)\big[\nabla\mu(u,Y_u)\sigma_u-\nabla\mu(s,y)\sigma_s\big]\,dW_u .
    \end{align*}
    Starting from the integral formulation of Eqn. \eqref{eqn: SDE_diffusion_model},
    consider an expansion of the two integrals around $(s,y)$, like in Eqn. \eqref{eq:integral_expansion}. The stochastic integral, given $r_\sigma = \sigma_u-\sigma_s-\partial_t\sigma_s(u-s)$, yields exactly $R_\sigma$. Similarly, by It\^o formula, we have that for $u\in[t,t+h]$
    \[
        \mu(u,Y_u)=\mu(s,y)+\int_s^u\gen\mu(r,Y_r)\,dr+\int_s^u\nabla\mu(r,Y_r)\,\sigma(r)\,dW_r .
    \]
    Integrating in $u$ over $[s,s+\delta]$, and exchanging the order of integration via Fubini-Tonelli,
    \[
        \begin{aligned}
        \int_s^{s+\delta}\mu(u,Y_u)\,du
        ={}& \mu(s,y)\,\delta
        + \int_s^{s+\delta} (s+\delta-u)\,\gen\mu(u,Y_u)\,du \\
        & + \int_s^{s+\delta} (s+\delta-u)\,\nabla\mu(u,Y_u)\sigma_u\,dW_u .
        \end{aligned}
    \]
    Adding and subtracting $\gen\mu(s,y)\frac\delta2$ and $\nabla\mu(s,y)\sigma_s \big( \delta W_{s,s+\delta}-\int_s^{s+\delta}W_{s,u}\,du \big)$ then gives
    \[
        \int_s^{s+\delta}\mu(u,Y_u)\,du
        = \mu(s,y)\,\delta
        + \gen\mu(t,x) \tfrac{\delta^2}{2}
        + \nabla\mu(t,x)\sigma(t)\Big( \delta W_{s,s+\delta}-\int_s^{s+\delta}W_{s,u}\,du \Big)
        + R_1
        + R_2 .
    \]
    It is now apparent that the remainder can be split exactly as in \eqref{eq:rho-decomp}.

    We now need to prove the desired bounds. Since the two stochastic integrals have zero mean, the second bound depends exclusively on $R_1$. In order to attain it, we consider separately the space and time increments of $\gen \mu$, i.e.
    \[
        \gen\mu(u,Y_u) - \gen\mu(s,y)
        = \Big[ \gen\mu(u,Y_u)-\gen\mu(u,y) \Big]
        + \Big[ \gen\mu(u,y)-\gen\mu(s,y) \Big] .
    \]
    By (H3), even though $\gen \mu(u,\cdot)$ is not globally Lipschitz, we can use the mean value theorem to say that
    \[
        \big|\gen\mu(u,Y_u)-\gen\mu(u,x)\big|
        \leq \sup_{x\in[y,Y_u]}|\nabla\gen\mu(u,x)|\;|Y_u-y|
        \leq C\big(1+|Y_u|+|y|\big)\,|Y_u-y| .
    \]
    Then, after taking the $L^2$ norm, Cauchy-Schwarz and Lemma \ref{lemma:moment_exact} for $p=4$ give
    \[
        \big\| \gen\mu(u,Y_u)-\gen\mu(u,x) \big\|_{L^2}
        \leq C(1+|y|)^2\,(u-s)^{1/2}.
    \]
    Regarding the time difference above, (H3) grants the existence of some $q$ such that
    \[
        \big\| \gen\mu(u,y)-\gen\mu(s,y) \big\|_{L^2}
        \leq \sup_{r\in[s,u]} | \partial_t \gen\mu(r,y) | \, (u-s)
        \leq C(1+|y|^q)\,(u-s)^{1/2} .
    \]
    The second bound in \eqref{eq:remainder_estimate} is then obtained by summing the two terms with $q_0 = q+2$, and considering that
    \begin{align*}
        \big| \E [ R(s,y,\delta) ] \big|
        \leq \big\| R_1 \big\|
         & \leq \int_s^{s+\delta}(s+\delta-u)\,\big\|\gen\mu(u,Y_u)-\gen\mu(s,y)\big\|_{L^2}\,du \\
         & \leq C(1+|y|^{q_0})\int_s^{s+\delta}(s+\delta-u)(u-s)^{1/2}\,du
        = C(1+|y|^{q_0})\,\delta^{5/2} .
    \end{align*}

    We are left with the $L^2$ estimates of $R_\sigma$ and $R_2$. Since $\sigma\in C^2$ and $r_\sigma$ is $\mathcal{O}\big((u-s)^2\big)$, It\^o isometry grants that $\E[|R_\sigma|^2] \leq C\delta^5$. As for $R_2$, we can add and subtract terms and get
    \[
        [\nabla\mu(u,Y_u) - \nabla\mu(u,y)] \sigma_u
        + [\nabla\mu(u,y) - \nabla\mu(s,y)] \sigma_r
        + \nabla\mu(s,y) [\sigma_r - \sigma_s] .
    \]
    The first difference is has second moments bounded, thanks to (H1), (H3) and Lemma \ref{lemma:moment_exact}, by $C(1+|y|^2)(u-s)$; the second one by $C(1+|y|^2)(u-s)^2$ thanks to (H1) and (H3); the third one by $C(u-s)^2$ thanks to (H1). Again by It\"o isometry, then,
    \[
        \E \big[ |R_2|^2 \big]
        \leq C(1+|y|^2) \int_s^{s+\delta} (s+\delta-u)^2 (u-s)\,dr
        \leq C(1+|y|^2)\,\delta^4 ,
    \]
    and the first bound in \eqref{eq:remainder_estimate} follows.
\end{proof}

We now have all the elements to prove Theorem \ref{Thm: SDE_approx}. In particular, since $R(s,y,\delta)$ is a measurable function of $y$ and $(W_{s,r})_{r\in[s,s+\delta]}$, that are independent from $\mathcal{F}_t$, Lemma \ref{lemma:remainder} and Lemma \ref{lemma:moment_exact} grant that
\begin{equation}\label{eq:freeze}
\begin{gathered}
    \big|\E [R(t_i,Y_{t_i},\delta_i) \mid \mathcal F_{t_i}] \big|
    \leq C (1+|Y_{t_i}|^{q_0})\,\delta_i^{5/2},
    \quad\text{and}\\
    \E \big[|R(t_i,Y_{t_i},\delta_i)|\big]^2
    \leq C\delta_i^4\,\E[(1+|Y_{t_i}|^{q_0})^2]
    \leq C\delta_i^4 .
\end{gathered}
\end{equation}

\begin{proof}[\textbf{Proof of Theorem \ref{Thm: SDE_approx}}]
    Let $\err_i=Y_{t_i}-\hat Y_i$ be the error at the mesh point $t_i$. Given that
    \[
        Y_{t_{i+1}} = F(t_i,Y_{t_i},\delta_i)+R(t_i,Y_{t_i},\delta_i)
        \quad\text{and}\quad
        \hat Y_{i+1} = F(t_i,\hat Y_i,\delta_i) ,
    \]
    and that $\sigma_{t_i} W_{t_i, t_{i+1}}$ and
    $\partial_t \sigma_{t_i}\big( \delta W_{t_i,t_{i+1}}-\int_{t_i}^{t_{i+1}}W_{t_i,u}\,du \big)$ do not depend on the state and cancel, $\err_{i+1}$ reduces to
    \begin{equation}\label{eq:rec}
        \begin{aligned}
            \err_{i+1} ={} &
            \err_i
            + \delta_i \big[\mu(t_i,Y_{t_i})-\mu(t_i,\hat Y_i)\big]
            + \tfrac{\delta_i^2}{2} \big[\gen\mu(t_i,Y_{t_i})-\gen\mu(t_i,\hat Y_i)\big]                                                                                         \\
                           & + \sigma_i\big[\nabla\mu(t_i,Y_{t_i}) - \nabla\mu(t_i,\hat Y_i)\big] \Big( \delta W_{t_i,t_{i+1}} - \int_{t_i}^{t_{i+1}}W_{t_i,u}\,du \Big)
            + R(t_i,Y_{t_i},\delta_i)                                                                                                                                            \\
            ={}            & I_i^{(1)} + I_i^{(2)} + R(t_i,Y_{t_i},\delta_i) ,
        \end{aligned}
    \end{equation}
    where $I_i^{(2)} \coloneqq \sigma_i\big[\nabla\mu(t_i,Y_{t_i}) - \nabla\mu(t_i,\hat Y_i)\big] \big( \delta W_{t_i,t_{i+1}} - \int_{t_i}^{t_{i+1}}W_{t_i,u}\,du \big)$ is $\mathcal{F}_{t_i}$-independent and centered, hence with zero conditinoal mean, and $I_i^{(1)}$, containing the other terms, if $\mathcal{F}_{t_i}$-measurable.

    The goal of this proof is to show that
    \begin{equation}\label{eq:goal}
        \E \big[|\err_{i+1}|^2\big]
        \leq (1+C\delta_i)\,\E \big[|\err_i|^2\big] + C\delta_i^4 ,
        \qquad \forall i=0,\dots,N-1 ,
    \end{equation}
    and conclude with Gr\"onwall: for $M_i=\E[|\err_i|^2]$, and noticing that $M_0$ is trivially $0$,
    \[
        M_n
        \leq \sum_{i<n} C \delta_i^4 \prod_{i<j<n} (1+C\delta_j)
        \leq \Big( \prod_{j<N} (1+C\delta_j) \Big) \, C \sum_{i<N} \delta_i^4
        \leq e^{CT}\,C\,\delta^3 \sum_{i<N} \delta_i
        = e^{CT} CT\,\delta^3 .
    \]
    It would follow that $\max_n\|\err_n\|_{L^2} \leq C\delta^{3/2}$, which is the thesis.

    In order to prove Eqn. \eqref{eq:goal}, let us consider first the $L^2$ norm of Eqn. \eqref{eq:rec}, i.e.
    \[
        \E \big[|\err_{i+1}|^2\big]
        = \E \big[|I_i^{(1)} + I_i^{(2)}|^2\big]
        + 2\,\E \big[\langle I_i^{(1)} + I_i^{(2)},R_i,\rangle\big]
        + \E \big[|R_i|^2\big] ,
    \]
    where we wrote $R_i \coloneqq R(t_i,Y_{t_i},\delta_i)$ for the sake of readibility, and bound the three terms on the right-hand side separately.

    Starting from the first one, the two differences in $I_i^{(1)}$ are controlled by the Lipschitz properties (due to (H3)) of $\mu$ and $\gen\mu$: $\mu(t_i,\cdot)$ is globally $L$-Lipschitz, so $\big|\mu(t_i,Y_{t_i})-\mu(t_i,\hat Y_i)\big| \leq L\,|\err_i|$; then $|\nabla\gen\mu(t_i,z)|\leq C(1+|z|)$, so the mean value theorem gives $\big|\gen\mu(t_i,Y_{t_i})-\gen\mu(t_i,\hat Y_i)\big| \leq C\,Z_i\,|\err_i|$, for $Z_i:=1+|Y_{t_i}|+|\hat Y_i|$, since $|z|\leq|Y_{t_i}|+|\hat Y_i|$.
    Notice that the second bound is only locally Lipschitz, as the constant $CZ_i$ is random. Eventually, one can rewrite $|I_i^{(1)}| \leq a+b$, with $a=(1+L\delta_i)|\err_i|$ and $b=C\delta_i^2 Z_i|\err_i|$. Squaring and applying Young's inequality with coefficient $\delta_i$ then yields
    \begin{equation}\label{eq:Dsq}
        |I_i^{(1)}|^2
        \leq (1+\delta_i)\,a^2 + (1+\delta_i^{-1})\,b^2
        \leq (1+C\delta_i)\,|\err_i|^2+C\delta_i^3\, Z_i^2 |\err_i|^2 .
    \end{equation}
    The second term on the right-hand side is handled pointwise: if $\delta_iZ_i\le1$ then $\delta_i^3 Z_i^2 |\err_i|^2 \leq \delta_i |\err_i|^2$; if $\delta_i Z_i > 1$ then $\delta_i^3 Z_i^2 |\err_i|^2 < \delta_i^4 Z_i^3 |\err_i|^2$. Therefore, in all cases
    \begin{equation}\label{eq:split}
        \delta_i^3Z_i^2|\err_i|^2
        \leq \delta_i|\err_i|^2+\delta_i^4Z_i^3|\err_i|^2,
        \qquad\text{and}\qquad
        \E\big[\delta_i^4Z_i^3|\err_i|^2\big]
        \leq \delta_i^4\,\|Z_i\|_{L^6}^3\,\|\err_i\|_{L^4}^2
    \end{equation}
    by Cauchy--Schwarz.
    Now, by Minkowski inequality and Lemmas \ref{lemma:moment_exact} and \ref{lemma:moment_scheme}, for every $p\geq2$ we have that
    \begin{gather*}
        \|Z_i\|_{L^p}
        \leq 1 + \|Y_{t_i}\|_{L^p} + \|\hat Y_i\|_{L^p}
        \leq 1 + C_p(1+\E\big[|Y_0|^p\big]) + C_p ,
        \\
        \|\err_i\|_{L^p}
        \leq \|Y_{t_i}\|_{L^p} + \|\hat Y_i\|_{L^p}
        \leq C_p(1+\E\big[|Y_0|^p\big]) + C_p ,
    \end{gather*}
    hence
    \begin{equation*}
        \|Z_i\|_{L^p}+\|e_i\|_{L^p}
        \leq C_p
        \qquad \text{for all } i\leq N .
    \end{equation*}
    It follows that the second inequality of \eqref{eq:split} becomes $\E\big[\delta_i^4Z_i^3|\err_i|^2\big]  \leq C\delta_i^4$, and
    \begin{equation}\label{eq:Ddet}
        \E\big[|I_i^{(1)}|^2\big]
        \leq (1+C\delta_i)\,\E\big[|\err_i|^2\big] + C \delta_i^4 .
    \end{equation}

    Moving on to the parts involving the noise, $\mathcal{F}_{t_i}$-measurability of $I_i^{(1)}$, together with $\mathcal{F}_{t_i}$-independence and zero-mean of $I_i^{(2)}$, ensure that $\E[\langle I_i^{(1)},I_i^{(2)}\rangle\mid\mathcal F_{t_i}]=0$. We can use instead (H3) to obtain
    \[
        \E\big[|I_i^{(2)}|^2\mid\mathcal F_{t_i}\big]=\sigma_i^2\,\tfrac{h_i^3}{3}\,\big|\nabla\mu(t_i,Y_{t_i})-\nabla\mu(t_i,\hat Y_i)\big|
        \le C\delta_i^3\,|\err_i|^2 .
    \]
    Together with \eqref{eq:Ddet},
    \[
        \E|I_i^{(1)}+I_i^{(2)}|^2=\E|I_i^{(1)}|^2+\E|I_i^{(2)}|^2\le(1+Ch_i)\,\E|e_i|^2+Ch_i^4 .
    \]

    The last term is the cross-product with the remainder. By \eqref{eq:freeze}, and applying again Young inequality with coefficient $\delta_i$,
    \begin{align*}
        2\,\E\big[\langle I_i^{(1)},R_i\rangle\big]
         & = 2\, \E\big[\big\langle I_i^{(1)},\E[R_i\mid\mathcal F_{t_i}]\big\rangle\big]
        \leq \delta_i\,\E\big[|I_i^{(1)}|^2\big] + \delta_i^{-1}\,\E\big[\big|\E[R_i\mid\mathcal F_{t_i}]\big|^2\big] \\
         & \leq \delta_i\,\E\big[|I_i^{(1)}|^2\big] + C\delta_i^{4}\,\E\big[(1+|Y_{t_i}|^{q_0})^2\big]
        \leq C\delta_i\,\E|\err_i|^2 + C\delta_i^4 ,
    \end{align*}
    using \eqref{eq:Ddet} and Lemma \ref{lemma:moment_exact} for the last step.
    For the noise part Cauchy--Schwarz is enough, since by $\|I_i^{(2)}\|_{L^2} \leq C\delta_i^{3/2}\|\err_i\|_{L^2}$ and, by \eqref{eq:freeze}, $\|R_i\|_{L^2} \leq Ch_i^2$:
    \[
        2\,|\E\big[\langle I_i^{(2)},R_i\rangle|\big]
        \leq 2\|I_i^{(2)}\|_{L^2}\|R_i\|_{L^2}
        \leq \delta_i^3\,\E\big[|\err_i|^2\big] + C\delta_i^4 .
    \]
    Eventually, combining the estimates obtained here and $\E|R_i|^2 \leq C\delta_i^4$ from \eqref{eq:freeze} gives \eqref{eq:goal}.
\end{proof}

\subsection{The Signature of a Path}
In this subsection, we introduce the definition of path signatures, with a focus on the Stratonovich signature of Brownian motion. We present key properties that are directly relevant to our proposed method. For a comprehensive treatment of path signatures and their desirable mathematical properties as principled and efficient feature representations for time series, we refer readers to \cite{lyons2014rough, levin2013learning, lyons2025signature}.
\subsubsection{Path Space}
Let $E= \mathbb{R}^d$, equipped with its canonical basis $\left\{e_{1},\ldots ,e_{d}\right\}$. Consider a continuous path $X: J \rightarrow E$, where $J$ is a compact time interval. We introduce the $p$-variation of the path, which is a commonly used measure to quantify the smoothness of the path.
\begin{definition} [$p$-variation]\label{p_variation}
Let $p\geq 1$ be a real number. Let $X: J \rightarrow E$ be a continuous path. The $p$-variation of $X$ on the interval $J$ is defined by
\begin{equation}
    \vert\vert X\vert\vert_{p,J}=\left[ \sup_{{D}\subset J}\sum_{j=0}^{r-1}\left\vert X_{t_{j+1}}-X_{t_j}\right\vert^p\right]^{\frac{1}{p}},
\end{equation}
where the supremum is taken over any time partition of $J$, i.e. ${D} = (t_{1}, t_{2}, \cdots, t_{r})$.
\end{definition}
Let $\Omega_{p}([0, T], \mathbb{R}^{d})$ denote the space of continuous $d$-dimensional paths of finite $p$-variation starting from the origin. We endow $\Omega_{0}([0, T], \mathbb{R}^{d})$ with the $p$-variation metric.
For each $p \geq 1$, the $p$-variation norm of a path $X \in \mathcal{C}_{p}(J, E)$ is denoted by $\vert\vert X \vert\vert_{p-var}$ and defined as follows:
\begin{eqnarray*}
	\vert \vert X \vert\vert_{p-var} = \vert\vert X \vert\vert_{p, J} + \sup_{t \in J} \vert\vert X_{t}\vert\vert.
\end{eqnarray*}

\subsubsection{Tensor Algebra Space and the Signature}
We then introduce the tensor algebra space of $E=\mathbb{R}^{d}$, where the signature of an $E$-valued path takes values. Consider the successive tensor powers $E^{\otimes n}$ of $E$. If one regard the elements $e_{i}$ as letters, then $E^{\otimes n}$ is spanned by words of length $n$ in the letters $\left\{ e_{1},\ldots,e_{d}\right\} $, and can be identified with the space of real homogeneous non-commuting polynomials of degree $n$ in $d$ variables, i.e., $(e_{I}:= e_{i_{1}} \otimes \cdots \otimes e_{i_n})_{I = (i_1, \cdots, i_n) \in \{1, \cdots, d\}^{n}}$. We give the formal definition of the tensor algebra series as follows.
\begin{definition}[Tensor algebra space of $E$]
The space of all formal $E$-tensors series, denoted by $T\left(\left( E\right) \right)$ is defined to be the following space of infinite series:
\begin{eqnarray*}
T((E)) = \left\{\mathbf{a} = (a_0, a_1, \cdots) \Big \vert a_n \in E^{\otimes n}, \forall n \geq 0 \right\}.
\end{eqnarray*}
It is equipped with two operations, an addition and a product defined as follows: $\forall \mathbf{a} = (a_0, a_1, \cdots )$, $\mathbf{b} = (b_0, b_1, \cdots) \in T((E))$, it holds that
\begin{eqnarray*}
\mathbf{a}+\mathbf{b} &=& (a_0+b_0, a_1+b_1, \cdots );\\
\mathbf{a} \otimes \mathbf{b} &=&  (c_0, c_1, \cdots),
\end{eqnarray*}
where $c_n = \sum_{j=0}^{n} a_j \otimes b_{n-j}$.
\end{definition}
Let $\pi^I: T((E)) \rightarrow \mathbb{R}$ denote a projection map such that
$$ a=\sum_{n=0}^{\infty}\sum_{|J|=n} a_J e_{j_1} \otimes \cdots \otimes e_{j_n} \mapsto a_I.$$

\begin{definition}[The Signature of a path]
Let $X \in \Omega_{p}([0, T], \mathbb{R}^{d})$ for some $p \geq 1$. Then the signature of the path is defined as $(1, \mathbf{X}_{s, t}^{1}, \cdots, \mathbf{X}_{s, t}^{n}, \cdots )$, where
\begin{eqnarray*}
\mathbf{X}_{s, t}^{n} := \int_{0 \leq t_1 \leq t_2 \leq \cdots \leq t} dX_{t_1} \otimes dX_{t_2} \otimes \cdots \otimes dX_{t_n} \in E^{\otimes n}.
\end{eqnarray*}
\end{definition}

\begin{remark}
The signature of the path $X$ can be also written as
\begin{eqnarray}
    S(X)=\sum_{n=0}^{\infty} S^{(i_{1}, i_{2}, \ldots, i_{k})}(X)  (e_{i_1} \otimes e_{i_2} \otimes \cdots \otimes e_{i_k}),
\end{eqnarray}
where $S^{(i_{1}, i_{2}, \ldots, i_{k})}(X)$ denotes the iterated integral of $x$ indexed by $I$, and is defined as
\begin{eqnarray*}
S^{(i_{1}, i_{2}, \ldots, i_{k})}(X) = \int_{0<t_{1} <t_{2}<\cdots <t_{k}<T} dX_{t_{1}}^{(i_{1})}  dX_{t_{2}}^{(i_{2})} \cdots  dX_{t_{k}}^{(i_{k})}.
\end{eqnarray*}
\end{remark}

\begin{example}[Signature of a Brownian motion]
Since Brownian motion has finite $p$-variation for any $p > 2$, there are multiple ways to define its iterated integrals, such as the Itô and Stratonovich integrals \cite{lyons2007differential}. Throughout this paper, we adopt the Stratonovich signature of Brownian motion, as it provides the natural basis for solving Stratonovich differential equations driven by Brownian paths.
\end{example}

The precise definition of the $p$-geometric rough path can be found in \cite{lyons2007differential}. Roughly speaking, a $p$-geometric rough path is a generalization of the signature of a path. It is defined via the limit of the signature of bounded 1-variation paths under the $p$-variation metric for some $p\geq 1$. Notably, the Stratonovich signature is a geometric $p$-rough path.

\begin{remark}[Connection between partial signature and signature]
Let  $\bar{W}$ denote the time-argumented Brownian motion. Then the signature of $\bar{W}_{[0, t]}$ of degree 2 is given as follows:
\begin{eqnarray*}
(1,  \bar{\mathbf{W}}^1_{[0, t]},  \bar{\mathbf{W}}^2_{[0, t]}),
\end{eqnarray*}
where $\bar{\mathbf{W}}^1_{[0, t]} = (t, W_t)$ and $ \bar{\mathbf{W}}^2_{[0, t]} =\left (\frac{t^2}{2}, \int_{0}^{t} s \mes W_s, \int_{0}^s W_s  \mes s, \frac{W_{t}^2}{2}\right)$. It is easy to see that the partial signature of $W_{[0, t]}$ is a subset of the signature of $\bar{W}_{[0, t]}$ of degree 2.
\end{remark}

\begin{definition}[Shuffle product]\label{def:shuffle}
Let $I:=(i_1,\ldots,i_k)$ and $J:=(j_1,\ldots,j_m)$ denote two multi-indices, where $k,m\in \mathbb{N}$ and $i_1,\ldots,i_k,j_1,\ldots,j_m\in [d]$. Define a multi-index $R$ by
\begin{equation*}
R:=(r_1,\ldots,r_{k+m})=(i_1,\ldots,i_k,j_1,\ldots,j_m).
\end{equation*}
The shuffle product of $I$ and $J$ is a finite set
\begin{equation*}\label{eqn:shuffle}
    {I \shuffle J}=\{(r_{\sigma(1)},\ldots,r_{\sigma(k+m)})|\sigma\in (k,m)\text{-shuffle}\},
\end{equation*}
where a permutation $\sigma$ of the set $[k+m]$ is called a $(k,m)$-\text{shuffle} if $\sigma(1)<\ldots<\sigma(k)$ and $\sigma(k+1)<\ldots<\sigma(k+m)$.
\end{definition}
\begin{lemma}[Shuffle product property]
Let $S(X)$ denote the geometric rough path. Then for any index $I$ and $J$, it holds that
\begin{eqnarray*}
 \pi^{I}(S(X)) \cdot \pi^{J}(S(X)) =  \pi^{I} \shuffle \pi^{J}(S(X)).
\end{eqnarray*}
\end{lemma}
Given $S(\bar{W})$ is a geometric rough path where the stochastic integral is understood in the Stratonovich sense, we have that
\begin{eqnarray}\label{eqn:shuffle_BM}
    \underbrace{\pi^{(1, 2)}S(\bar{W})}_{\int_{0}^t sdW_s} + \underbrace{\pi^{(2, 1)}S(\bar{W})}_{\int_{0}^t W_sds}= \underbrace{\pi^{(1)}S(\bar{W})  \cdot \pi^{(2)}S(\bar{W})}_{= tW_t}.
\end{eqnarray}

Eqn. \eqref{eqn:shuffle_BM} is essential for our proposed method. Note that the Stratonovich signature of Brownian motion satisfies the shuffle product property, whereas the It\^o signature does not.
\subsection{Proof of Lemma \ref{lemma1} }
\label{sec:lemmaproof}
The proof of Lemma \ref{lemma1} is given as follows.
\begin{proof}
Due to the translation invariance of Brownian motion, it is sufficient to prove that $W_t$ and $M_t:=\int_0^t \int_0^{u} \mes W_{s_1} d u$ have the covariance matrix
$\Sigma =
\begin{bmatrix}
 t & \frac{t^2}{2}  \\
 \frac{t^2}{2} & \frac{t^3}{3}
\end{bmatrix}$.

It is trivial to prove that  $W_t$ and $M_t$ have zero mean. As the variance of $W_t$ is $t$, we only need to compute $(1)$ the variance of $M_t$ and $(2)$ the covariance between $W_t$ and $M_t$.

\begin{enumerate}
    \item[(1)] Expanding $\int_0^t \int_0^{u} dW_{s_1} \mes u$ as $\int_0^t W_{u}  \mes u$, it holds that
    \begin{equation*}
\Var\left(\int_0^t W_{u} \mes u \right) = \mathbb{E}\left(\int_0^{t} W_{u}  \mes u \right)^2,
\end{equation*}
as $\int W_u \mes u$ has zero expectation. By direct calculation, we have
\begin{eqnarray*}
&&\mathbb{E}\left(\int_0^{t} \int_0^{t}W_{u_1}W_{u_2}  du_1 du_2\right)
=\int_0^{t} \int_0^{t}\mathbb{E}\left(W_{u_1}W_{u_2} \right) du_1 du_2
\\
&=& \int_0^{t} \int_0^{t} \min(u_1,u_2)   \mes u_1   \mes  u_2 = \frac{t^3}{3}.
\end{eqnarray*}
\item[(2)] Expanding the covariance between $W_t$ and $M_t$, we have
\begin{equation*}
\Cov(W_t, M_t) = \mathbb{E}[W_t M_t] = \mathbb E \left [ \int_0^t  W_u   W_t  \mes u  \right ] = \int_{0}^{t} \min(u, t) \mes u = \frac{t^2}{2} .
\end{equation*}
\end{enumerate}
Hence, the variance and covariance are computed as required.
\end{proof}
\begin{lemma}[Gaussian law of partial signatures of discrete Brownian motion]
\label{lem:discrete_ps_gaussian}
Let $W$ be a standard $d$-dimensional Brownian motion and
$W_{D_f}$ its piecewise-linear interpolation on an arbitrary
deterministic finite partition $D_f$. For $s<t$ in $D_f$, write
\[
D_f\cap[s,t]=\{s=u_0<\cdots<u_n=t\},
\qquad h=t-s,\qquad \Delta_k=u_k-u_{k-1}.
\]
Then the stochastic components of $\mathrm{PS}(W_{D_f})_{s,t}$ satisfy
\[
\begin{pmatrix}
W_t-W_s\\[2pt]
\displaystyle\int_{s \leq t_1\leq t_2\leq t}  dW_{D_f}(t_1)\,dt_2
\end{pmatrix}
\sim \mathcal{N}(0,\Sigma_{D_f}),
\]
where
\[
\Sigma_{D_f}
=
\begin{pmatrix}
h & \dfrac{h^2}{2}\\[4pt]
\dfrac{h^2}{2}
& \dfrac{h^3}{3}-\dfrac{1}{12}\displaystyle\sum_{k=1}^{n}\Delta_k^3
\end{pmatrix}
\otimes I_d.
\]
Thus these components can be sampled exactly on any such partition.
As $\max_k\Delta_k\to0$ with $s,t$ fixed, $\Sigma_{D_f}$ converges
to the corresponding continuous-path covariance.
\end{lemma}
\begin{proof}
By translation invariance of $W$ we may take $s=0$ and $W_s=0$. Let $(A_k,J_k)$ denote the two components over $[u_0,u_k]$, with $A_0=J_0=0$, and let $\Delta W_k:=W_{u_k}-W_{u_{k-1}}$ and $h_k:=u_k-s$. The path is linear on $[u_{k-1},u_k]$, so Chen's relation~\eqref{eqn:ps_recursion} gives
\[
A_k=A_{k-1}+\Delta W_k,\qquad J_k=J_{k-1}+\Delta_k A_{k-1}+\frac{\Delta_k}{2}\,\Delta W_k,
\]
where $\frac{\Delta_k}{2}\Delta W_k$ is the $(W,t)$ second-level signature coordinate of the linear segment, and $\Delta_k A_{k-1}$ is the corresponding Chen cross term arising from concatenation. Since $\Delta W_k\sim\mathcal{N}(0,\Delta_k I_d)$ is independent of $(A_{k-1},J_{k-1})$, $(A_k,J_k)$ is centred Gaussian, and induction on $k$ gives $\operatorname{Var}(A_k)=h_k$, $\operatorname{Cov}(A_k,J_k)=h_k^2/2$ and $\operatorname{Var}(J_k)=h_k^3/3-\frac{1}{12}\sum_{j\le k}\Delta_j^3$ (times $I_d$). Taking $k=n$ gives $\Sigma_{D_f}$; finally, $\sum_k\Delta_k^3\le(\max_k\Delta_k)^2h\to0$.
\end{proof}

\section{Universality of Sig-DEG model}\label{appendix:univerality}

In order to apply Theorem \ref{Thm: SDE_approx} to our setting, we first need to guarantee that our reverse-SDE process \citep{song2021scorebased} satisfies the required hypotheses (H1)--(H3). To this aim, let $X$ be the forward (noising) process and $X_0\sim p_{\text{data}}$: the forward SDE is
\begin{equation}\label{eq:forward}
    dX_t=-\frac12\beta(t)X_t\,dt+\sqrt{\beta(t)}\,dW_t ,
\end{equation}
so that $X_t\mid X_0\sim\mathcal N(\alpha_tX_0,s_t^2I)$ with $\alpha_t=e^{-\frac12\int_0^t\beta(u)\,du}$ and $s_t^2=1-\alpha_t^2$.
\begin{remark}
    In the CIFAR-10 case, for example, the teacher process follows an equation like \eqref{eq:forward}, with $\beta(t)=\beta_{\min}+t\,\beta_d$, $\beta_{\min}=0.1$, $\beta_d=19.9$, and the reverse process stops at $\varepsilon=10^{-3}$.
\end{remark}
The generative (denoising) backward SDE runs from $T$ to $\varepsilon$, and takes the form
\begin{equation}\label{eq:reverseSDE}
  dY_t=\big[-\tfrac12\beta(t)Y_t-\beta(t)\nabla\log p_t(Y_t)\big]dt+\sqrt{\beta(t)}\,d\tilde W_t ,
\end{equation}
where $p_t$ is the law of $X_t$ and $\tilde{W}_t$ is a backward-in-time Brownian motion. By Tweedie's formula, with $m_t(y)=\E[U_0\mid U_t=y]$, the score function can be rewritten as
\[
  \nabla\log p_t(y)=\frac{\alpha_t m_t(y)-y}{s_t^2},
\]
so that the drift takes the form
\[
  \mu(t,y) = \beta(t)\Big(\frac1{s_t^2}-\frac12\Big)y
           \ -\ \frac{\beta(t)\alpha_t}{s_t^2}\,m_t(y)
           \eqqcolon A_t y + g(t,y) ,
\]
which coincides with the one prescribed by Theorem \ref{Thm: SDE_approx}. Notice also that Tweedie's formula, that provides the split $\mu=A_ty+g$, holds for every forward process with Gaussian transition $X_t=\alpha_t X_0+s_t Z$, hence applying to VP, VE, sub-VP and EDM models, each with its own $A_t$.

In order to ensure the applicability of Theorem \ref{Thm: SDE_approx}, let the data satisfy $|X_0|\leq D$ a.s., let $\beta\in C^\infty([0,T])$ be
positive and fix $\varepsilon>0$. Then we have $g(t,y)=c(t)\,m_t(y)$ with $c(t)=-\beta(t)\alpha_t/s_t^2$ and
$m_t(y)=\E[X_0\mid X_t=y]$. On $[\varepsilon,T]$ the scalar functions $A$, $c$ and $\sigma=\sqrt\beta$ are $C^\infty$ with
bounded derivatives, thus verifying (H1)--(H3) reduces to a statement about $m$.

\begin{lemma}\label{prop:vp}
On $[\varepsilon,T]\times\R^d$, $\epsilon > 0$, the posterior mean $m$ is $C^\infty$ and, for all $a,b\ge0$,
\[
  \big|\partial_t^a\nabla_y^b\,m_t(y)\big|\le C_{a,b}\,(1+|y|^{a}),
\]
with $C_{a,b}$ depending on $D,\varepsilon,T,\beta$ only. Hence, all $y$-derivatives of $g$ are bounded, $\partial_tg$ and
$\partial_t\nabla g$ grow at most linearly, $\partial_t^2g$ at most quadratically, and (H1)--(H3) hold.
\end{lemma}

\begin{proof}
Let $P$ be the law of $X_0$, with support on the closed ball $\bar B_D$, and write the posterior distribution in terms of a Gibbs measure,
\[
  \pi_{t,y}(du)=\frac{e^{\ell(t,y,u)}\,P(du)}{\int e^{\ell(t,y,u')}\,P(du')},\quad\text{with}\quad
  \ell(t,y,u)=-\frac{|y-\alpha_tu|^2}{2s_t^2} ,
  \]
so that $m_t(y)=\E_{\pi_{t,y}}[U]$.

We start by proving that $m\in C^\infty$. Consider any derivative $\partial_t^a\nabla_y^be^{\ell}$: it is a product between $e^\ell$ and a polynomial in $(y,u)$, with
coefficients smooth on $[\varepsilon,T]$. Since these coefficients are bounded on $[\varepsilon,T]\times K\times\bar B_D$ (for some $K$ compact) and $P$
is finite, differentiation under the integral is allowed, and it follows that $m$ is $C^\infty$.

Let $\partial$ be either $\partial_t$ or $\partial_{y_k}$. Differentiating under the integral
\[
  \E_{\pi_{t,y}}[\psi(U)]
  = \frac{N_\psi(t,y)}{Z(t,y)} ,
  \quad\text{width}\quad
  N_\psi(t,y)=\int\psi(u)\,e^{\ell(t,y,u)}\,P(du)
\]
and $\partial e^\ell=(\partial\ell)e^\ell$, gives
\begin{equation}\label{eq:diff-rule}
  \partial\,\E_\pi[\psi(U)]
  =\E_\pi[\psi\,\partial\ell]-\E_\pi[\psi]\,\E_\pi[\partial\ell]
  =\operatorname{Cov}_\pi\big(\psi(U),\,\partial\ell(t,y,U)\big).
\end{equation}
Expanding $\ell=-\frac{1}{2s_t^2}\big(|y|^2-2\alpha_t\,y\cdot u+\alpha_t^2|u|^2\big)$,
\begin{align*}
  \nabla_y\ell&=\frac{\alpha_tu-y}{s_t^2},\\
  \partial_t\ell&=\frac{s_t'}{s_t^3}\big(|y|^2-2\alpha_t\,y\cdot u+\alpha_t^2|u|^2\big)
                 -\frac{1}{2s_t^2}\big(-2\alpha_t'\,y\cdot u+2\alpha_t\alpha_t'|u|^2\big)\\
               &=\kappa(t)\,|y|^2+\lambda(t)\,y\cdot u+\theta(t)\,|u|^2 ,
\end{align*}
with $\kappa=s_t'/s_t^3$, $\lambda=\alpha_t'/s_t^2-2\alpha_ts_t'/s_t^3$, $\theta=\alpha_t^2s_t'/s_t^3-\alpha_t\alpha_t'/s_t^2$ that are all smooth and bounded on $[\varepsilon,T]$. Considering that
$\operatorname{Cov}_\pi(\psi,\chi+c) = \operatorname{Cov}_\pi(\psi,\chi)$ for $c=c(t,y)$, Eqn. \eqref{eq:diff-rule} becomes
\begin{equation}\label{eq:cov-rules}
\begin{aligned}
  \nabla_y\E_\pi[\psi(U)]&=\frac{\alpha_t}{s_t^2}\operatorname{Cov}_\pi\big(\psi(U),U\big),\\
  \partial_t\E_\pi[\psi(U)]&=\lambda(t)\operatorname{Cov}_\pi\big(\psi(U),U\big)\,y+\theta(t)\operatorname{Cov}_\pi\big(\psi(U),|U|^2\big).
\end{aligned}
\end{equation}
Now, let $\mathcal M_k$ be the set of the functions of $(t,y)$ that are finite sums of products $\varphi(t)\, Q(y)\prod_j \E_\pi [p_j(U)]$ with $\varphi$ smooth on $[\varepsilon,T]$, $Q$ a polynomial of degree $\le k$ and
$P_j$ polynomials. By \eqref{eq:cov-rules} and the Leibniz rule, $\nabla_y\mathcal M_k\subseteq\mathcal M_k$ and
$\partial_t\mathcal M_k\subseteq\mathcal M_{k+1}$. Since $m\in\mathcal M_0$, because we can see its $l$-th component $m^{(l)}_t(y)=\E_{\pi_{t,y}}[U_l]$, with $\varphi\equiv1$ and $Q\equiv1$ and one posterior expectation of $P(u)=u_l$, it follows that
\[
  \partial_t^a\nabla_y^bm\in\mathcal M_a .
\]
The last condition (H3) eventually follows from the fact that, in a product like $\varphi(t)Q(y)\prod_j\E_\pi[p_j(U)]$, $\varphi$ is bounded on the compact set $[\varepsilon, T]$, $\phi$ is absolutely continuous with respect to the prior $P$, hence it is compactly supported and $\E_\pi[p_j(U)]$ is finite, and $Q$ has polynomial growth in $y$.
\end{proof}

We now proceed to prove that our method provides a universal approximator.
Since the sampler integrates the generative SDE backwards in physical time, from the final state $Y_T$ of the forward process down to time $\varepsilon$, we keep the physical labels: a partition is $T=t_0>t_1>\dots>t_{N_c}=\varepsilon$, with steps $\delta_i=t_i-t_{i+1}$ and mesh $\delta=\max_i\delta_i$.
In the reversed clock $s=T-t$ the equation is the additive-noise SDE \eqref{eq:reverseSDE} on $[0,T-\varepsilon]$, started at $Y_T$. Hypotheses (H1)--(H3) refer to the reversed clock, where they are the ones verified in Proposition~\ref{prop:vp}.

Let us start by defining a \emph{modulus of continuity}.

\begin{definition}[Modulus of continuity]
Given a function $w: [0, T] \rightarrow \mathbb{R}^d$, belonging to a set $K$ compact in the uniform topology, we define
$$
 \omega_w(r)=\sup_{\abs{u-v}\le r}\norm{w(u)-w(v)}, \qquad \omega_K (r) = \sup_{w\in K} \omega_w(r), \quad \forall r > 0.
$$
\end{definition}\begin{theorem}[Compact-path approximation]
\label{thm:pathwise}
Assume (H1)--(H3). Let $B\subset\R^d$ and $K\subset C_0([0,T-\varepsilon];\R^d)$ be compact, $K$ in the uniform
topology. There are $h_0>0$ and a constant $C_{B,K}$, both independent of the partition, such that for every partition
with mesh $h\le h_0$
\begin{equation}
 \sup_{y\in B,\,w\in K}\ \max_{0\le i\le N_c}
 \norm{Y_{t_i}^{y,w}-\hat Y_{t_i}^{y,w}}
 \le C_{B,K}\bigl(h+\omega_K(h)^2\bigr).
 \label{eq:pathwise-rate}
\end{equation}
\end{theorem}

\begin{proof}
Along the proof, $C$ denotes a (possibly changing) constant depending only on $B$, $K$, $T$, $\varepsilon$ and the constants in (H1)--(H3).

Since $K$ is compact, it is uniformly bounded ($\bar w:=\sup_{w\in K}\norm w_\infty<\infty$)
and equicontinuous ($\omega_K(r)\to0$ as $r\to0$); moreover, $\omega_K$ is nondecreasing and $\omega_K\le2\bar w$. Since $|\mu(t,x)|\le C(1+|x|)$, and the pathwise integral is bounded by $C\bar w$, Gr\"onwall applied to
$|Y_u|\le|y|+C\bar w+C\int_0^u(1+|Y_r|)\,dr$ gives
\[
  \sup_{y\in B,\ w\in K,\ u}\ |Y^{y,w}_u|\le R_0<\infty .
\]
Let also $\mathfrak B$ be the closed ball of radius $R_0+1$. On $[\varepsilon,T]\times\mathfrak B$ the functions $\mu$, $\partial_t\mu$, $\nabla\mu$, $\nabla^2\mu$, $\gen\mu$ and $\nabla\gen\mu$ are continuous by (H2), hence bounded.

Now, given step of length $\ell=\delta_i$, consider the solution starting at physical time $t=t_i$ from the state $x=Y^{y,w}_{t_i}$, $|x|\leq R_0$. We then have, for $u\in[t_i,t_{i+1}]$,
\[
Y_u - y = \sigma_t (w_u-w_s) + \O(h) .
\]
Hence, a Taylor expansion of the drift gives
\begin{align*}
  \mu(u,Y_u)&=\mu(t,x)+\nabla\mu(t,x)(Y_u-x)+\O(\ell)+O(|Y_u-x|^2)\\
            &=\mu(t,x)+\nabla\mu(t,x)\sigma(t)(w_u-w_s)+\O\big(h+\omega_K(h)^2\big).
\end{align*}
Integrating over the time step,
\[
  \int \mu(u,Y_u)\,du=\mu(t,x)\ell+\nabla\mu(t,x)\sigma(t)\int_{t_i}^{t_{i+1}}w_{t_i,u}du +\O\big(\ell^2+\ell\,\omega_K(\ell)^2\big).
\]
As for the drift, an expansion
$\sigma(u)=\sigma(t)+\sigma'(t)(u-s)+O((u-s)^2)$ and integration-by-parts yield
\[
  \int\sigma\,dw=\sigma(t)\Delta w_i+\sigma'(t)\Big(\ell w_{t_i,t_{i+1}}-\int_{t_i}^{t_{i+1}}w_{t_i,u}du\Big)+\O\big(\ell^2\omega_K(\ell)\big),
\]
and $\ell^2\omega_K(\ell)\le2\bar w\ell^2$. Since the deterministic term $\frac{\ell^2}2\gen\mu(t,x)$ of $F$ is
$O(\ell^2)$ on $\mathfrak B$, we can bound the remainder
$R_i\coloneqq Y_{t_{i+1}}-Y_{t_i}-F(Y_{t_i},t_i,PS_i(w))$ as
\[
  |d_i|\le C\bigl(\delta_i^2+\delta_i\,\omega_K(\delta_i)^2\bigr).
\]

Let us now verify the stabilit yof the solution map$\Phi_i(x)\coloneqq x+F(x,t_i,PS_i(w))$.  For $x,x'\in\mathfrak B$, it satisfies
\[
  |\Phi_i(x)-\Phi_i(x')|\le\Big(1+C\delta_i+C\Big|\ell w_{t_i,t_{i+1}}-\int_{t_i}^{t_{i+1}}w_{t_i,u}du\Big|+C\delta_i^2\Big)|x-x'|\le(1+L\delta_i)|x-x'| ,
\]
using the bounds on $\nabla\mu$, $\nabla^2\mu$, $\nabla\gen\mu$ on $\mathfrak B$ and $\big|\ell w_{t_i,t_{i+1}}-\int_{t_i}^{t_{i+1}}w_{t_i,u}du\big|\le \delta_i\omega_K(\delta_i)\le2\bar w\delta_i$.

The proof is closed by induction. Indeed, let $\err_i=Y_{t_i}-\hat Y_{t_i}$, so $\err_0=0$,
$Y_{t_{i+1}}=\Phi_i(Y_{t_i})+d_i$ and $\hat Y_{t_{i+1}}=\Phi_i(\hat Y_{t_i})$. As long as
$\hat Y_{t_j}\in\mathfrak B$ for $j\leq i$, we have seen that $|\err_{i+1}|\leq(1+L\delta_i)|\err_i|+|d_i|$, and discrete Gr\"onwall
yields
\[
  |\err_{i+1}|
  \leq e^{LT}\,C\sum_{j\leq i}\big(\delta_j^2+\delta_j\,\omega_K(\delta_j)^2\big)
  \leq e^{LT}\,C\,T\big(\delta+\omega_K(\delta)^2\big)\eqqcolon C_{B,K}\big(\delta+\omega_K(\delta)^2\big),
\]
by $\delta_j\leq \delta$, $\sum_j\delta_j\leq T$ and the monotonicity of $\omega_K$. Choose $\delta_0$ with
$C_{B,K}(\delta_0+\omega_K(\delta_0)^2)\leq1$. Then $|\err_{i+1}|\leq1$, so $\hat Y_{t_{i+1}}\in\mathfrak B$, the induction closes
and \eqref{eq:pathwise-rate} holds for all $i$.
\end{proof}

\begin{corollary}[H\"older-path bounds]
\label{cor:holder}
If $\omega_K(r)\le Mr^\alpha$ for some $0<\alpha<\frac12$, the bound \eqref{eq:pathwise-rate} is $O(\delta^{2\alpha})$.
\end{corollary}

\begin{theorem*}[Universality of Sig-DEG]
Assume (H1)--(H3). Let $\xi\sim p_{\mathrm{noise}}$ be a Gaussian random vector, $W$ a $d$-dimensional Brownian
motion on $[0,T-\varepsilon]$ independent of $\xi$ and $Y=Y^{\xi,W}$ be the solution started at $Y_T=\xi$. Consider
\[
  Z^{\theta}_{t_0}=\xi,\qquad
  Z^{\theta}_{t_{i+1}}=Z^{\theta}_{t_i}+N_\theta\bigl(Z^{\theta}_{t_i},t_i,PS_i(W)\bigr),
\]
where $N_\theta$ ranges over a class of feedforward networks that is a universal approximator of continuous
functions on compact subsets of $\R^d\times[\varepsilon,T]\times\R^{1+2d}$. For every $\vartheta\in(0,1)$, there exist
$C_\vartheta<\infty$ and $\delta_\vartheta\in(0,T-\varepsilon]$ independent of the partition such that, for every partition with
mesh $\delta\le \delta_\vartheta$, there is a parameter $\theta=\theta(\vartheta,D)$ with
\[
  \mathbb P\Big(\max_{0\le i\le N_c}\bigl|Y_{t_i}-Z^{\theta}_{t_i}\bigr|\le C_\vartheta\,\delta\log(eT/\delta)\Big)\ge1-\vartheta .
\]
\end{theorem*}

\begin{proof}
Fix $\vartheta\in(0,1)$ and set $q(\delta)=\delta\log(eT/\delta)$; note $\delta\le q(\delta)$ for $\delta\le T$.

First, we need to prove that there exists a compact set of paths, with probability $1-\vartheta$, on which the modulus $\omega_w$ is bounded by $R_\vartheta\sqrt{r\log(eT/r)}$ for some constant $R_\vartheta$. By L\'evy's modulus of continuity there is an a.s.\ finite
random variable $R$ with $\omega_W(r)\le R\sqrt{r\log(eT/r)}$ for all $0<r\le T-\varepsilon$: L\'evy's theorem gives
it for small $r$, and for $r$ bounded away from zero $\omega_W(r)\le2\norm W_\infty$ while $\sqrt{r\log(eT/r)}$ is
bounded away from zero. Choose $R_\vartheta$ with $\mathbb P(R>R_\vartheta)\le\vartheta/2$ and a closed ball $B_\vartheta$ with
$\mathbb P(\xi\notin B_\vartheta)\le\vartheta/2$, and put
\[
  K_\vartheta=\Big\{w\in C_0([0,T-\varepsilon];\R^d):\ \omega_w(r)\le R_\vartheta\sqrt{r\log(eT/r)}\ \text{for all }0<r\le T-\varepsilon\Big\}.
\]
Then $\mathbb P(\xi\in B_\vartheta,\ W\in K_\vartheta)\ge1-\vartheta$. The set $K_\vartheta$ is closed under uniform convergence,
uniformly bounded ($|w_s|=|w_s-w_0|\le\omega_w(s)$) and equicontinuous, hence compact by Arzel\`a--Ascoli, and
$\omega_{K_\vartheta}(\delta)^2\le R_\vartheta^2\,q(\delta)$.

With the compact set $K_\vartheta$ at hand, Theorem~\ref{thm:pathwise} with $B=B_\vartheta$ and $K=K_\vartheta$ provides a $\delta_0(\vartheta)$ such that, for $\delta\le \delta_0(\vartheta)$,
\begin{equation}\label{eq:step2}
  \sup_{y\in B_\vartheta,\,w\in K_\vartheta}\ \max_i\ \bigl|Y^{y,w}_{t_i}-\hat Y^{y,w}_{t_i}\bigr|
  \le C^{(0)}_\vartheta\,q(\delta),\qquad C^{(0)}_\vartheta:=(1+R_\vartheta^2)\,C_{B_\vartheta,K_\vartheta}.
\end{equation}
Let now $\mathfrak B_\vartheta$ be the ball defined in the proof of such Theorem, containing all exact solutions and all reference iterates for $\delta\le \delta_0(\vartheta)$, uniformly in the partition, and let $\mathfrak B'_\vartheta$ be the concentric ball with radius $r(\mathfrak B_\vartheta) + 1$. We have seen that, on $\mathfrak B'_\vartheta$,
\[
  \bigl|\Phi^w_i(x)-\Phi^w_i(x')\bigr|\le(1+L_\vartheta \delta_i)\,|x-x'|,\qquad \Phi^w_i(x)=x+F(x,t_i,PS_i(w)),
\]
uniformly over $w\in K_\vartheta$ and over the partition. Fix now a partition $D$ with mesh $\delta\le \delta_0(\vartheta)$: the inputs of network $N_\theta$, i.e. $(x,t_i,PS_i(w))$, with $x\in\mathfrak B'_\vartheta$ and $w\in K_\vartheta$ range over a compact subset of
$\R^d\times[\varepsilon,T]\times\R^{1+2d}$, on which $F$ is continuous by (H2). Universal approximation therefore provides
$N_\theta$ with
\[
  \big|N_\theta(x,t_i,PS_i(w))-F(x,t_i,PS_i(w))\big|\leq\eta_D:=\frac{q(\delta)\,e^{-L_\vartheta T}}{N_c}
\]
for all such inputs and all $0\le i<N_c$; $\theta$ depends on $\vartheta$ and on $D$ through $\eta_D$ and the input set.

In order to conclude the proof, we are only left with showing stability of the neural recurrence. For $(y,w)\in B_\vartheta\times K_\vartheta$, run both recurrences from
$y$ with the same features and let $\err_i=|Z^{\theta,y,w}_{t_i}-\hat Y^{y,w}_{t_i}|$. As long as the neural iterates
stay in $\mathfrak B'_\vartheta$,
\begin{equation}\label{eq:step4}
  \err_0=0,
  \qquad
  \err_{i+1} \leq (1+L_\vartheta \delta_i)\,e_i+\eta_D,
\end{equation}
thus, by Gr\"onwall, $\max_i\err_i\leq N_c\,\eta_D\,e^{L_\vartheta T}=q(\delta)$. Take $\delta_\vartheta\le \delta_0(\vartheta)$ so small
that $q(\delta)\le\frac12$ for $\delta\le \delta_\vartheta$: then $e_i\le\frac12<1$ keeps every neural iterate in $\mathfrak B'_\vartheta$, and the bound holds throughout the recurrence.

It is now sufficient to combine \eqref{eq:step2} and \eqref{eq:step4} to obtain the thesis:
\[
  \sup_{y\in B_\vartheta,\,w\in K_\vartheta}\ \max_i\ \bigl|Y^{y,w}_{t_i}-Z^{\theta,y,w}_{t_i}\bigr|
  \le\bigl(C^{(0)}_\vartheta+1\bigr)\,q(\delta).
\]
Indeed, on the event $\{\xi\in B_\vartheta,\ W\in K_\vartheta\}$ the processes $Y$, $\hat Y$ and $Z^\theta$ are these pathwise
objects evaluated at $(\xi,W)$ , which proves the claim with
$C_\vartheta=C^{(0)}_\vartheta+1$.
\end{proof}

\section{Summary of Accelerated Samplers}\label{app:related}
The summary of the literature on accelerated samplers is provided in Table \ref{tab:related}.
\begin{table*}[t]
  \centering
  \footnotesize
  \renewcommand{\arraystretch}{1.05}
  \setlength{\tabcolsep}{3.5pt}
  \resizebox{\linewidth}{!}{%
\begin{tabular}{@{}>{\raggedright\arraybackslash}p{0.17\linewidth}>{\raggedright\arraybackslash}p{0.25\linewidth}>{\raggedright\arraybackslash}p{0.3\linewidth}>{\raggedright\arraybackslash}p{0.24\linewidth}@{}}
    \toprule
    Method & Approximates & Supervision & Stochastic input \\
    \midrule
    \multicolumn{4}{l}{\textit{PF-ODE target}} \\
    EDM Heun \citep{edm} & PF-ODE (second-order Heun) & none (training-free) & initial noise only \\
    DPM-Solver++ \citep{lu2025dpm} & PF-ODE & none (training-free) & initial noise only \\
    PD \citep{salimans2022progressive} & PF-ODE (DDIM sampler) & one student step matches two teacher steps & initial noise only \\
    Consistency models \citep{song2023consistency} & PF-ODE solution map & self-consistency on adjacent teacher ODE steps & fresh re-noising between steps \\
    CTM \citep{kim2023consistency} & PF-ODE map, any $t \to s$ & soft consistency, DSM, adversarial & optional fresh re-noising \\
    PCM \citep{wang2024phased} & PF-ODE per phase & phased consistency + adversarial & initial noise; optional fresh noise \\
    \multicolumn{4}{l}{\textit{Distribution matching}} \\
    DMD \citep{yin2024one} & teacher's output distribution & reverse-KL matching + output regression & one Gaussian latent \\
    \midrule
    \multicolumn{4}{l}{\textit{Reverse-SDE target}} \\
    EDM SDE \citep{edm} & reverse SDE (noise injection, then a Heun step) & none (training-free) & fresh Gaussian per step \\
    SEEDS \citep{gonzalez2023seeds} & reverse SDE (exponential integrator) & none (training-free) & fresh Gaussians per step \\
    SA-Solver \citep{xue2023sasolver} & reverse SDE (variance-controlled) & none (training-free) & fresh Gaussian per step \\
    GMS \citep{guo2023gaussian} & reverse-SDE kernel as Gaussian mixture & extra heads regress noise moments & fresh Gaussian + mixture draw per step \\
    SCott \citep{liu2024scott} & consistency map along SDE-solver trajectories & consistency distillation (SDE teacher) + adversarial & fresh re-noising between steps \\
    It\^{o} maps$^{\dagger}$ \citep{pan2026maps} & SDE flow map given the Brownian path & self-distillation & Brownian path as KL or Haar features \\
    \midrule
    Sig-DEG (ours) & reverse SDE on $D_c$ & regression on teacher trajectories, same Brownian path & partial signatures, sampled exactly \\
    \bottomrule
    \multicolumn{4}{@{}l}{\scriptsize $^{\dagger}$Trained from scratch by default; a pre-trained flow-map model can optionally supply the drift target.} \\
  \end{tabular}%
}
  \caption{Accelerated samplers for pre-trained diffusion models, grouped by their target (PF-ODE: probability-flow ODE). Sig-DEG is the only one that regresses fine-resolution teacher trajectories under the teacher's own Brownian path, which enters as partial signatures sampled exactly from their Gaussian law (Lemmas~\ref{lemma1} and~\ref{lem:discrete_ps_gaussian}).}
  \label{tab:related}
\end{table*}

\section{Implementation and Computational Details}
\label{sec::comput_details}
\subsection{General notes}
\label{general}
\textbf{Code Availability.} The complete source code, including model implementations, training scripts, and evaluation tools, will be released publicly to facilitate reproducibility and further research.

\textbf{Software Environment.} All numerical experiments were conducted using Python 3.13 and PyTorch 2.9.1 with CUDA 12.8 and cuDNN 9.10 for GPU-accelerated computations. Wasserstein distances between empirical and model-generated distributions in the 1D Gaussian mixture dataset were computed using the \texttt{POT} (Python Optimal Transport) library \citep{flamary2021pot}. For image datasets, generative quality was evaluated using the Fréchet Inception Distance (FID): the \texttt{pytorch\_fid} package was used for MNIST, while for CIFAR-10 we chose reference implementation of \citet{edm} for better comparison.

\textbf{Hardware Infrastructure.} Experiments were executed on the following computational platforms:

\begin{itemize}
    \item \textbf{5-GPU Workstation.} This system is composed of 5 NVIDIA Quadro RTX 8000 GPUs (each with 49\,GB VRAM), although only a single GPU was used per experiment. It runs Ubuntu 22.04.5 LTS and is powered by dual Intel Xeon Gold 6238 CPUs (88 threads in total, 2.10\,GHz base frequency), 502\,GB RAM, and SSD storage.
    \item \textbf{H100 Workstation.} A workstation with one NVIDIA H100 NVL GPU (94\,GB), an AMD EPYC 9554P CPU (64 cores, 128 threads) and 1.1\,TB RAM, running Ubuntu 24.04.5 LTS. All reported sampling and training times were measured on it (see \emph{Timing} below).
    \item \textbf{RTX A6000 Workstations.} Single-GPU workstations with an NVIDIA RTX A6000 (48\,GB), used for part of the rBerg and CIFAR-10 training.
    \item \textbf{GPU Cluster.} Part of the image-dataset experiments were performed on a GPU cluster whose compute nodes each host 4 NVIDIA GH200 Grace Hopper Superchips.
    Each superchip pairs a 72-core Arm Neoverse~V2 Grace CPU with a Hopper-based GH100 GPU with 96\,GB of HBM3 memory, connected through NVLink-C2C.
\end{itemize}

\textbf{Timing.}\label{app:timing} All inference and training times are measured on the same NVIDIA H100 NVL GPU, so every speed and cost comparison is made on identical hardware. Inference times are measured with PyTorch's float32 matrix-multiplication precision set to \texttt{high}, as in all training and evaluation runs; they are CUDA-synchronised and are the mean $\pm$ standard deviation over 10 timed passes after a discarded warm-up (the MNIST teacher's time is extrapolated from 5 timed batches); the batch is the 1,000-sample test set on the 1D Gaussian Mixture and 20k samples on rBerg; on MNIST and CIFAR-10, 79 batches of 128 images (10,112 images). $\times$Speed is the teacher's time divided by the model's. Training times are the measured time per epoch times the number of epochs, plus the time spent on validation (on CIFAR-10, without validation).

\textbf{Uncertainty.} Unless stated otherwise, $\pm$ is the standard deviation over repeated sampling evaluations of one trained model, so it measures sampling variability only. On the 1D Gaussian mixture, the mean is over five independently trained models and the SD over test batches (Table~\ref{tab:1dtoy}); for the Schr\"odinger-bridge teachers, mean and SD are over three independently trained students (Table~\ref{tab:sb}).

\textbf{Consistency-model baselines.}\label{app:cm} On the 1D Gaussian Mixture, MNIST and rBerg we compare with consistency distillation (CD) \citep{song2023consistency} from the same teacher that Sig-DEG distils. At each training step a time index $n$ is drawn uniformly from the teacher's grid $t_1 < \dots < t_N$, $x_{t_n}$ is drawn from the forward perturbation kernel given a training sample, and one step of the teacher's probability-flow ODE maps it to $\hat{x}_{t_{n-1}}$. The consistency function $f_\theta(x, t) = c_{\text{skip}}(t)\,x + c_{\text{out}}(t)\,F_\theta(x, t)$, with the boundary scalings of \citet{song2023consistency} ($\sigma_{\text{data}} = 0.5$), is regressed at $(x_{t_n}, t_n)$ onto an exponential-moving-average copy of itself (decay $0.99$) at $(\hat{x}_{t_{n-1}}, t_{n-1})$ with a log-MSE loss. The network $F_\theta$ has the architecture of the Sig-DEG generator without the partial-signature input, and, like the Sig-DEG generator, it is initialised at random: the Sig-DEG generator cannot be initialised from the teacher because of its additional input, and on rBerg the students are narrower than the teacher, so both students start from the same point. A $K$-step sample applies $f_\theta$ at $K$ times of a uniform grid from $T$ towards $0$, re-noising the current estimate to the next time with fresh Gaussian noise between evaluations \citep{song2023consistency}, so one trained model is scored at every $K$. The checkpoint is selected on the same validation metric as for Sig-DEG. Dataset-specific settings are given in the respective subsections.

\textbf{Optimization}. All experiments were optimized using the Adam optimizer~\citep{2015-kingma}, with learning rates tailored to each specific task. The choice of learning rate was determined via grid search on a validation set.

As network architectures, hyperparameters, and test metrics depend on the dataset, we defer the relevant details to the corresponding subsections for each dataset.

\subsection{1D Gaussian mixture}
\label{sec:1D_dataset}
\subsubsection{Data and teacher}
The \textbf{1D Gaussian Mixture} is a mixture of two Gaussians with standard deviation $0.5$ centred at $-2$ and $+2$, a controlled low-dimensional test of whether a model recovers both modes. We draw 10,000 samples: 8,000 for training, 1,000 for validation and 1,000 for test. The teacher is a diffusion model with $300$ time steps. %

\subsubsection{Network and training}
\label{sec:1d_network}
The 1D and rBerg models use a residual MLP, the \emph{Residual Architecture}. A sinusoidal embedding $t_{\text{emb}} \in \mathbb R^{128}$ of the time step is mapped to $t_{\text{feat}} \in \mathbb R^{256}$ by $\text{Linear}(128, 256) \rightarrow \text{SiLU} \rightarrow \text{Linear}(256, 256)$. The input $x \in \mathbb R^{d_x}$ is mapped by $\text{Linear}(d_x, 256)$ and three residual blocks, each of three linear layers followed by SiLU, with a residual connection scaled by $1/\sqrt{2}$,
\begin{equation}\nonumber
 h_{i+1} = \tfrac{1}{\sqrt{2}} \big(h_i + \text{SiLU}(\text{Linear}(\text{SiLU}(\text{Linear}(\text{SiLU}(\text{Linear}(h_i))))))\big),
\end{equation}
to $x_{\text{feat}} \in \mathbb R^{256}$, and the output is $y = \text{MLP}_{\text{out}}(x_{\text{feat}} + t_{\text{feat}})$ with $\text{MLP}_{\text{out}} := \text{Linear}(256, 256) \rightarrow \text{SiLU} \rightarrow \text{Linear}(256, d_{\text{out}})$.

\subsubsection{Test metrics}
We report the Wasserstein-1 distance $W_1$ between the real and generated samples and the variance error, the absolute difference between their empirical variances.

\subsubsection{Consistency-model baseline}\label{app:1d_cm}
The CM columns of Table~\ref{tab:1dtoy} use the recipe of Appendix~\ref{app:cm} and the Residual Architecture above, and distil, for each seed, the same 300-step teacher as Sig-DEG. Training: Adam (learning rate $10^{-4}$, weight decay $10^{-4}$), batch size 256, 2000 epochs, checkpoint selected on validation W$_1$ every 5 epochs. The $K$ evaluation times are the $K$-step coarse grid $\{T, T - T/K, \dots, T/K\}$ that Sig-DEG uses. As for Sig-DEG, one CM is trained per seed and the reported numbers are averaged over the 5 seeds.

\subsubsection{Results}
\takeaway{Table~\ref{tab:1dtoy} compares Sig-DEG with CM at the same number of steps. Sig-DEG has the lower $W_1$ at both 10 and 5 steps and the lower variance error at 10 steps ($0.171$ vs.\ $0.239$); at 5 steps the two variance errors are close ($0.219$ vs.\ $0.204$). It is also faster ($5.1$ vs.\ $6.7$\,ms and $2.7$ vs.\ $3.3$\,ms) and reaches the 300-step teacher's $W_1$ ($0.191$ and $0.208$ vs.\ $0.217$) at a $32$--$61\times$ speed-up.}
\begin{table}[ht]
    \centering
    \small
    \renewcommand{\arraystretch}{1}
   \resizebox{0.5\textwidth}{!}{ \begin{tabular}{rccccc}
      \toprule
      & Teacher & \multicolumn{2}{c}{10 steps} & \multicolumn{2}{c}{5 steps} \\
      \cmidrule(lr){3-4}\cmidrule(lr){5-6}
      Metric & (300) & SD & CM & SD & CM \\
      \midrule
      W$_1$ $\downarrow$ & \pmc{0.217}{0.090} & \pmcb{0.191}{0.079} & \pmc{0.217}{0.087} & \pmcb{0.208}{0.087} & \pmc{0.218}{0.089} \\
      Variance $\downarrow$ & \pmc{0.166}{0.103} & \pmcb{0.171}{0.107} & \pmc{0.239}{0.116} & \pmc{0.219}{0.118} & \pmcb{0.204}{0.118} \\
      Time (ms) $\downarrow$ & \pmc{163.76}{0.52} & \pmcb{5.062}{0.019} & \pmc{6.680}{0.016} & \pmcb{2.665}{0.034} & \pmc{3.347}{0.018} \\
      $\times$ Speed $\uparrow$ & -- & \numb{32} & $25$ & \numb{61} & $49$ \\
      \bottomrule
    \end{tabular}}
    \caption{\textbf{1D Gaussian Mixture} (300-step teacher): Sig-DEG (SD) vs.\ CM. Mean over 5 training seeds $\pm$ standard deviation over the test batches; time for 1,000 samples. Bold: the better of SD and CM.}
    \label{tab:1dtoy}
\end{table}

\takeaway{Figure~\ref{fig:1d_backward} shows the fit of the backward process itself: driven by the teacher's Brownian path, Sig-DEG (5) follows each teacher trajectory at its five coarse times (a), and its marginals agree with the teacher's at every coarse time, from Gaussian noise to the two modes (b).}
\begin{figure}[ht]
    \centering
    \begin{subfigure}[t]{0.3\linewidth}
        \centering
        \includegraphics[width=\linewidth]{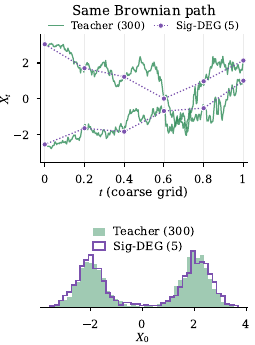}
        \caption{Paired paths}
        \label{fig:1d_backward_A}
    \end{subfigure}\hspace{0.06\linewidth}
    \begin{subfigure}[t]{0.3\linewidth}
        \centering
        \includegraphics[width=\linewidth]{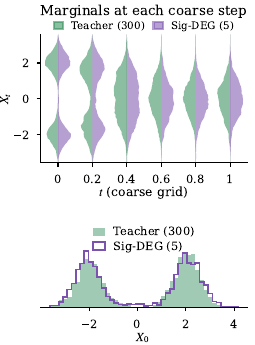}
        \caption{Marginals at each coarse time}
        \label{fig:1d_backward_C}
    \end{subfigure}
    \caption{\textbf{1D Gaussian Mixture}: 300-step teacher (green) and Sig-DEG (5) (purple), generating from $t = 1$ (right) to $t = 0$ (left). (a) Two teacher paths and Sig-DEG driven by the same Brownian paths (dots: coarse times). (b) Marginals at each coarse time, Sig-DEG with independent noise. Bottom: 2,000 final samples, (a) with the same Brownian paths, (b) with independent noise.}
    \label{fig:1d_backward}
\end{figure}

\FloatBarrier
\subsection{MNIST}
\label{sec:MNIST}
\subsubsection{Data and teacher}
MNIST~\citep{lecun1998gradient} contains 60,000 training and 10,000 test images of handwritten digits (ten classes, $28\times28$ grey-scale) and is available under the Creative Commons Attribution-Share Alike 3.0 licence; 1,000 training images are held out for validation. Following latent diffusion \citep{rombach2022high}, the teacher and Sig-DEG operate in the latent space of a frozen VQ-VAE \citep{oord2017neural} (latent shape $3\times7\times7$); the teacher is trained with $1500$ diffusion steps.

\subsubsection{Network and training}
\label{mnist_network}
The teacher's denoiser is the U-Net of \citet{rombach2022high}: a symmetric encoder--decoder with skip connections and several ResNet blocks per level, self-attention layers and time embeddings added to each block; the class label enters through a learned class embedding added to the time embedding. The Sig-DEG generator is the same network with one additional input, the partial signature $PS(W)_{t_i, t_{i-1}}$, a Gaussian vector (Lemma~\ref{lemma1}) concatenated with the noisy latent along the feature dimension; this input stage is the only architectural difference from the teacher. Sig-DEG is trained against teacher trajectories sampled afresh for every mini-batch, with no warm-up stage, and the checkpoint with the best validation FID is reported.

\subsubsection{Test metrics}
We generate 50,000 images, decode them with the VQ-VAE and report the Fréchet Inception Distance (FID) against the training set and the Inception Score (IS), each as the mean $\pm$ standard deviation over three independent evaluations. Since the ImageNet Inception network is not a meaningful classifier for grey-scale digits, the IS is computed with a ResNet-18 classifier trained on MNIST, so its scale (at most $10$) is not comparable with IS values on natural images. The VQ-VAE column of Table~\ref{tab:mnist} is the FID of the 10,000 test images passed through the VQ-VAE (7.85); as FID decreases with the sample size, extrapolating it to 50,000 samples gives 7.33.

\subsubsection{Consistency-model baseline}\label{app:mnist_cm}
The CM columns of Table~\ref{tab:mnist} use the recipe of Appendix~\ref{app:cm} in the latent space of the shared VQ-VAE, with the 1500-step teacher and the teacher's class-conditional U-Net (Section~\ref{mnist_network}) as $F_\theta$. Training: Adam (learning rate $2 \times 10^{-4}$), batch size 512, 1500 epochs; the checkpoint is selected every 10 epochs on the validation FID of one-step sampling, and the selected model is scored at 5 and 10 steps. FID, IS and inference time are measured exactly as for Sig-DEG. A $K$-step CM sample uses $K$ network evaluations at the times of a uniform grid from $T$ to $0$ inclusive.

\subsubsection{Results}
\takeaway{Table~\ref{tab:mnist} compares Sig-DEG with CM at the same number of steps. Sig-DEG reaches FID $8.28$ and $7.96$ at 10 and 5 steps, better than both the 1500-step teacher ($8.94$) and CM ($10.27$ and $10.04$), with a higher IS than CM, while sampling $149$--$297\times$ faster than the teacher.}
\begin{table}[ht]
    \centering
    \small
    \renewcommand{\arraystretch}{1}
    \resizebox{0.75\textwidth}{!}
    {
    \begin{tabular}{rcccccc}
        \toprule
        & VQ-VAE & Teacher & \multicolumn{2}{c}{10 steps} & \multicolumn{2}{c}{5 steps} \\
        \cmidrule(lr){4-5}\cmidrule(lr){6-7}
        Metric & recon. & (1500) & SD & CM & SD & CM \\
        \midrule
        FID $\downarrow$ & $7.85^{\dagger}$ & \pmc{8.94}{0.07} & \pmcb{8.28}{0.06} & \pmc{10.27}{0.01} & \pmcb{7.96}{0.03} & \pmc{10.04}{0.06} \\
        IS $\uparrow$ & -- & \pmc{9.02}{0.02} & \pmcb{9.06}{0.01} & \pmc{8.88}{0.01} & \pmcb{9.18}{0.02} & \pmc{9.05}{0.01} \\
        Time (s) $\downarrow$ & -- & $886.3$ & \pmcb{5.935}{0.003} & \pmc{5.994}{0.002} & \pmcb{2.986}{0.001} & \pmc{2.990}{0.001} \\
        $\times$ Speed $\uparrow$ & -- & -- & \numb{149} & $148$ & \numb{297} & $296$ \\
        \bottomrule
    \end{tabular}
    }
    \caption{\textbf{MNIST} (1500-step teacher): Sig-DEG (SD) vs.\ CM. FID (50,000 samples) and IS over three evaluations; time for 79 batches of 128 images. Bold: the better of SD and CM. $^{\dagger}$FID of the 10,000 test images after VQ-VAE encoding and decoding, without any generative model.}
    \label{tab:mnist}
\end{table}

\subsection{CIFAR-10}
\label{app:cifar}

\subsubsection{Data and teacher}
CIFAR-10 \citep{krizhevsky2009learning} is a widely used image dataset, containing $50,000$ training and $10,000$ test images, formatted as $32\times32\times3$ pixel tensors and labelled into 10 different classes. The goal of our experiments is to generalise the results we obtained with the MNIST dataset to a more complex setting, consistently used as a benchmark for classification and generative models.

The parent network was chosen to be ``DDPM++ cont. (VP)'' model \citep{song2021scorebased}, with the sampling method provided by \cite{edm}.
Its measured FID score is $3.6590\pm0.0201$, reached via first--order Euler--Maruyama integration on a $1985$--steps time mesh.
We are continuously revising our models to further improve their performance.
The cited method \citep{edm} provides different options to sample from the same reverse SDE; however, our goal was to obtain a student close to its teacher in terms of performance, hence the teacher performance per se is less important than the student--teacher FID gap.

It should also be noted that two slightly different sampling schedules were used. For direct training and progressive distillation, the teacher's fine mesh was set at $N_f=2048$, allowing for uniform-length (in the number of fine steps) coarse steps.
In the case of curriculum learning, however, this was not possible due to the adopted nested refinement of the time mesh.
Starting from a reference mesh of $R=32$ steps, each curriculum stage doubles the fine resolution, with $D_f^{(k)}\subset D_f^{(k+1)}$, so that coarse boundaries are preserved.
In \citet{edm}, the $N-1$ interior steps are uniformly spaced in the time reparametrisation, followed by a final step to $\sigma=0$; successive meshes are therefore nested only if $N-1=(R-1)\,2^k$. Placing the $C$ coarse boundaries on reference nodes further requires $R$ to be divisible by $C$. With $C=8$ and $R=32$, this yields exactly our choices of $N_f^{(k)}=31\cdot 2^k+1$, i.e. the mesh sizes are $32,63,125,249,497,993,1985$. As a result, curriculum learning gives an uneven choice of coarse steps, with the last one covering $193$ instead of $256$ fine steps at the finest resolution.

\subsubsection{Network and training}

Several possibilities were explored to reach an optimal training strategy for these models.  In particular, an ideal workflow could be the following:
\begin{enumerate}
    \item[(i)] Choose a target number $N_c$ of coarse steps for the distilled model.
    \item[(ii)] Start with a very coarse teacher's time mesh and train a $2^k N_c$-student on it, $k\in\mathbb{N}$, then gradually refine the teacher's mesh it until it reaches its final size.
    \item[(iii)] Use the $2^k N_c$ student as a new parent to train a $2^{k-1} N_c$ one, then repeat until reaching $N_c$ coarse steps.
\end{enumerate}
Also, a hybrid loss has been exploited. Pixel-wise logMSE proved useful to reach a good level of convergence, but, due to its tendency to focus on the student following exact teacher's trajectories, it showed significant limitations after the student was starting to close in on the teacher. Adding a distributional loss from the beginning also did not help, due to its weights and bandwidth being too generic. Eventually, we decided to keep pixel-wise logMSE until the teacher reached the final time mesh during training, then add a Gaussian-kernel Maximum Mean Discrepancy (MMD) \citep{gretton2012kernel} to it, comparing stundent- with teacher-generated samples. The Gaussian kernel bandwidth was set, after different trials, to be the inverse of the teacher's features pairwise $L^2$ distance (we chose not to measure MMD on pixels, but rather on features obtained via DINOv2 embedding \citep{dinov2}. Moreover, given the loss taking the form $\operatorname{logMSE} + w\cdot \operatorname{MMD}$, we calibrated the weight $w = r/\rho$ by choosing $r=1.5$ after a sweep over possible values, with $\rho = \| \nabla \operatorname{logMSE} \|_2 / \| \nabla \operatorname{MMD} \|_2$ computed over the rung-entrance checkpoint.
Overall, the addition of the distributional term lowered the FID by $1.63$, as measured by an ablation that restarted from the same rung-entrance checkpoint and repeated the final rung with the same number of epochs and
the same cosine schedule, with the MMD term removed.
It is also to notice that, in all CIFAR-10 runs, including progressive distillation, the logMSE of each coarse step is weighted by $\lvert \Delta\sigma \rvert^{-1/2}$, the gap in noise level that the step covers, with the weights normalised to mean one. Each progressive-distillation stage passes on the checkpoint with the lowest FID.

Table~\ref{tab:cifar_epoch_time}\label{app:cifar_time} reports the training time per epoch on each rung of the refinement curriculum. On a given mesh the cost roughly doubles each time the number of steps halves, because every training pair rolls the teacher out over one coarse interval of $N_f/C$ fine steps; the finest rungs dominate the total.
\begin{table}[ht]
    \centering
    \small
    \setlength{\tabcolsep}{5pt}
    \begin{tabular}{lcccccccc}
        \toprule
        Teacher mesh (fine steps) & 32 & 63 & 125 & 249 & 497 & 993 & 1985 & 2048 \\
        \midrule
        Sig-DEG (64), direct & -- & -- & -- & -- & -- & -- & -- & 508 \\
        Sig-DEG (32), direct & -- & -- & -- & -- & -- & -- & -- & 885 \\
        Sig-DEG (16), curriculum & 106 & 138 & 181 & 277 & 543 & 977 & 1802 & -- \\
        Sig-DEG (8), curriculum & 126 & 185 & 287 & 573 & 1005 & 1858 & 3640 & -- \\
        \bottomrule
    \end{tabular}
    \caption{\textbf{CIFAR-10}: training time per epoch (seconds, batch size 128), for the resolution curriculum on each rung of its teacher mesh and for direct training on the full 2048-step mesh. The cost grows with the fine steps per coarse step, $N_f/C$: $32$ and $64$ for Sig-DEG (64) and (32) on the 2048-step mesh, but about $124$ and $248$ for Sig-DEG (16) and (8) on the 1985-step rung. The curriculum rows are the runs of Sig-DEG (16) and (8) in Table~\ref{tab:cifar10} (Appendix~\ref{app:cifar_curriculum}).}
    \label{tab:cifar_epoch_time}
\end{table}

\subsubsection{Test metrics}
We evaluate (unconditional) image quality by the FID on 50,000 samples.

\subsubsection{Results}
\paragraph{Progressive distillation.}\label{app:cifar_pd}
\takeaway{Table~\ref{tab:cifar_pd} lists two PD chains, started from the directly trained Sig-DEG (64) and (32); each stage halves the number of steps in 12 epochs. Halving 64 steps to 32 raises the FID by $0.21$ ($3.63$ vs.\ $3.42$), to the level of the directly trained Sig-DEG (32) ($3.65$) in $21$\,h less training time (Figure~\ref{fig:cifar_pd_curves}); each further halving raises it more ($4.07$ at 16 steps, $5.18$ at 8).}
\begin{table}[ht]
    \centering
    \small
    \setlength{\tabcolsep}{6pt}
    \resizebox{\linewidth}{!}{%
\begin{tabular}{lcccc}
        \toprule
        & \multicolumn{4}{c}{Sampling steps} \\
        \cmidrule(lr){2-5}
        Chain & 64 & 32 & 16 & 8 \\
        \midrule
        From Sig-DEG (64) & \pmc{3.42}{0.04} (direct, $62.1$\,h) & \pmc{3.63}{0.04} ($25.0$\,h) & \pmc{4.07}{0.05} ($9.9$\,h) & \pmc{5.18}{0.04} ($4.5$\,h) \\
        From Sig-DEG (32) & -- & \pmc{3.65}{0.02} (direct, $108.2$\,h) & \pmc{4.05}{0.04} ($9.9$\,h) & \pmc{5.28}{0.02} ($4.5$\,h) \\
        \bottomrule
    \end{tabular}%
}
    \caption{\textbf{CIFAR-10}: PD chains from the directly trained Sig-DEG (64) and (32). Each stage halves the steps in 12 epochs on the frozen parent's rollouts. FID on 50,000 samples, mean $\pm$ standard deviation over five sampling seeds; in brackets: training time without validation, on one H100 NVL (direct: 440 epochs at the times of Table~\ref{tab:cifar_epoch_time}; PD: that stage).}
    \label{tab:cifar_pd}
\end{table}

\paragraph{Resolution curriculum.}\label{app:cifar_curriculum}

Results presented in this section come from a training strategy aimed at substituting warmup with a progressive refinement of the teacher model's fine mesh, i.e.\ the resolution curriculum of Section~\ref{sec:training_accel}.
The student is hence trained on a refining teacher, with $N_f = 32, 63, 125, 249, 497, 993, 1985$, in order to allow it to start from an almost matching teacher's time resolution and progressively learn how to cover larger steps.
This results in an improved training phase, since lower rungs of this ladder are faster to cover and, at the beginning, the student does not really need the full-resolution time mesh. Moreover, the strategy also proved to achieve better results overall, as showed in Figure \ref{fig:ladder_vs_flat}.

More in detail, the strategy is the following.
First, the model goes through the first $N-1$ rungs with constant learning rate ($8\times10^{-4}$ for a batch size of $128$).
Then, entering the last rung (which is granted significantly more epochs), the learning rate starts decaying with a cosine anneal towards $8\times10^{-6}$, while an additional, distributional term is calibrated and then added to the loss.
The latter is a Gaussian-kernel MMD term, computed on the features resulting from an embedding of the images via DINOv2 \citep{dinov2}.

Results show that the distributional term significantly improves the training, allowing to better close the student-teacher gap even when the pixel-wise, trajectory-focused logMSE loss saturates.

\paragraph{Adaptive mesh.}\label{app:adaptive}

In order to lower the gap between teacher and student, we also devised an a posteriori strategy to better position the coarse step boundaries. Indeed, on a uniform (in time steps, but not in the time reparametrisation) mesh, where e.g. for Sig-DEG (8) one coarse mesh would cover $256$ fine ones, the error does not accumulate uniformly over all steps.
\begin{figure}[t!]
    \centering
    \includegraphics[width=0.5\linewidth]{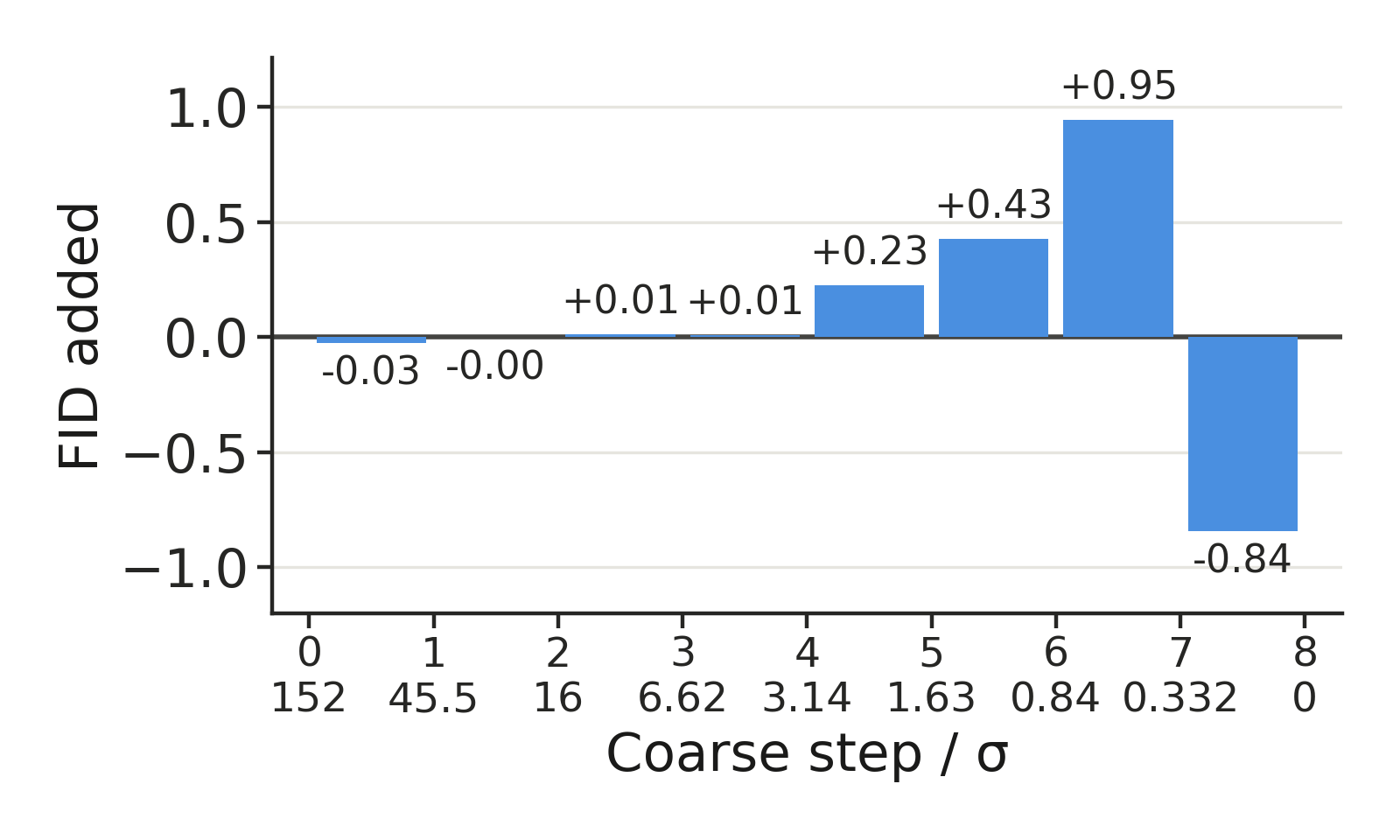}
    \caption{\textbf{Accumulation error for Sig-DEG (8).} The figure shows the FID difference between a teacher completing a student-generated partial trajectory vs. the student doing it. Both start from the same student-generated states. $\sigma$ values are also reported, highlighting how its $\Delta\sigma$ varies enormously over intervals.}
    \label{fig:cifar_accumulation_error}
\end{figure}
Figure \ref{fig:cifar_accumulation_error} illustrates how the accumulation error is distributed on the FID $4.47$ Sig-DEG (8) model of Table~\ref{tab:cifar10}: it is was obtained by having Sig-DEG run through the first $k$ steps, and letting the teacher complete the trajectory on its own fine mesh.
In this way it is possible to observe that the first blocks reproduce the teacher closely, while most of the error is concentrated on the $5$th, $6$th and $7$th steps.
The last step acts as some strong denoiser, over-correcting the trajectory and hence compensating for the accumulated error.
Such error structure led us to select a different coarse mesh, combining the first four steps into two coarser ones, and splitting in half both the $6$th and the $7$th steps.
In this way, we went from a coarse mesh with step boundaries at $[256, 512, 768, 1024, 1280, 1536, 1792, 1985]$ fine steps to one with step boundaries at $[512, 1024, 1280, 1408, 1536, 1664, 1792, 1985]$: heuristically, this could be viewed as exploiting the best parts of Sig-DEG (4), Sig-DEG (8) and Sig-DEG (16).
The results were promising, and the uneven-mesh model reached FID $3.88$, comparable to Sig-DEG (16) ($3.93$), all while retaining $248\times$ fewer network evaluations per trajectory, which make the student $115\times$ faster than its teacher at comparable sample quality.

\paragraph{Training-free solvers.}\label{app:cifar_baselines}
\takeaway{Table~\ref{tab:cifar_baselines} compares Sig-DEG with training-free solvers on the teacher's network at comparable sampling cost. With 8 and 16 steps Sig-DEG has the lowest FID, also against EDM's Heun sampling algorithm at a similar number of network evaluations; from 32 steps on, the mesh becomes fine enough for the ODE solvers to start scoring better.}
\begin{table}[ht]
    \centering
    \small
    \setlength{\tabcolsep}{6pt}
    \begin{tabular}{lccc}
        \toprule
        Method & FID $\downarrow$ & Time (s) $\downarrow$ & $\times$ Speed $\uparrow$ \\
        \midrule
        Sig-DEG-A (8) & \pmc{3.88}{0.03} & \pmc{39.33}{0.02} & 115 \\
        Sig-DEG (8) & \pmc{4.47}{0.05} & \pmc{39.21}{0.20} & 115 \\
        DPM-Solver++, ODE (8) & \pmc{10.45}{0.07} & 26.3 & 172 \\
        DPM-Solver++, SDE (8) & 26.63 & 26.4 & 171 \\
        \midrule
        EDM Heun (8) & 6.62 & 50.8 & 89 \\
        Sig-DEG (16) & \pmc{3.93}{0.03} & \pmc{78.55}{0.19} & 58 \\
        DPM-Solver++, ODE (16) & 4.47 & 52.4 & 86 \\
        \midrule
        Sig-DEG (32) & \pmc{3.65}{0.02} & \pmc{157.24}{0.05} & 29 \\
        DPM-Solver++, ODE (32) & \pmc{3.25}{0.03} & 105.2 & 43 \\
        EDM Heun (18) & 3.09 & 119.6 & 38 \\
        \midrule
        Sig-DEG (64) & \pmc{3.42}{0.04} & \pmc{314.22}{0.31} & 14 \\
        Teacher (2048) & \pmc{3.63}{0.05} & \pmc{4517.0}{5.9} & -- \\
        \bottomrule
    \end{tabular}
    \caption{\textbf{CIFAR-10}: Sig-DEG against training-free solvers run on the teacher's network: DPM-Solver++(2M) \citep{lu2025dpm} as an ODE and an SDE solver, with the better of two step schedules ($\rho=7$ and uniform in log-SNR) per budget, and the Heun sampler of \citet{edm}, which evaluates the network twice per step (once on the last), so that EDM Heun (8) and (18) cost 15 and 35 evaluations. Rows are grouped by sampling cost. FID on 50,000 samples, mean $\pm$ SD over five sampling seeds where given; time and speed-up as in Table~\ref{tab:cifar10}.}
    \label{tab:cifar_baselines}
\end{table}

\providecommand{\Nf}{N_f}

\subsection{Rough Bergomi}
\label{Rough_Bergomi}
\subsubsection{Data and teacher}
The rough Bergomi (rBergomi) price process, introduced by \cite{bayer2016pricing}, defines the spot price process as

\begin{align}
    S_t := \exp \left\{ \int_0^t \sqrt{ V_u } \mathrm{d}B_u - \frac{1}{2}\int_0^t V_u \mathrm{d}u \right\},\quad B_u:=\rho W_u^1 + \sqrt{1 - \rho^2}W_u^2, \\
    V_t := \xi\ \exp \left\{ \eta Y^a_t - \frac{\eta^2}{2} t^{2a + 1}\right\}, \quad Y_t^a := \sqrt{2a + 1} \int_0^t (t - u)^a  \mathrm{d}W^1_u,
\end{align}
for Brownian motions $W^1, W^2$. We sample the process with the help of the hybrid scheme from \cite{bennedsen2017hybrid}.
Our dataset is composed of 100 000 sample paths, among which 60$\%$ is used for training, 20$\%$ for validation, and the remaining 20$\%$  is used for testing.

We adopt the standard configuration used in the rough Bergomi dataset, with parameters set as follows: $\xi = 0.235^2$, $\eta = 1.9$, $\rho = -0.9$, $a = -0.35$ (i.e., Hurst parameter of $0.15$), $T = 1.0$, and $S_0 = 1.0$. Each path is sampled on a uniform grid of 25 steps on $[0, T]$, and the models are trained on the standardised log-increments of the price and variance series.
The teacher is a VP diffusion with 1000 steps on the Residual Architecture of Appendix~\ref{sec:1d_network} (width 256, three residual blocks); the 1500-step teacher of Table~\ref{tab::rberg} has the same architecture and is trained separately.

\subsubsection{Network and training}
Sig-DEG students use the Residual Architecture with width 192 unless stated otherwise (Table~\ref{tab:rberg_capacity}). Students are trained for up to 600 epochs with batch size 512 and Adam (learning rate $10^{-3}$, cosine schedule); every 5 epochs the checkpoint is evaluated on the validation set and the one with the lowest marginal distribution loss is kept, with early stopping after 20 evaluations without improvement. Progressive distillation (Table~\ref{tab:rberg_training}) trains each child for 200 epochs with the same optimiser and selection rule on its frozen parent's rollout (fresh noise every mini-batch, MSE); each child is initialised from its parent, except the $5 \to 1$ child, which is initialised afresh because initialising it from the 5-step parent stalled at a histogram loss of $0.234$. A PD child has the same network and sampling time as the directly trained model with the same number of steps.

\subsubsection{Test metrics}
To evaluate the quality of the generated data, we use several quantitative metrics that assess both marginal distributions and temporal dependencies.

\begin{enumerate}
    \item \textbf{Marginal Distribution Loss.} To evaluate the alignment between the marginal distributions of the generated and reference data, we compute histograms of the marginal probability density functions (PDFs) and measure the average absolute difference between the empirical distributions.

    \item \textbf{Cross-correlation loss (CrossCorr).} Each channel is standardised over paths and time, and for lags $k = 0,\dots,5$ and channel pairs $i \ge j$ we compute the lagged cross-correlation $c_k^{(i,j)}(X) = \mathbb{E}\big[X^{(i)}_{t+k} X^{(j)}_{t}\big]$, averaged over paths and time. With $X$ the real and $\hat X$ the generated paths, the loss is the mean absolute difference
    \begin{equation}
    \mathcal{L}_{\mathrm{CrossCorr}}(X, \hat{X}) = \frac{1}{6\,P} \sum_{k=0}^{5} \sum_{i \ge j} \left\lvert c_k^{(i,j)}(X) - c_k^{(i,j)}(\hat{X}) \right\rvert,
    \end{equation}
    where $P = d(d+1)/2$ is the number of channel pairs ($d = 2$).

    \item \textbf{Autocorrelation loss (AutoCorr).} Let $\mathrm{AC}^{(i)}_k(\cdot)$ be the empirical autocorrelation of channel $i$ at lag $k$, pooled over paths and time. The loss is the Euclidean norm over the lags $k = 0,\dots,5$ of the difference, averaged over the channels:
    \begin{equation}
    \mathcal{L}_{\mathrm{AutoCorr}}(X, \hat{X}) = \frac{1}{d} \sum_{i=1}^{d} \Big( \sum_{k=0}^{5} \big( \mathrm{AC}^{(i)}_k(X) - \mathrm{AC}^{(i)}_k(\hat{X}) \big)^2 \Big)^{1/2}.
    \end{equation}

    \item \textbf{Discriminative Score (D-score).} A two-layer GRU classifier (hidden size 32) is trained for 30 epochs to separate real from generated paths, on 80\% of the real and generated test samples; the score is $\lvert \mathrm{acc} - 0.5 \rvert$, where $\mathrm{acc}$ is its accuracy on the remaining 20\%.

    \item \textbf{Predictive Score (P-score).} Train on synthetic, test on real: a three-layer LSTM (hidden size 64) is trained for 50 epochs on generated paths to predict the second half of a path from its first half, and the score is its mean absolute error on real paths.
\end{enumerate}
All metrics are computed on 20k generated paths against the test set, and each evaluation is repeated 10 times.

\subsubsection{Consistency-model baseline}\label{app:rberg_cm}
The CM rows of Tables~\ref{tab::rberg} and~\ref{tab:rberg_capacity} use the recipe of Appendix~\ref{app:cm}, with the 1000-step teacher and the Residual Architecture of Section~\ref{sec:1d_network} at the width of the Sig-DEG generator ($F_\theta$ has slightly fewer parameters because it lacks the partial-signature input). Training: Adam (learning rate $5 \times 10^{-4}$, cosine schedule), batch size 512, up to 800 epochs with early stopping after 20 evaluations without improvement, checkpoint selected on validation histogram loss every 5 epochs. One trained model is scored at $K = 10, 5, 1$ steps; the width sweep of Table~\ref{tab:rberg_capacity} trains one CM per width with the same recipe. The CM for the 1500-step teacher is trained with the same network and recipe.

\subsubsection{Results}\label{subsec:results}
\paragraph{Comparison with CM.}
\takeaway{Table~\ref{tab::rberg} compares Sig-DEG with CM at the same number of steps, for a 1000-step and a 1500-step teacher. With either teacher, Sig-DEG has the lower histogram loss, D-score and P-score at every number of steps, and with 10 steps it beats its teacher's histogram loss ($0.0323$ vs.\ $0.0341$, and $0.0337$ vs.\ $0.0362$) while sampling $106\times$ and $157\times$ faster.}
\begin{table*}[t]
  \centering
  \small
  \renewcommand{\arraystretch}{1.0} %
  \setlength{\tabcolsep}{3pt}
  \resizebox{\linewidth}{!}{ %
  \begin{tabular}{rccccccc}
    \toprule
    & Teacher & \multicolumn{2}{c}{10 steps} & \multicolumn{2}{c}{5 steps} & \multicolumn{2}{c}{1 step} \\
    \cmidrule(lr){3-4}\cmidrule(lr){5-6}\cmidrule(lr){7-8}
    Metric & & SD & CM & SD & CM & SD & CM \\
    \midrule
    \multicolumn{8}{l}{\emph{1000-step teacher}} \\
    Hist. $\downarrow$ & \pmc{0.0341}{0.0013} & \pmcb{0.0323}{0.0005} & \pmc{0.1340}{0.0015} & \pmcb{0.0348}{0.0010} & \pmc{0.1114}{0.0010} & \pmcb{0.0467}{0.0008} & \pmc{0.0943}{0.0007} \\
    CrossCorr $\downarrow$ & \pmc{0.1368}{0.0081} & \pmc{0.1191}{0.0095} & \pmcb{0.1187}{0.0091} & \pmc{0.1147}{0.0124} & \pmcb{0.1111}{0.0092} & \pmcb{0.0731}{0.0184} & \pmc{0.1080}{0.0133} \\
    AutoCorr $\downarrow$ & \pmc{0.2548}{0.0168} & \pmc{0.2058}{0.0341} & \pmcb{0.1945}{0.0482} & \pmc{0.1943}{0.0334} & \pmcb{0.1566}{0.0316} & \pmcb{0.1415}{0.0437} & \pmc{0.1542}{0.0465} \\
    D-score $\downarrow$ & \pmc{0.2293}{0.0079} & \pmcb{0.2446}{0.0150} & \pmc{0.4441}{0.0155} & \pmcb{0.2557}{0.0128} & \pmc{0.4161}{0.0304} & \pmcb{0.2191}{0.0134} & \pmc{0.4281}{0.0116} \\
    P-score $\downarrow$ & \pmc{0.0573}{0.0005} & \pmcb{0.0575}{0.0007} & \pmc{0.0699}{0.0028} & \pmcb{0.0578}{0.0004} & \pmc{0.0663}{0.0021} & \pmcb{0.0575}{0.0004} & \pmc{0.0649}{0.0011} \\
    Time (ms) $\downarrow$ & \pmc{699.75}{0.65} & \pmcb{6.620}{0.027} & \pmc{7.911}{0.055} & \pmcb{3.380}{0.015} & \pmc{3.937}{0.033} & \pmc{0.855}{0.026} & \pmcb{0.731}{0.030} \\
    $\times$ Speed $\uparrow$ & -- & \numb{106} & $88$ & \numb{207} & $178$ & $818$ & \numb{957} \\
    \midrule
    \multicolumn{8}{l}{\emph{1500-step teacher}} \\
    Hist. $\downarrow$ & \pmc{0.0362}{0.0011} & \pmcb{0.0337}{0.0005} & \pmc{0.1060}{0.0013} & \pmcb{0.0337}{0.0008} & \pmc{0.0919}{0.0012} & \pmcb{0.0497}{0.0010} & \pmc{0.1233}{0.0014} \\
    CrossCorr $\downarrow$ & \pmc{0.1492}{0.0371} & \pmc{0.1238}{0.0108} & \pmcb{0.1143}{0.0131} & \pmcb{0.1035}{0.0166} & \pmc{0.1095}{0.0105} & \pmc{0.1256}{0.0039} & \pmcb{0.1169}{0.0150} \\
    AutoCorr $\downarrow$ & \pmc{0.3477}{0.1678} & \pmc{0.2545}{0.1173} & \pmcb{0.1773}{0.0463} & \pmc{0.1698}{0.0510} & \pmcb{0.1587}{0.0372} & \pmc{0.2291}{0.0094} & \pmcb{0.2060}{0.0790} \\
    D-score $\downarrow$ & \pmc{0.1805}{0.0076} & \pmcb{0.2445}{0.0202} & \pmc{0.4337}{0.0109} & \pmcb{0.2554}{0.0300} & \pmc{0.4082}{0.0249} & \pmcb{0.2208}{0.0191} & \pmc{0.4462}{0.0175} \\
    P-score $\downarrow$ & \pmc{0.0388}{0.0002} & \pmcb{0.0390}{0.0003} & \pmc{0.0430}{0.0012} & \pmcb{0.0391}{0.0003} & \pmc{0.0415}{0.0009} & \pmcb{0.0396}{0.0006} & \pmc{0.0425}{0.0006} \\
    Time (ms) $\downarrow$ & \pmc{1051.65}{1.12} & \pmcb{6.693}{0.014} & \pmc{7.930}{0.037} & \pmcb{3.432}{0.019} & \pmc{3.944}{0.030} & \pmc{0.860}{0.033} & \pmcb{0.774}{0.133} \\
    $\times$ Speed $\uparrow$ & -- & \numb{157} & $133$ & \numb{306} & $267$ & $1223$ & \numb{1360} \\
    \bottomrule
  \end{tabular}
  }
  \caption{\textbf{rBerg}: Sig-DEG (SD) vs.\ CM at the same number of steps, with a 1000-step and a 1500-step teacher (mean $\pm$ SD over 10 sampling runs of 20k paths; time for 20k paths). P-scores are comparable only within a block (different cut-offs).}
  \label{tab::rberg}\label{tab:rberg_T1500}
\end{table*}

\takeaway{Figure~\ref{fig:rberg_radar} summarizes Table~\ref{tab::rberg}, one wheel per teacher: with either teacher, Sig-DEG reaches about $0.7$ or more of the best value on every spoke, while CM drops to $0.24$--$0.37$ on the histogram loss and $0.40$--$0.55$ on the D-score.}
\begin{figure}[ht]
    \centering
    \includegraphics{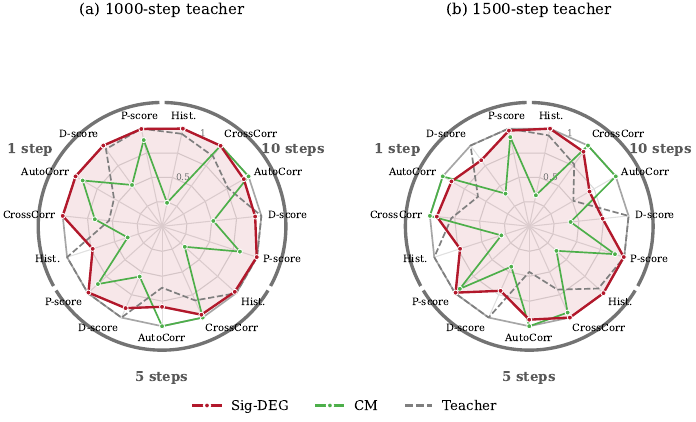}  %
    \caption{\textbf{rBerg}: Table~\ref{tab::rberg} as best value / model value per metric and number of steps (rim: best; larger is better), with (a) the 1000-step and (b) the 1500-step teacher.}
    \label{fig:rberg_radar}
\end{figure}

\paragraph{Ablations.}

\takeaway{Table~\ref{tab:rberg_teacher_coarse} runs the 1000-step teacher itself on $N$ steps. As a reverse SDE it fails at 1--5 steps (D-score $\approx 0.5$) and needs between 50 and 100 steps to reach one-step Sig-DEG's histogram loss; its probability-flow ODE does better on coarse meshes but stays at $0.055$--$0.063$ for every $N \ge 10$, well behind Sig-DEG ($0.032$ at 10 steps).}
\begin{table}[ht]
  \centering
  \footnotesize
  \renewcommand{\arraystretch}{0.95}
  \setlength{\tabcolsep}{4pt}
  \begin{tabular}{rcc|cc|c}
    \toprule
    & \multicolumn{2}{c|}{Teacher, SDE on $N$ steps} & \multicolumn{2}{c|}{Teacher, PF-ODE on $N$ steps} & Sig-DEG ($N$) \\
    $N$ & Hist.\ $\downarrow$ & D-score $\downarrow$ & Hist.\ $\downarrow$ & D-score $\downarrow$ & Hist.\ $\downarrow$ \\
    \midrule
    1 & \pmc{0.7818}{0.0007} & \pmc{0.4997}{0.0002} & \pmc{0.3041}{0.0016} & \pmc{0.4897}{0.0030} & $0.0467$ \\
    2 & \pmc{0.8725}{0.0004} & \pmc{0.4999}{0.0002} & \pmc{0.2403}{0.0015} & \pmc{0.4862}{0.0023} & $0.0384$ \\
    5 & \pmc{0.6342}{0.0012} & \pmc{0.4949}{0.0012} & \pmc{0.0699}{0.0006} & \pmc{0.3433}{0.0263} & $0.0348$ \\
    10 & \pmc{0.2276}{0.0013} & \pmc{0.4307}{0.0094} & \pmc{0.0557}{0.0011} & \pmc{0.2615}{0.0192} & $0.0323$ \\
    20 & \pmc{0.1006}{0.0009} & \pmc{0.3265}{0.0215} & \pmc{0.0553}{0.0011} & \pmc{0.2623}{0.0176} & -- \\
    50 & \pmc{0.0484}{0.0009} & \pmc{0.2366}{0.0105} & \pmc{0.0592}{0.0010} & \pmc{0.2824}{0.0158} & -- \\
    100 & \pmc{0.0408}{0.0011} & \pmc{0.2345}{0.0145} & \pmc{0.0607}{0.0008} & \pmc{0.2753}{0.0244} & -- \\
    250 & \pmc{0.0361}{0.0012} & \pmc{0.2309}{0.0078} & \pmc{0.0626}{0.0015} & \pmc{0.2801}{0.0131} & -- \\
    1000 & \pmc{0.0345}{0.0014} & \pmc{0.2370}{0.0097} & \pmc{0.0628}{0.0015} & \pmc{0.2658}{0.0191} & -- \\
    \bottomrule
  \end{tabular}
  \caption{\textbf{rBerg}: the 1000-step teacher sampled on $N$ steps with its reverse SDE (Euler--Maruyama) or its probability-flow ODE, vs.\ Sig-DEG ($N$) (10 sampling runs; $N = 1000$: an independent re-evaluation of the teacher).}
  \label{tab:rberg_teacher_coarse}
\end{table}

\takeaway{Table~\ref{tab:rberg_bmonly} retrains the directly trained models without the second-order term $\int_{t_i}^{t_{i+1}} (W_u - W_{t_i})\,du$ of the partial signature: the histogram loss rises at every number of steps, by $0.001$--$0.003$, while the correlation errors are no larger and are about halved at one step, so on rBerg the term gives a small gain in the marginal distributions (comparable to run-to-run variation except at one step).}
\begin{table}[ht]
  \centering
  \footnotesize
  \renewcommand{\arraystretch}{0.95}
  \setlength{\tabcolsep}{3pt}
  \begin{tabular}{clcccc}
    \toprule
    Steps & Features & Hist.\ $\downarrow$ & CrossCorr $\downarrow$ & AutoCorr $\downarrow$ & D-score $\downarrow$ \\
    \midrule
    10 & full & \pmcb{0.0323}{0.0005} & \pmc{0.119}{0.010} & \pmc{0.206}{0.034} & \pmcb{0.245}{0.015} \\
       & $\Delta W$ only & \pmc{0.0355}{0.0007} & \pmcb{0.109}{0.017} & \pmcb{0.176}{0.065} & \pmc{0.248}{0.013} \\
    \midrule
    5 & full & \pmcb{0.0348}{0.0010} & \pmc{0.115}{0.012} & \pmc{0.194}{0.033} & \pmc{0.256}{0.013} \\
      & $\Delta W$ only & \pmc{0.0369}{0.0009} & \pmcb{0.108}{0.022} & \pmcb{0.175}{0.059} & \pmcb{0.250}{0.016} \\
    \midrule
    2 & full & \pmcb{0.0384}{0.0008} & \pmc{0.130}{0.009} & \pmc{0.238}{0.029} & \pmc{0.267}{0.035} \\
      & $\Delta W$ only & \pmc{0.0397}{0.0008} & \pmcb{0.117}{0.019} & \pmcb{0.203}{0.050} & \pmcb{0.243}{0.008} \\
    \midrule
    1 & full & \pmcb{0.0467}{0.0008} & \pmc{0.073}{0.018} & \pmc{0.142}{0.044} & \pmcb{0.219}{0.013} \\
      & $\Delta W$ only & \pmc{0.0493}{0.0006} & \pmcb{0.035}{0.012} & \pmcb{0.081}{0.032} & \pmc{0.242}{0.013} \\
    \bottomrule
  \end{tabular}
  \caption{\textbf{rBerg}: Sig-DEG with the full partial signature $(\Delta W, \int (W-W_s)\,du)$ vs.\ $\Delta W$ alone (10 sampling runs; bold: the better one).}
  \label{tab:rberg_bmonly}
\end{table}

\paragraph{Network width.}

\takeaway{Table~\ref{tab:rberg_capacity} varies the network width at 5 steps: Sig-DEG samples faster than CM at every width and has a clearly lower histogram loss and D-score from width 128 up; its histogram loss matches the teacher's at width 256 ($0.0335$ vs.\ $0.0341$) and is below it at width 384 ($0.0306$).}

\begin{table}[ht]
    \centering
    \footnotesize
    \renewcommand{\arraystretch}{1.0}
    \setlength{\tabcolsep}{2pt}
    \resizebox{\linewidth}{!}{
    \begin{tabular}{rccccccc}
      \toprule
      & & \multicolumn{2}{c}{Hist. $\downarrow$} & \multicolumn{2}{c}{D-score $\downarrow$} & \multicolumn{2}{c}{Time (ms)} \\
      \cmidrule(lr){3-4}\cmidrule(lr){5-6}\cmidrule(lr){7-8}
      Width & \# Params (SD/CM) & SD & CM & SD & CM & SD & CM \\
      \midrule
      64 & 66k / 60k & \pmcb{0.1344}{0.0010} & \pmc{0.1387}{0.0007} & \pmcb{0.4561}{0.0144} & \pmc{0.4596}{0.0134} & \pmcb{2.759}{0.029} & \pmc{3.369}{0.013} \\
      128 & 223k / 211k & \pmcb{0.0354}{0.0009} & \pmc{0.1224}{0.0012} & \pmcb{0.2775}{0.0234} & \pmc{0.4530}{0.0097} & \pmcb{2.776}{0.018} & \pmc{3.269}{0.020} \\
      192 & 471k / 451k & \pmcb{0.0348}{0.0010} & \pmc{0.1114}{0.0010} & \pmcb{0.2557}{0.0128} & \pmc{0.4161}{0.0304} & \pmcb{3.380}{0.015} & \pmc{3.937}{0.033} \\
      256 & 808k / 782k & \pmcb{0.0335}{0.0008} & \pmc{0.0903}{0.0010} & \pmcb{0.2453}{0.0138} & \pmc{0.4079}{0.0163} & \pmcb{3.956}{0.013} & \pmc{4.268}{0.037} \\
      384 & 1753k / 1714k & \pmcb{0.0306}{0.0009} & \pmc{0.0892}{0.0011} & \pmcb{0.2402}{0.0071} & \pmc{0.3957}{0.0263} & \pmcb{5.453}{0.020} & \pmc{5.771}{0.018} \\
      \bottomrule
    \end{tabular}
    }
    \caption{\textbf{rBerg}: network width at 5 sampling steps, Sig-DEG (SD) vs.\ CM.}
    \label{tab:rberg_capacity}
\end{table}

\paragraph{Training strategies.}
\takeaway{Table~\ref{tab:rberg_training} compares the training strategies at the same number of steps. The resolution curriculum, which starts on a 125-step mesh (250 at 10 and 2 steps), saves $8$--$35\%$ of the training time, with a similar histogram loss at 10 and 5 steps, a lower one at 2 steps ($0.0338$ vs.\ $0.0384$) and a higher one at one step ($0.0504$ vs.\ $0.0467$). PD from Sig-DEG (10) lowers the histogram loss at 5, 2 and 1 steps ($0.0316$, $0.0345$, and $0.0418$ along $10 \to 5 \to 1$, vs.\ $0.0348$, $0.0384$ and $0.0467$) and, since each stage costs under an hour, takes $5.5$--$6.1$\,h in total instead of $6.2$--$13.1$\,h; a curriculum-trained parent saves a further $0.4$\,h on every child at comparable quality.}
\begin{table*}[t]
  \centering
  \footnotesize
  \renewcommand{\arraystretch}{1.0}
  \setlength{\tabcolsep}{3pt}
  \resizebox{\linewidth}{!}{%
  \begin{tabular}{clcccccc}
    \toprule
    Steps & Training & Hist. $\downarrow$ & CrossCorr $\downarrow$ & AutoCorr $\downarrow$ & D-score $\downarrow$ & P-score $\downarrow$ & Train (h) $\downarrow$ \\
    \midrule
    10 & Direct & \pmcb{0.0323}{0.0005} & \pmc{0.1191}{0.0095} & \pmcb{0.2058}{0.0341} & \pmcb{0.2446}{0.0150} & \pmc{0.0575}{0.0007} & $4.83$ \\
     & Curriculum & \pmc{0.0332}{0.0010} & \pmcb{0.1148}{0.0116} & \pmc{0.2070}{0.0380} & \pmc{0.2489}{0.0169} & \pmcb{0.0572}{0.0002} & \numb{4.46} \\
    \midrule
    5 & Direct & \pmc{0.0348}{0.0010} & \pmc{0.1147}{0.0124} & \pmc{0.1943}{0.0334} & \pmc{0.2557}{0.0128} & \pmc{0.0578}{0.0004} & $6.24$ \\
     & Curriculum & \pmc{0.0347}{0.0009} & \pmc{0.1066}{0.0143} & \pmc{0.1767}{0.0391} & \pmc{0.2536}{0.0130} & \pmc{0.0575}{0.0004} & \numb{4.98} \\
     & PD $10\to5$, direct parent & \pmcb{0.0316}{0.0009} & \pmc{0.1101}{0.0140} & \pmc{0.1763}{0.0401} & \pmc{0.2670}{0.0164} & \pmcb{0.0571}{0.0002} & $5.49$ \\
     & PD $10\to5$, curriculum parent & \pmc{0.0324}{0.0006} & \pmcb{0.1063}{0.0160} & \pmcb{0.1689}{0.0460} & \pmcb{0.2345}{0.0074} & \pmc{0.0572}{0.0002} & $5.12$ \\
    \midrule
    2 & Direct & \pmc{0.0384}{0.0008} & \pmc{0.1295}{0.0088} & \pmc{0.2375}{0.0292} & \pmc{0.2669}{0.0353} & \pmc{0.0573}{0.0006} & $11.59$ \\
     & Curriculum & \pmc{0.0338}{0.0011} & \pmc{0.1342}{0.0065} & \pmc{0.2529}{0.0188} & \pmcb{0.2432}{0.0181} & \pmc{0.0573}{0.0002} & $7.59$ \\
     & PD $10\to2$, direct parent & \pmc{0.0345}{0.0007} & \pmc{0.0994}{0.0239} & \pmc{0.1812}{0.0475} & \pmc{0.2524}{0.0086} & \pmcb{0.0570}{0.0002} & $5.50$ \\
     & PD $10\to2$, curriculum parent & \pmcb{0.0336}{0.0009} & \pmcb{0.0803}{0.0129} & \pmcb{0.1371}{0.0325} & \pmc{0.2519}{0.0165} & \pmcb{0.0570}{0.0001} & \numb{5.13} \\
    \midrule
    1 & Direct & \pmc{0.0467}{0.0008} & \pmc{0.0731}{0.0184} & \pmc{0.1415}{0.0437} & \pmcb{0.2191}{0.0134} & \pmc{0.0575}{0.0004} & $13.07$ \\
     & Curriculum & \pmc{0.0504}{0.0008} & \pmc{0.1154}{0.0152} & \pmc{0.2101}{0.0351} & \pmc{0.2519}{0.0100} & \pmc{0.0575}{0.0004} & $10.73$ \\
     & PD $10\to5\to1$, direct parent & \pmcb{0.0418}{0.0008} & \pmcb{0.0627}{0.0169} & \pmcb{0.0925}{0.0357} & \pmc{0.2916}{0.0191} & \pmc{0.0575}{0.0004} & $6.14$ \\
     & PD $10\to5\to1$, curriculum parent & \pmc{0.0462}{0.0010} & \pmc{0.0737}{0.0148} & \pmc{0.1241}{0.0387} & \pmc{0.2493}{0.0059} & \pmc{0.0574}{0.0003} & $5.77$ \\
     & PD $10\to2\to1$, direct parent & \pmc{0.0481}{0.0007} & \pmc{0.0727}{0.0217} & \pmc{0.1256}{0.0478} & \pmc{0.2646}{0.0076} & \pmc{0.0573}{0.0002} & $6.12$ \\
     & PD $10\to2\to1$, curriculum parent & \pmc{0.0441}{0.0006} & \pmc{0.0778}{0.0204} & \pmc{0.1339}{0.0452} & \pmc{0.2439}{0.0099} & \pmc{0.0572}{0.0002} & \numb{5.75} \\
     & PD $5\to1$, direct parent & \pmc{0.0438}{0.0010} & \pmc{0.0895}{0.0235} & \pmc{0.1511}{0.0503} & \pmc{0.2641}{0.0157} & \pmc{0.0578}{0.0005} & $6.89$ \\
     & PD $2\to1$, direct parent & \pmc{0.0451}{0.0012} & \pmc{0.1179}{0.0114} & \pmc{0.2051}{0.0355} & \pmc{0.2505}{0.0199} & \pmcb{0.0571}{0.0002} & $12.22$ \\
    \bottomrule
  \end{tabular}
  }
  \caption{\textbf{rBerg}: training strategies at the same number of sampling steps: direct training on the 1000-step mesh, the resolution curriculum, and PD from a directly trained or curriculum-trained parent (10 sampling runs; bold: best per number of steps). Train: total training time, including validation and the parents.}\label{tab:rberg_training}\label{tab:rberg_curriculum}\label{tab:rberg_pd}\label{tab:rberg_curr_pd}
\end{table*}

\takeaway{Figure~\ref{fig:rberg_curves} shows the corresponding validation curves: at 2 steps the curriculum spends its early epochs on cheaper coarse meshes and ends at the level of direct training in less time (left), and one-step Sig-DEG distilled from the curriculum-trained Sig-DEG (2) reaches a lower validation loss than direct one-step training ($0.039$ vs.\ $0.049$) with a total budget of $8.2$ instead of $13.1$\,h (right).}

\begin{figure*}[ht]
    \centering
    \begin{subfigure}[t]{0.48\textwidth}
        \centering
        \includegraphics[width=\linewidth]{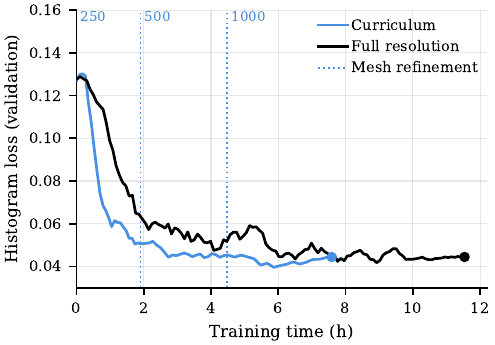}
        \caption{Sig-DEG (2), curriculum vs.\ full-resolution training, 5-point moving average; labels: teacher steps of each mesh; dots: end of training.}
        \label{fig:rberg_curves_curriculum_oldlogs}
    \end{subfigure}\hfill
    \begin{subfigure}[t]{0.48\textwidth}
        \centering
        \includegraphics[width=\linewidth]{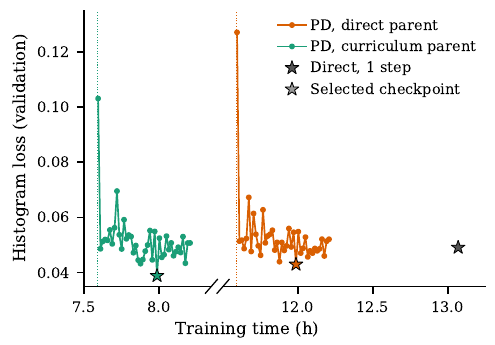}
        \caption{PD $2 \to 1$ from the direct and the curriculum-trained Sig-DEG (2); time includes the parent's (dotted: end of parent training). Stars: selected checkpoints, and direct Sig-DEG (1) after $13.1$\,h.}
        \label{fig:rberg_curves_pd}
    \end{subfigure}
    \caption{\textbf{rBerg}: validation histogram loss (2,048 paths) against training time, validation included (seed 42).}
    \label{fig:rberg_curves}
\end{figure*}

\paragraph{Marginal distributions.}
\takeaway{Figure~\ref{fig::rberg_dist_joint} compares the marginals of the generated paths at six path time steps: the 10-step Sig-DEG reproduces the teacher's marginals at every step shown, including the heavy tails of the price returns.}

\begin{figure*}[ht]
    \centering
    \includegraphics[width=\linewidth]{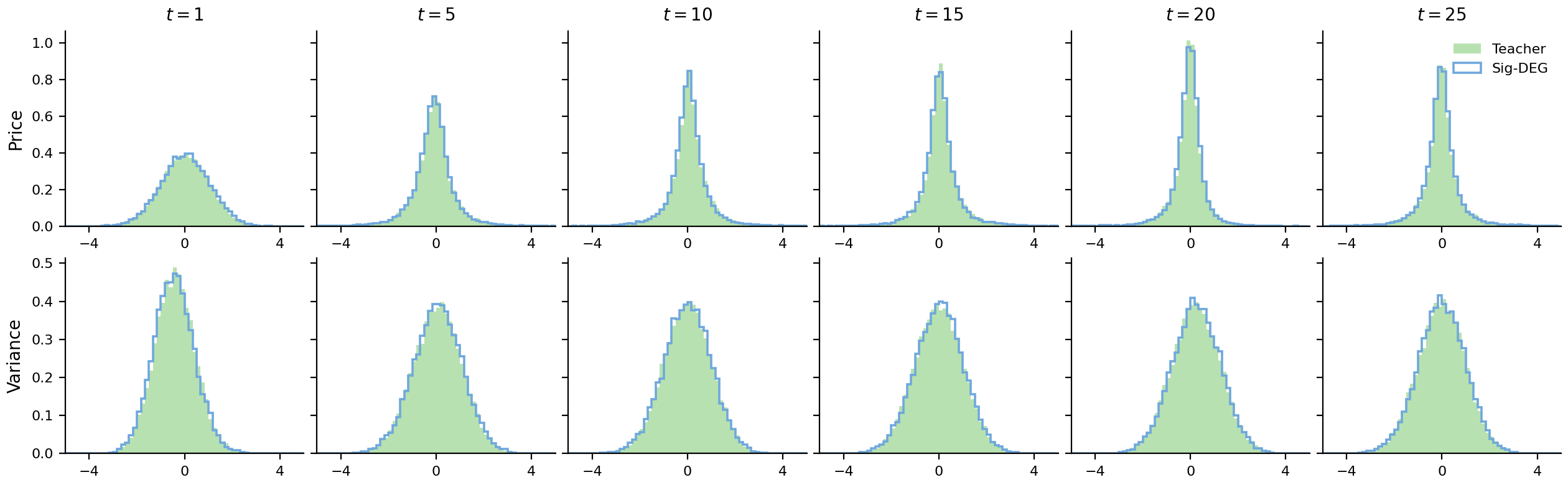}
    \caption{Marginals of generated \textbf{rBerg} paths at path times 1, 5, 10, 15, 20 and 25 (of 25; final samples): standardised price log-returns (top) and variance log-increments (bottom). Filled green: 1000-step teacher; blue: Sig-DEG (10).}
    \label{fig::rberg_dist_joint}
\end{figure*}

\subsection{Schrödinger-bridge teachers}
\label{app:sb}

Unlike the denoising diffusion teachers of Sections~\ref{sec:MNIST} and~\ref{app:cifar}, this section distils a neural Schrödinger bridge, to check that Sig-DEG does not depend on the denoising structure.
A Schrödinger bridge transports a source distribution $\mu_0$ to a target $\mu_T$ along the SDE $dX_t = b_\theta(t, X_t)\,dt + \sigma\,dW_t$, possibly with an added path cost (obstacles, congestion). The student therefore starts from samples of $\mu_0$ rather than from noise, the drift is not a score, and the paths can be strongly non-linear.

We use the pretrained two-dimensional checkpoints released with \citet{noble2026twisted} under the MIT licence. Each teacher is a frozen MLP drift for the SDE above with $\sigma = 2$ and $T = 2$, sampled with the authors' $200$-step Euler--Maruyama scheme; all Gaussians below have variance $0.5$ per coordinate.
We report here three of seven two-dimensional teachers from the same source; the comparison holds qualitatively on the other four.
\begin{itemize}
    \item \textbf{V-neck} (\texttt{gsbm\_vneck\_2d}): one Gaussian centred at $(-7, 0)$ is moved to one centred at $(7, 0)$ through an hourglass-shaped wall. The admissible region is $\lvert x_2\rvert < \sqrt{0.36 + 5x_1^2}$, so every path must pass a neck of width $1.2$ at $x_1 = 0$. Trained with GSBM \citep{liu2024gsbm}; the best-converged teacher.
    \item \textbf{GMM with obstacles} (\texttt{gsbm\_gmm\_2d}): four Gaussians on a circle of radius $4$ are spread to eight Gaussians on a circle of radius $16$, around three circular obstacles of radius $3.5$ and with a congestion penalty.
    Each source mode splits between neighbouring target modes. Trained with GSBM.
    \item \textbf{Unconstrained GMM} (\texttt{dsbm\_gmm\_2d}): two Gaussians centred at $(-10, \pm 10)$ are moved to two Gaussians centred at $(10, \pm 10)$ with no path cost. Trained with DSBM \citep{shi2023dsbm}; the paths are nearly straight.
\end{itemize}

\subsubsection{Network Details and Training}
\label{app:sb_network}
For this experiment, the Sig-DEG network $s_\theta$ is an MLP with three hidden layers of width $128$ and SiLU activations ($35{,}330$ parameters).
The \emph{increment-only} ablation is the same network with the cross-integral removed from the input ($35{,}074$ parameters), i.e.\ a student that sees only the first level of the signature.
We also report the teacher's performance when integrated with $C$ Euler steps as \emph{coarse Euler}.

Training uses $q=1$ and we use logMSE, the base loss also in the other experiments. We train four students at different resolution levels, i.e. $C \in \{2, 4, 5, 10, 20\}$, for $4{,}000$ updates (batch $512$, Adam, cosine learning rate from $2\times10^{-3}$ to $5\times10^{-5}$, EMA decay $0.995$) on $4{,}096$ teacher paths; all numbers are the mean $\pm$ standard deviation over three seeds with separate training sets.
The fixed budget favours small $C$: doubling it at $C = 20$ lowers the terminal error of both students on the V-neck by a quarter, so the rise from $C = 10$ to $C = 20$ in Figure~\ref{fig:sb_steps} is at least partly a budget effect. For this reason, we show results for $C\leq10$, as they are directly comparable.

\subsubsection{Evaluation Metrics}
\label{app:sb_metrics}
Terminal samples are compared with the teacher's by the debiased Sinkhorn divergence $\mathcal{S}$ \citep{feydy2019sinkhorn} with cost $\tfrac12\lVert x - y\rVert^2$ and entropic regularisation $\varepsilon = 0.05^2$, so that $\mathcal{S} \approx W_2^2/2$, on $8{,}192$ fresh endpoints per side.
For the multimodal teachers the pooled divergence mostly measures sampling noise in the mode weights, so the terminal $\mathcal{S}$ is computed within each target mode and averaged with the teacher's mode weights; the coupling is measured separately by the total variation between the student's and the teacher's source-to-target mode-transition matrices (mode-TV).
The \emph{joint} $\mathcal{S}$ applies the same estimator, pooled, to the pairs $(X_0, X_T)$ and measures whether the student reproduces the bridge's coupling, not only its marginal.

Every student is scored against one fixed reference sample of the teacher, and the \emph{floor} of each divergence is its value between four further independent teacher samples and that reference (mean $\pm$ SD); hence, a value within two floor SDs is indistinguishable from the teacher.
As a final metric, since the networks are too small for wall-clock times to be meaningful, we report inference cost counted in multiply-adds: $34{,}944$ per Sig-DEG step against $98{,}816$ per teacher step, i.e.\ $0.14$M per sample for Sig-DEG (4) against $19.8$M for the $200$-step teacher, about $140\times$ less.

\subsubsection{Results}
\label{app:sb_results}
Table~\ref{tab:sb} and Figure~\ref{fig:sb_steps} report the results. At few steps the second-level term closes most of the gap to the teacher on the one- and two-mode teachers.
On the V-neck teacher, measured as excess over the floor, Sig-DEG is $12\times$ closer to the teacher than the increment-only student at $C = 2$ and $9\times$ at $C = 4$.
On the unconstrained GMM, Sig-DEG is indistinguishable from the teacher from $C = 2$, the increment-only student only from $C = 10$.
On the GMM with obstacles neither student reaches the floor even at $C = 20$ and the advantage is smaller ($1.4\times$ at $C = 2$, $1.8\times$ at $C = 4$), but it extends to the coupling: at $C = 4$ the mode-TV is $0.034 \pm 0.006$ against $0.052 \pm 0.010$ (floor $0.018$) and the joint divergence falls from $1.79$ to $0.99$ (floor $0.23$).
As for pathwise error, Sig-DEG's is lower, at every $C$ and on every teacher
.
In every studied case, the coarse Euler control is worse than both students for $C \leq 10$; only at $C = 20$ it matches Sig-DEG on the GMM with obstacles (and beats the increment-only student).

\begin{table}[t!]
    \centering
    \footnotesize
    \setlength{\tabcolsep}{2pt}
    \resizebox{\linewidth}{!}{%
\begin{tabular}{llcccccc}
        \toprule
        & & \multicolumn{2}{c}{terminal $\mathcal{S}$ ($\downarrow$)} & \multicolumn{2}{c}{joint $\mathcal{S}$ ($\downarrow$)} & \multicolumn{2}{c}{path MSE ($\downarrow$)} \\
        \cmidrule(lr){3-4} \cmidrule(lr){5-6} \cmidrule(lr){7-8}
        teacher & $C$ & increment-only & Sig-DEG & increment-only & Sig-DEG & increment-only & Sig-DEG \\
        \midrule
        \multirow{4}{*}{V-neck}
               & 2  & \pmc{0.0124}{0.0005} & \pmcb{0.0025}{0.0001} & \pmc{0.0451}{0.0004} & \pmcb{0.0367}{0.0001} & \pmc{0.5930}{0.0113} & \pmcb{0.4143}{0.0045} \\
               & 4  & \pmc{0.0063}{0.0007} & \pmcb{0.0021}{0.0002} & \pmc{0.0400}{0.0005} & \pmcb{0.0368}{0.0006} & \pmc{0.3274}{0.0018} & \pmcb{0.1847}{0.0003} \\
               & 10 & \pmc{0.0021}{0.0002} & \pmcb{0.0016}{0.0001} & \pmc{0.0371}{0.0005} & \pmc{0.0371}{0.0005} & \pmc{0.1353}{0.0007} & \pmcb{0.0458}{0.0013} \\
               & floor & \multicolumn{2}{c}{\pmc{0.0016}{0.0001}} & \multicolumn{2}{c}{\pmc{0.0366}{0.0005}} & \multicolumn{2}{c}{--} \\
        \midrule
        \multirow{4}{*}{GMM with obstacles}
               & 2  & \pmc{1.49}{0.15} & \pmcb{1.10}{0.09} & \pmc{3.50}{0.87} & \pmc{2.73}{0.23} & \pmc{27.5}{2.3} & \pmcb{23.2}{1.0} \\
               & 4  & \pmc{0.424}{0.067} & \pmcb{0.245}{0.027} & \pmc{1.79}{0.35} & \pmcb{0.99}{0.05} & \pmc{19.0}{1.0} & \pmcb{13.7}{1.1} \\
               & 10 & \pmc{0.056}{0.001} & \pmc{0.057}{0.008} & \pmc{0.83}{0.03} & \pmcb{0.64}{0.07} & \pmc{12.2}{1.0} & \pmcb{10.2}{0.2} \\
               & floor & \multicolumn{2}{c}{\pmc{0.0113}{0.0004}} & \multicolumn{2}{c}{\pmc{0.23}{0.03}} & \multicolumn{2}{c}{--} \\
        \midrule
        \multirow{4}{*}{Unconstrained GMM}
               & 2  & \pmc{0.0163}{0.0010} & \pmcb{0.0037}{0.0002} & \pmc{0.0610}{0.0006} & \pmcb{0.0505}{0.0003} & \pmc{0.2322}{0.0013} & \pmcb{0.0457}{0.0004} \\
               & 4  & \pmc{0.0066}{0.0006} & \pmcb{0.0032}{0.0002} & \pmc{0.0531}{0.0005} & \pmcb{0.0510}{0.0004} & \pmc{0.0779}{0.0001} & \pmcb{0.0074}{0.0002} \\
               & 10 & \pmc{0.0038}{0.0010} & \pmc{0.0031}{0.0001} & \pmc{0.0520}{0.0015} & \pmc{0.0520}{0.0011} & \pmc{0.0140}{0.0002} & \pmcb{0.00104}{0.00014} \\
               & floor & \multicolumn{2}{c}{\pmc{0.0031}{0.0005}} & \multicolumn{2}{c}{\pmc{0.0512}{0.0006}} & \multicolumn{2}{c}{--} \\
        \bottomrule
    \end{tabular}%
}
    \caption{\textbf{Schrödinger-bridge teachers}: increment-only student against Sig-DEG for $C$ coarse steps (teacher: $200$ steps), mean $\pm$ SD over three training seeds. Terminal Sinkhorn divergence to the teacher (within-mode for the GMM teachers) and joint divergence of the $(X_0, X_T)$ pairs (pooled), each with its floor (Appendix~\ref{app:sb_metrics}); path MSE under the teacher's Brownian path, for which a perfect student scores $0$. Bold: the lower of the two student means when their difference exceeds the larger of their two SDs.}
    \label{tab:sb}
\end{table}

Figure~\ref{fig:sb_vneck} shows the V-neck teacher next to the two $4$-step students. Both respect the wall: at most $0.6\%$ of their intermediate states fall inside it (a soft penalty, which the teacher's own states cross in $0.04$--$0.19\%$ of cases), against $8$--$17\%$ for coarse Euler at $C \leq 5$; on the GMM with obstacles the students' rate is at most $1.2\%$. At mid-time, the difference between the students is visible (Figure~\ref{fig:sb_vneck_snap}): at $T/2$, the $W_2$ distance is $0.46$ for increment-only and $0.36$ for Sig-DEG (against a $0.24$ floor). However, at $C \leq 4$ both students have a mean state cost (the teacher's near-wall and congestion penalty) significantly higher than the teacher's, which the terminal metrics do not capture; Sig-DEG has the lower cost at $C = 4$ and $10$ but a slightly higher wall-violation rate at $C \leq 5$.

\begin{figure}[t!]
    \centering
    \begin{subfigure}[t]{0.32\textwidth}
        \includegraphics[width=\linewidth]{sb_steps_terminal_vneck.png}
        \vspace{-15pt}\caption{V-neck}\vspace{-5pt}
    \end{subfigure}
    \hfill
    \begin{subfigure}[t]{0.32\textwidth}
        \includegraphics[width=\linewidth]{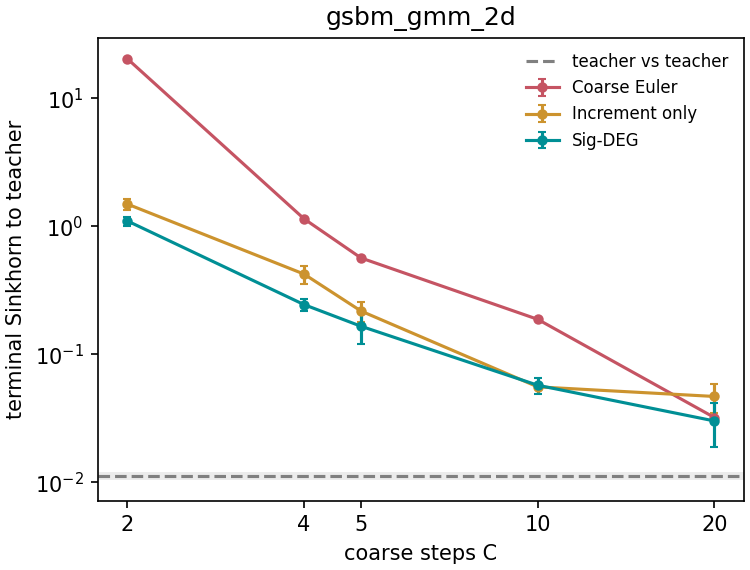}
        \vspace{-15pt}\caption{GMM with obstacles}\vspace{-5pt}
    \end{subfigure}
    \hfill
    \begin{subfigure}[t]{0.32\textwidth}
        \includegraphics[width=\linewidth]{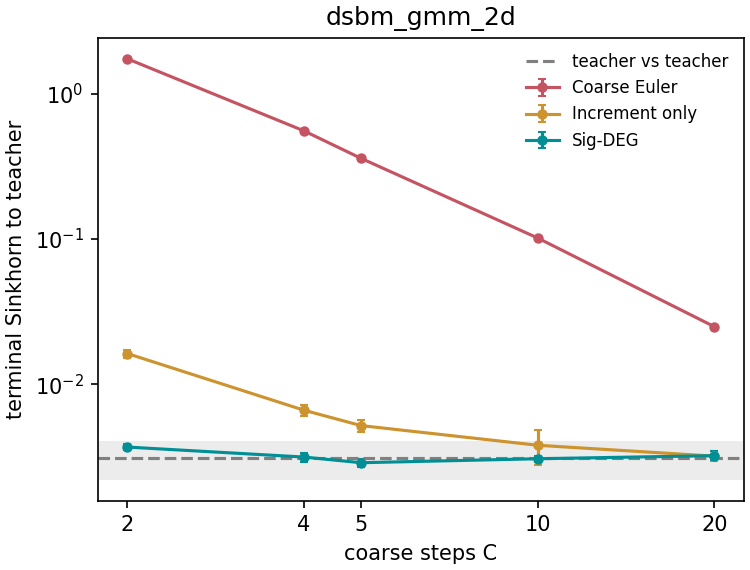}
        \vspace{-15pt}\caption{Unconstrained GMM}\vspace{-5pt}
    \end{subfigure}
    \caption{\textbf{Schrödinger-bridge teachers}: terminal Sinkhorn divergence to the teacher against the number of coarse steps $C$ (log--log; within-mode for the GMM teachers; error bars are the SD over three seeds). The dashed line is the floor and the grey band its $\pm 2$ SD. Teal: Sig-DEG (labelled SigDEG in the panels); orange: increment-only student; red: coarse Euler with the teacher's drift. Panel titles give the checkpoint names.}
    \label{fig:sb_steps}
\end{figure}

\begin{figure}[t!]
    \centering
    \includegraphics[width=\linewidth]{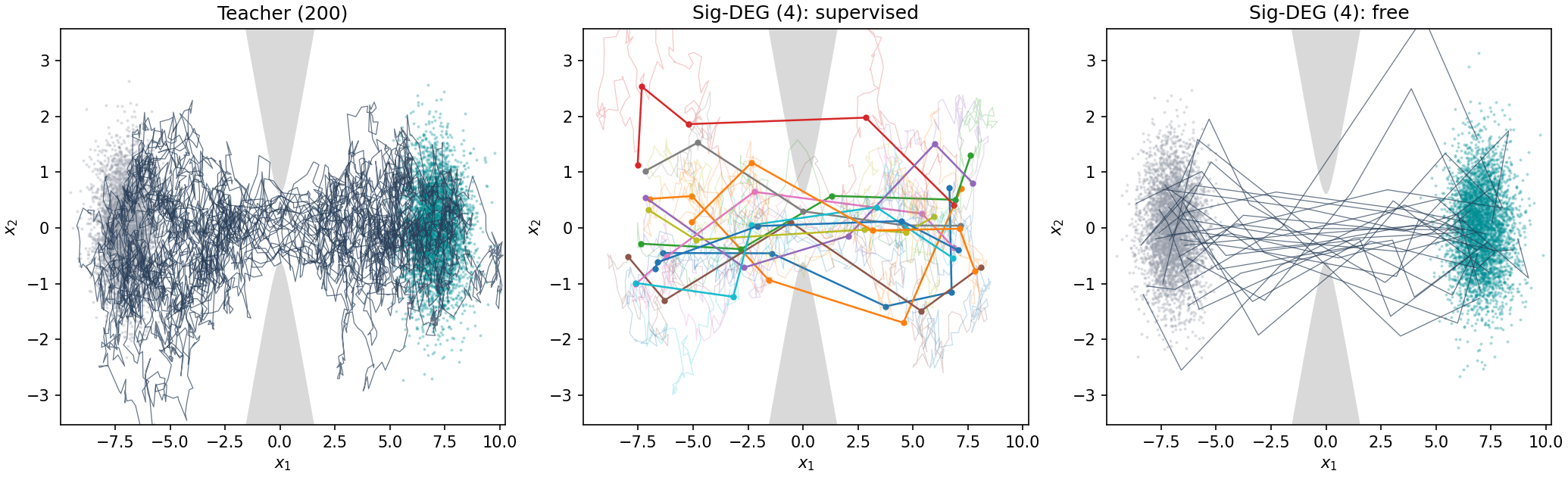}
    \caption{\textbf{Sig-DEG (4) distilled from the V-neck teacher.} \emph{Left}: the $200$-step teacher, with source (grey) and terminal (teal) samples and $25$ paths. \emph{Center}: $12$ teacher paths (faint) and the student driven by their Brownian paths, starting from the start points; the dots are the student's states at the coarse times. \emph{Right}: free generation of the student from $\mu_0$ with fresh noise, and $25$ of its paths. Segments between coarse states are drawn for legibility and are not paths: only the dots are states.}
    \label{fig:sb_vneck}
\end{figure}

\begin{figure}[t!]
    \centering
    \includegraphics[width=\linewidth]{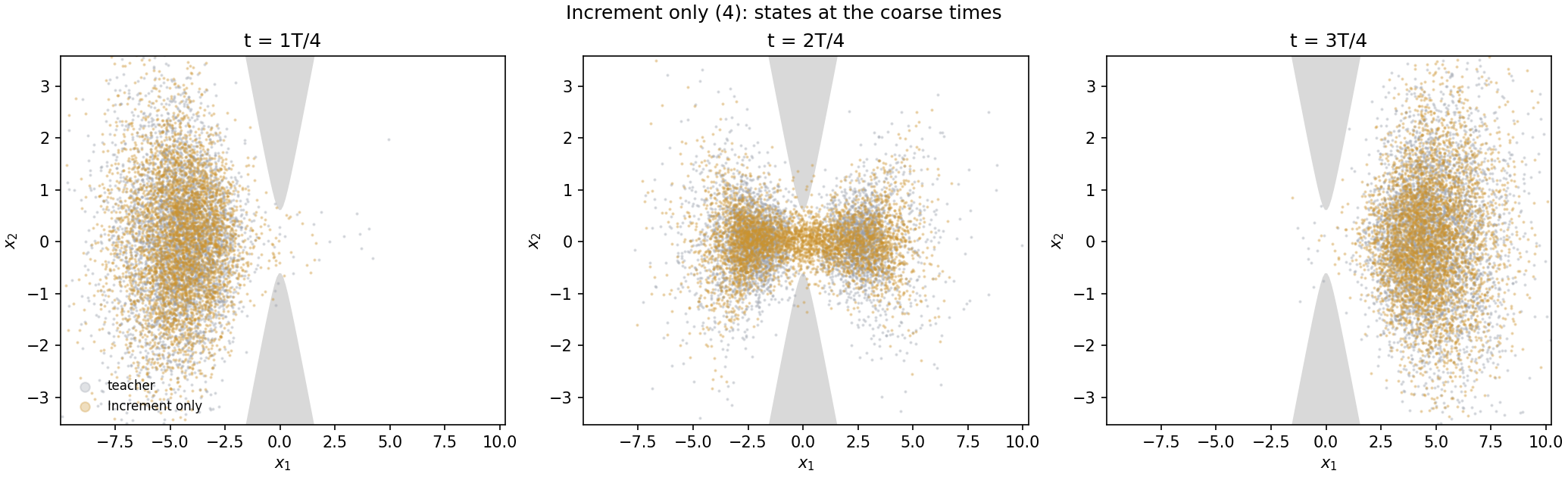}\\
    \includegraphics[width=\linewidth]{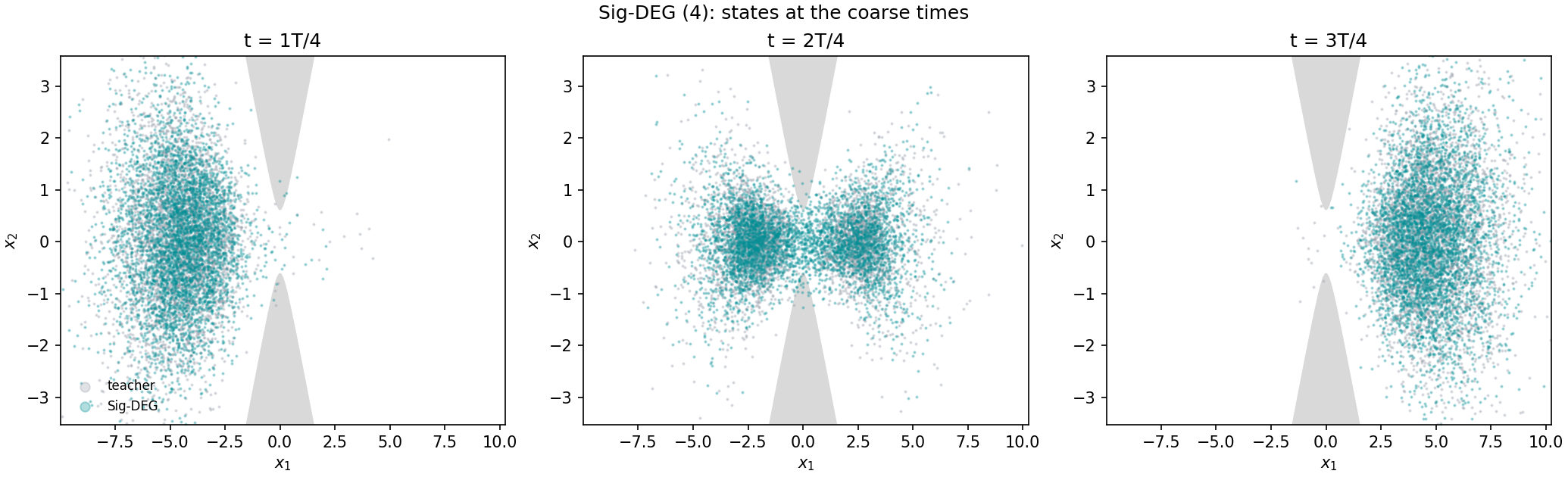}
    \caption{\textbf{V-neck teacher}, states of $4{,}096$ freely generated paths at the three interior coarse times of the $4$-step students (colour) over the teacher's states at the same times (grey). Top: increment-only; bottom: Sig-DEG (4). At $t = T/2$ the teacher's crowd is split on either side of the neck; both students put mass in the neck between the two lobes, the increment-only student more so.}
    \label{fig:sb_vneck_snap}
\end{figure}
\FloatBarrier

\end{document}